\documentclass[final]{IEEEtran}
\usepackage{amssymb,amsmath,bm}
\usepackage{epsfig}
\usepackage{subfigure}
\usepackage{algorithm}
\usepackage{algorithmic}
\usepackage{multirow}
\usepackage[T1]{fontenc}
\usepackage{url}
\usepackage{color}

\newcommand{\matH}{\mathbf{H}}

\newcommand{\matP}{\mathbf{P}}
\newcommand{\matQ}{\mathbf{Q}}

\newcommand{\matV}{\mathbf{V}}
\newcommand{\matW}{\mathbf{W}}

\newcommand{\matY}{\mathbf{Y}}

\newcommand{\calG}{\mathcal{G}}

\newcommand{\calI}{\mathcal{I}}

\newcommand{\calN}{\mathcal{N}}
\newcommand{\calO}{\mathcal{O}}
\newcommand{\calP}{\mathcal{P}}

\newcommand{\bbR}{\mathbb{R}}

\newcommand{\vecm}{\mathbf{m}}

\newcommand{\vecp}{\mathbf{p}}

\newcommand{\vecx}{\mathbf{x}}
\newcommand{\vecy}{\mathbf{y}}

\newcommand{\veceta}{\boldsymbol{\eta}}

\newcommand{\vecmu}{\boldsymbol{\mu}}

\newcommand{\vecTheta}{\boldsymbol{\Theta}}

\newcommand{\veczero}{\mathbf{0}}

\newcommand{\fracpartial}[2]{\frac{\partial #1}{\partial  #2}}

\newcommand{\xt}{\tilde{x}}
\newcommand{\mut}{\tilde{\mu}}
\newcommand{\Vh}{\widehat{\matV}}

\newcommand{\betap}{{\beta+1}}
\newcommand{\betam}{{\beta-1}}
\newcommand{\betapp}{({\beta+1})}
\newcommand{\betamm}{({\beta-1})}

\newtheorem{theorem}{Theorem}
\newtheorem{lemma}[theorem]{Lemma}

\title{Learning the Information Divergence}
\author{Onur~Dikmen,
        Zhirong~Yang,
        and~Erkki~Oja,% 
\thanks{The authors are with Department of Information and Computer Science,
Aalto University, 00076, Finland.
e-mail: onur.dikmen@aalto.fi; zhirong.yang@aalto.fi; erkki.oja@aalto.fi}}

\begin{document}
\maketitle
\begin{abstract}
Information divergence that measures the difference between two
nonnegative matrices or tensors has found its use in a variety of machine learning
problems. Examples are Nonnegative Matrix/Tensor Factorization, Stochastic
Neighbor Embedding, topic models, and Bayesian network optimization.  
The success of such a learning task depends heavily on a suitable
divergence. A large variety of divergences have been suggested and
analyzed, but very few results are available for an objective choice
of the optimal divergence for a given task. 
Here we present a framework that facilitates automatic
selection of the best divergence among a given family, based on
standard maximum likelihood estimation.  We first
propose an approximated Tweedie distribution for the
$\beta$-divergence family. Selecting the best $\beta$ then becomes a
machine learning problem solved by maximum likelihood.
Next, we reformulate $\alpha$-divergence in terms of
$\beta$-divergence, which enables automatic selection of $\alpha$ by
maximum likelihood with reuse of the learning principle for
$\beta$-divergence. Furthermore, we show the connections between
$\gamma$- and $\beta$-divergences as well as R\'enyi- and
$\alpha$-divergences, such that our automatic selection framework is
extended to non-separable divergences. Experiments on both synthetic and 
real-world data demonstrate that our method can quite accurately select
information divergence across different learning problems and various
divergence families.
\end{abstract}

%%%%%%%%%%%%%%%%%%%%%%%%%%%%%%%%%%%%%%%%%%%%%%%%
\section{Introduction}
\label{sec:intro}
Information divergences are an essential element in modern machine
learning. They originated in estimation theory where a divergence
maps the dissimilarity between two probability distributions to
nonnegative values. Presently, information divergences have been
extended for nonnegative tensors and used in many learning problems
where the objective is to minimize the approximation error between the
observed data and the model. Typical applications include Nonnegative
Matrix Factorization (see e.g.
\cite{kompass2006divergence,dhillo2006nips,cichocki2008alphanmf,TNN2011}),
Stochastic Neighbor Embedding \cite{hinton2002sne,maaten2008tsne},
topic models \cite{blei2001lda,sato2012rethinking}, and Bayesian
network optimization \cite{minka2005divergence}.

There exist a large variety of information divergences. In Section
\ref{sec:div}, we summarize the most popularly used parametric
families including $\alpha$-, $\beta$-, $\gamma$- and
R\'enyi-divergences
\cite{alphadiv,amari1985diff,betadiv,gammadiv,cichocki2010abgdiv} and
their combinations (e.g.\@ \cite{cichocki2011abnmf}). The four
parametric families in turn belong to broader ones such as the
Csisz\'ar-Morimoto $f$-divergences \cite{csiszarfdiv,morimotofdiv} and
Bregman divergences \cite{bregmandiv}.  
Data analysis techniques based on information divergences have been
widely and successfully applied
to various data such as text \cite{ICANN2011ROZ},
electroencephalography \cite{cichocki2008alphanmf}, facial images
\cite{ICANN2009ROZ}, and audio spectrograms
\cite{fevotte09nonnegative}.

Compared to the rich set of available information divergences, there is little
research on how to select the best one for a given application. This
is an important issue because the performance of a given
divergence-based estimation or modeling method in a particular task
very much depends on the divergence used. Formulating a learning task in
a family of divergences greatly increases the flexibility to handle different
types of noise in data. 
For example, Euclidean
distance is suitable for data with Gaussian noise; Kullback-Leibler divergence
has shown success for finding topics in text documents
\cite{blei2001lda}; and Itakura-Saito divergence has proven to be suitable
for audio signal processing \cite{fevotte09nonnegative}. A
conventional workaround is to select among a finite number of
candidate divergences using a validation set. This however cannot be
applied to divergences that are non-separable over tensor entries. The
validation approach is also problematic for tasks where all
data are needed for learning, for example, cluster analysis.

In Section \ref{sec:divlearning}, we propose a new method of
statistical learning for selecting the best divergence among the four
popular parametric families in any given data modeling task.
Our starting-point is the Tweedie
distribution~\cite{jorgensen87exponential}, which is known to have 
a relationship with $\beta$-divergence \cite{cichocki09nonnegative,yilmaz12alpha_beta}. The Maximum Tweedie Likelihood
(MTL) is in principle a disciplined and straightforward method 
for choosing the optimal $\beta$ value. However, in order for this to
be feasible in practice, two shortcomings with the MTL method 
have to be overcome: 1) Tweedie distribution is
not defined for all $\beta$; 2) calculation of Tweedie likelihood is
complicated and prone to numerical problems for large $\beta$.
To overcome these drawbacks, we propose here a novel distribution using
an exponential over the $\beta$-divergence with a specific augmentation term. The
new distribution has the following nice properties:
1) it is close to the Tweedie distribution, especially at four important special
cases;
2) it exists for all $\beta\in\bbR$;
3) its likelihood can be calculated by standard statistical software.
We call the new density the Exponential Divergence with Augmentation
(EDA). EDA is a non-normalized density, i.e., its likelihood includes a normalizing constant which is not analytically available. But, since the density is univariate the normalizing constant can be efficiently and accurately estimated by numerical integration.
The method of Maximizing the Exponential Divergence with Augmentation
Likelihood (MEDAL) thus gives a more robust $\beta$ selection
in a wider range than MTL. $\beta$ estimation on EDA can also be carried out using parameter estimation methods, e.g., Score Matching (SM)~\cite{hyvaerinen05estimation}, specifically proposed for non-normalized densities. In the experiments section, we show that SM on EDA also performs as accurately as MEDAL.

Besides $\beta$-divergence, the MEDAL method is extended to select the
best divergence in other parametric families.
We reformulate $\alpha$-divergence in terms of $\beta$-divergence
after a change of parameters so that $\alpha$ can be optimized using
the MEDAL method.
Our method can also be applied to non-separable cases. We show the
equivalence between $\beta$ and $\gamma$-divergences, and between
$\alpha$ and R\'enyi divergences by a connecting scalar, which allows
us to choose the best $\gamma$- or R\'enyi-divergence by reusing the
MEDAL method.

We tested our method with extensive experiments, whose results are
presented in Section \ref{sec:exp}. We have used both synthetic data
with a known distribution and real-world data including music, stock prices, and social networks. The MEDAL method is applied to different
learning problems: Nonnegative Matrix Factorization (NMF)
\cite{lee99learning,cichocki2008alphanmf,kompass2006divergence},
Projective NMF \cite{yuan05pnmf,TNN2010} and Symmetric Stochastic
Neighbor Embedding for visualization
\cite{hinton2002sne,maaten2008tsne}. We also demonstrate that our
method outperforms Score Matching on Exponential Divergence distribution (ED), a previous approach for
$\beta$-divergence selection \cite{lu12selecting}. Conclusions and
discussions on future work are given in Section \ref{sec:conclusions}.

%%%%%%%%%%%%%%%%%%%%%%%%%%%%%%%%%%%%%%%%%%%%%%%%
\section{Information Divergences}
\label{sec:div}
Many learning objectives can be formulated as an approximation of the
form $\vecx\approx\vecmu$, where $\vecx>0$ is the observed data (input)
and $\vecmu$ is the approximation given by the model. The formulation
for $\vecmu$ totally depends on the task to be solved. Consider
Nonnegative Matrix Factorization: then $\vecx>0$ is a data matrix and
$\vecmu$ is a product of two lower-rank nonnegative matrices which typically give
a sparse representation for the columns of $\vecx$. Other concrete
examples are given in Section \ref{sec:exp}.   

The 
approximation error can be measured by various information
divergences. Suppose $\vecmu$ is parameterized by $\vecTheta$.  The
learning problem becomes an optimization procedure that minimizes the
given divergence $D(\vecx||\vecmu(\vecTheta))$ over $\vecTheta$.
Regularization may be applied for $\vecTheta$ for complexity control.
For notational brevity we focus on definitions over vectorial $\vecx$,
$\vecmu$, $\vecTheta$ in this section, while they can be extended to
matrices or higher order tensors in a straightforward manner.

In this work we consider four parametric families of divergences,
which are the widely used $\alpha$-,
$\beta$-, $\gamma$- and R\'enyi-divergences. This collection is rich
because it covers most commonly used divergences. The definition of
the four families and some of their special cases are given below.

\begin{itemize}
\item $\alpha$-divergence \cite{alphadiv,amari1985diff} is defined as
\begin{align}
D_\alpha (\vecx||\vecmu) = \frac{\sum_i x_i^\alpha\mu_i^{1-\alpha} - \alpha
x_i + (\alpha-1)\mu_i}{\alpha(\alpha-1)}.
\end{align}
The family contains the following special cases:
\begin{align*}
D_{\alpha=2} (\vecx||\vecmu) =& 
D_\text{P}(\vecx||\vecmu) = \frac{1}{2} \sum_i \frac{(x_i-\mu_i)^2}{\mu_i}\\
D_{\alpha\rightarrow 1} (\vecx||\vecmu) =& 
D_\text{I}(\vecx||\vecmu) = \sum_i\left( x_i\ln \frac{x_i}{\mu_i} - x_i + \mu_i\right)\\
D_{\alpha=1/2} (\vecx||\vecmu) =& 2D_\text{H}(\vecx||\vecmu) = 2
\sum_i (\sqrt{x_i}-\sqrt{\mu_i})^2\\
D_{\alpha\rightarrow 0} (\vecx||\vecmu) =& 
D_\text{I}(\vecmu||\vecx) = \sum_i\left(\mu_i\ln \frac{\mu_i}{x_i} - \mu_i + x_i\right)\\
D_{\alpha=-1} (\vecx||\vecmu) =& D_\text{IP}(\vecx||\vecmu) = \frac{1}{2} \sum_i \frac{(x_i-\mu_i)^2}{x_i}
\end{align*}
where $D_\text{I}$, $D_\text{P}$, $D_\text{IP}$, and
$D_\text{H}$ denote non-normalized Kullback-Leibler, Pearson
Chi-square, inverse Pearson and Hellinger distances, respectively.

\item $\beta$-divergence \cite{eguchi01robustifing, minami02robust} is
  defined as 
\begin{align}\label{eq:betadiv}
D_\beta (\vecx||\vecmu) = \frac{\sum_i x_i^{\beta+1} + \beta\mu_i^{\beta+1} -
(\beta+1) x_i\mu_i^\beta}{\beta(\beta+1)}.
\end{align}
The family contains the following special cases:
\begin{align}
\label{eq:eusq}
D_{\beta=1} (\vecx||\vecmu) =& 
D_\text{EU}(\vecx||\vecmu) = \frac{1}{2} \sum_i (x_i-\mu_i)^2\\
\label{eq:idiv}
D_{\beta\rightarrow 0} (\vecx||\vecmu) =& 
D_\text{I}(\vecx||\vecmu) = \sum_i\left( x_i\ln \frac{x_i}{\mu_i} - x_i + \mu_i\right)\\
\label{eq:isdiv}
D_{\beta\rightarrow -1} (\vecx||\vecmu) =& 
D_\text{IS}(\vecx||\vecmu) = \sum_i\left(\frac{x_i}{\mu_i} - \ln\frac{x_i}{\mu_i} - 1\right)\\
\label{eq:indiv}
D_{\beta=-2} (\vecx||\vecmu) =& 
\sum_i\left(\frac{x_i}{2\mu_i^2}-\frac{1}{\mu_i}+\frac{1}{2x_i}\right),
\end{align}
where $D_\mathrm{EU}$ and $D_\mathrm{IS}$ denote the Euclidean
distance and Itakura-Saito divergence, respectively.

\item $\gamma$-divergence \cite{gammadiv} is defined as 
\begin{align}
\label{eq:divgamma}
\nonumber
D_\gamma(\vecx||\vecmu)=&\frac{1}{\gamma(\gamma+1)}
\left[
\ln\left(\sum_ix_i^{\gamma+1}\right)
+\gamma\ln\left(\sum_i\mu_i^{\gamma+1}\right)
\right.\\
&\left.-(\gamma+1)\ln\left(\sum_ix_i\mu_i^\gamma\right)\right].
\end{align}
The normalized Kullback-Leibler (KL) divergence is a special case of $\gamma$-divergence:
\begin{align}
\label{eq:divkl}
D_{\gamma\rightarrow 0}(\vecx||\vecmu)=D_\mathrm{KL}(\tilde{\vecx}||\tilde{\vecmu})
=\sum_i\xt_i\ln\frac{\xt_i}{\mut_i},
\end{align}
where $\xt_i=x_i/\sum_jx_j$ and $\mut_i=\mu_i/\sum_j\mu_j$.

\item R\'enyi divergence \cite{renyidivergence} is defined as 
\begin{align}
\label{eq:divrenyi}
D_\rho(\vecx||\vecmu)=\frac{1}{\rho-1}
\ln\left(\xt_i^\rho\mut_i^{1-\rho}\right)
\end{align}
for $\rho>0$. The R\'enyi divergence also includes the normalized
Kullback-Leibler divergence as its special case when
$\rho\rightarrow1$.

\end{itemize}

\section{Divergence Selection by Statistical Learning}
\label{sec:divlearning}

The above rich collection of information divergences basically allows
great
flexibility to the approximation framework. However, practitioners
must face a choice problem: \emph{how to select the best divergence in
a family?}  In most existing applications the selection is done
empirically by the human. A conventional automatic selection method is
cross-validation \cite{mollah2007beta,choi2010alphaintegration}, where
the training only uses part of the entries of $\vecx$ and the
remaining ones are used for validation. This method has a number of
drawbacks: 1) it is only applicable to the divergences where the
entries are separable (e.g.\@ $\alpha$- or $\beta$-divergence).
Leaving out some entries for $\gamma$- and R\'enyi divergences is
infeasible due to the logarithm or normalization; 2) separation of
entries is not applicable in some applications where all entries are
needed in the learning, for example, cluster analysis. 

Our proposal here is to use the familiar and proven technique of 
maximum likelihood estimation for \emph{automatic
divergence selection}, using a suitably chosen and very flexible probability density
model for the data. In the
following we discuss this statistical learning approach for automatic
divergence selection in the family of $\beta$ -divergences, followed
by its extensions to the other divergence families. 

\subsection{Selecting $\beta$-divergence}
\label{sec:selectbeta}

\subsubsection{Maximum Tweedie Likelihood (MTL)}
\label{sec:MTL}
We start from the probability density function (pdf) of an exponential
dispersion model (EDM)~\cite{jorgensen87exponential}:
\begin{align}
\label{eq:edm_pdf}
p_\text{EDM}(x;\theta,\phi,p) = f(x,\phi,p) \exp\left[\frac{1}{\phi}(x\theta-\kappa(\theta))\right]
\end{align}
where $\phi>0$ is the dispersion parameter, $\theta$ is the canonical
parameter, and $\kappa(\theta)$ is the cumulant function (when $\phi=1$
its derivatives w.r.t. $\theta$ give the cumulants). Such a
distribution has mean $\mu=\kappa'(\theta)$ and variance
$V(\mu,p)=\phi\kappa''(\theta)$. This density is defined for $x \geq
0$, thus $\mu > 0$. 

A Tweedie distribution is an EDM whose variance has a special form,
$V(\mu)=\mu^p$ with $p\in \bbR\backslash (0,1)$. The canonical
parameter and the cumulant function that satisfy this property are \cite{jorgensen87exponential}
\begin{align}
\label{eq:tweedie_param}
\theta = \left\{ \begin{array}{rl}
\frac{\mu^{1-p} - 1}{1-p}, &\mbox{ if $p\neq 1$} \\
\ln\mu, &\mbox{ if $p=1$}
\end{array}\right.,
\;\;\;\;
\kappa(\theta) = \left\{ \begin{array}{rl}
 \frac{\mu^{2-p} - 1}{2-p}, &\mbox{ if $p\neq 2$} \\
 \ln\mu, &\mbox{ if $p=2$}
\end{array} \right.\;.
\end{align}
Note that $\ln\mu$ is the limit of $\frac{\mu^t-1}{t}$ as $t \to 0$. 
Finite analytical forms of $f(x,\phi,p)$ in Tweedie distribution are
generally unavailable.  The function can be expanded with infinite
series \cite{dunn05series} or approximated by saddle point
estimation~\cite{dunn01tweedie}.

It is known that the Tweedie distribution has a connection to
$\beta$-divergence (see, e.g.,\@ \cite{cichocki09nonnegative,
  yilmaz12alpha_beta}): 
maximizing the likelihood of Tweedie distribution for certain $p$
values is equivalent to minimizing the corresponding divergence with
$\beta=1-p$.
Especially, the gradients of the log-likelihood of Gamma, Poisson and
Gaussian distributions over $\mu_i$ are equal to the ones of
$\beta$-divergence with $\beta=-1,0,1$, respectively. This motivates a
$\beta$-divergence selection method by Maximum Tweedie Likelihood
(MTL).

However, MTL has the following two shortcomings.
First, Tweedie distribution is not defined for $p\in(0,1)$. That is,
if the best $\beta=1-p$ happens to be in the range $(0,1)$, it cannot
be found by MTL; in addition, there is little research on the Tweedie
distribution with $\beta>1 \; (p < 0)$.
Second, $f(x,\phi,p)$ in Tweedie distribution is not the probability
normalizing constant (note that it depends on $x$), and its evaluation
requires ad hoc techniques. The existing software using the infinite
series expansion approach \cite{dunn05series} (see Appendix
\ref{sec:tweedieexpansion}) is prone to numerical computation problems
especially for $-0.1<\beta<0$. There is no existing implementation
that can calculate Tweedie likelihood for $\beta>1$.
%%%%%%%%%%%%%%%%%%%%%%%%%%%%%%%%%%%%%%%%%%%%%%%%
\subsubsection{Maximum Exponential Divergence with Augmentation
  Likelihood (MEDAL)}
\label{sec:medal}
Our answer to the above shortcomings in MTL is to design an
alternative distribution with the following properties:
1) it is close to the Tweedie distribution, especially for the four
crucial points when $\beta\in\{-2,-1,0,1\}$;
2) it should be defined for all $\beta\in\bbR$;
3) its pdf can be evaluated more robustly by standard statistical software.

From \eqref{eq:edm_pdf} and \eqref{eq:tweedie_param} the pdf of the Tweedie distribution is written as
\begin{align}
\label{eq:tweedie_pdf}
p_\text{Tw}(x;\mu,\phi,\beta) &= f(x,\phi,\beta)
\exp\left[\frac{1}{\phi}\left(\frac{x\mu^\beta}{\beta}-\frac{\mu^\betap}{\betap}\right)\right] 
\end{align}
w.r.t. $\beta$ instead of $p$, using the relation $\beta=1-p$. This
holds when 
$\beta \neq 0$ and $\beta \neq -1$.  The extra terms $1/(1-p)$ and
$1/(2-p)$ in (\ref{eq:tweedie_param}) have been absorbed in $f(x,\phi,\beta)$. The
cases $\beta= 0$ or $\beta = -1$ have to be analyzed separately. 
 
To make an explicit connection with $\beta$-divergence defined in
(\ref{eq:betadiv}), we suggest a new distribution given in the following
form: 
\begin{align}
&p_\text{approx}(x;\mu,\phi,\beta) = g(x,\phi,\beta) \exp\left\{- \frac{1}{\phi} D_\beta(x||\mu)\right\}\nonumber\\
\label{eq:pdf_approx}
&= g(x,\phi,\beta) \exp\left[\frac{1}{\phi}\left( -
    \frac{x^\betap}{\beta\betapp} + \frac{x\mu^\beta}{\beta} -
    \frac{\mu^\betap}{\betap} \right)\right]. 
\end{align}
Now the $\beta$-divergence for scalar $x$ appears in the exponent, and
$g(x,\phi,\beta)$ will be used to approximate this with the Tweedie
distribution. Ideally, the choice
\begin{align*}
g(x,\phi,\beta) = f(x,\phi,\beta)/\exp\left[\frac{1}{\phi}\left( - \frac{x^\betap}{\beta\betapp}\right)\right]
\end{align*}
would result in full equivalence to Tweedie distribution, as seen from
(\ref{eq:tweedie_pdf}). However, because
$f(x,\phi,\beta)$ is unknown in the general case, such $g$ is also
unavailable.

We can, however, try to approximate $g$ using the fact that
$p_\text{approx}$ must be a proper density whose integral is equal to
one. From (\ref{eq:pdf_approx}) it then follows 
\begin{align}
\exp\left[\frac{1}{\phi}\frac{\mu^\betap}{\betap}\right]\nonumber\\
\label{eq:Z_approx}
& = \int dx\, g(x,\phi,\beta) \exp\left[\frac{1}{\phi}\left( - \frac{x^\betap}{\beta\betapp} + \frac{x\mu^\beta}{\beta}\right)\right]
\end{align}
This integral is, of course, impossible to evaluate because we do not
even know the function inside. However, the integral can be
approximated nicely by Laplace's method. Laplace's approximation is
\begin{align*}
\int_a^b dx\, f(x) e^{M h(x)} \approx \sqrt{\frac{2\pi}{M|h''(x_0)|}} f(x_0) e^{M h(x_0)}
\end{align*}
where $x_0 = \arg\max_x h(x)$ and $M$ is a large constant.

In order to approximate \eqref{eq:Z_approx} by Laplace's method,
$1/\phi$ takes the role of $M$ and thus the approximation is valid for
small $\phi$. We
need the maximizer of the exponentiated term $h(x) = 
- \frac{x^\betap}{\beta\betapp} +
  \frac{x\mu^\beta}{\beta}$. This term has a zero first
derivative and negative second derivative, i.e., it is maximized,  at
$x=\mu$. Thus, Laplace's method gives us 
\begin{align*}
&\exp\left[\frac{1}{\phi}\frac{\mu^\betap}{\betap}\right] \\
&\approx \sqrt{\frac{2\pi\phi}{|-\mu^\betam|}} \;g(\mu,\phi,\beta) \exp\left[\frac{1}{\phi}\left( - \frac{\mu^\betap}{\beta\betapp} + \frac{\mu^\betap}{\beta}\right)\right]\\
&= \sqrt{\frac{2\pi\phi}{\mu^\betam}} \;g(\mu,\phi,\beta) \exp\left[\frac{1}{\phi}\frac{\mu^\betap}{\betap}\right]\,.
\end{align*}
The approximation gives $g(\mu,\phi,\beta) = \frac{1}{\sqrt{2\pi\phi}}
\mu^{\betamm/2}$ which suggests the function 
\begin{align*}
g(x,\phi,\beta) = \frac{1}{\sqrt{2\pi\phi}} x^{\betamm/2}\, =
\frac{1}{\sqrt{2\pi\phi}}\exp\left[\frac{\betamm}{2}\ln x \right]. 
\end{align*}

Putting this result into \eqref{eq:pdf_approx} as such does not
guarantee a proper pdf however, because it is an approximation, only
valid
at the limit $\phi \to 0$.  To
make it proper, we have to add a normalizing constant into the density in
\eqref{eq:pdf_approx}. 

The pdf of the final distribution, for a scalar argument $x$,  thus becomes
\begin{align}
\label{eq:edascalar_pdf}
p_\text{approx}(x;\mu,\beta,\phi) =
\frac{1}{Z(\mu,\beta,\phi)}\exp\left\{R(x,\beta) - \frac{1}{\phi} D_\beta(x||\mu)\right\}
\end{align}
where $Z(\mu,\beta,\phi)$ is the normalizing constant counting for
the terms which are independent of $x$, and $R(x,\beta)$ is an
augmentation term given as 
\begin{align}
\label{eq:scalarR}
R(x,\beta)=\frac{\beta-1}{2}\ln x\,.
\end{align}

This pdf is a proper density for all $\beta\in\bbR$,
which is guaranteed by the following theorem.
\begin{theorem}
\label{theo:edaexist}
Let $f(x)=\exp\left\{\frac{\beta-1}{2}\ln x-\frac{1}{\phi}D_\beta(x||\mu)\right\}$. The improper integral 
$\int_0^\infty f(x)dx$ converges.
\end{theorem}
\begin{proof}
Let
$q=\left|\frac{\beta-1}{2}\right|+1+\epsilon$ with
any $\epsilon\in(0,\infty)$, and $g(x)=x^{-q}$. By these definitions, we have
$q>\left|\frac{\beta-1}{2}\right|$, and then for $x\geq1$,
$\left(\frac{\beta-1}{2}+q\right)\phi\ln x\leq0\leq D_\beta(x||\mu)$,
i.e.\@ $0\leq f(x)\leq g(x)$. By Cauchy convergence test, we know that
$\int_1^\infty g(x)dx$ is convergent because $q>1$, and so is $\int_1^\infty
f(x)dx$.
Obviously $f(x)$ is continuous and bounded for $x\in[0,1]$. Therefore,
for $x\geq0$, $\int_0^\infty f(x)dx=\int_0^1 f(x)dx+\int_1^\infty
f(x)dx$ also converges.
\end{proof}

Finally, for vectorial $\vecx$, the pdf is a product of the marginal densities:

\begin{align}
\label{eq:eda_pdf}
p_\text{EDA}(\vecx;\vecmu,\beta,\phi) =
\frac{1}{Z(\vecmu,\beta,\phi)}\exp\left\{R(\vecx,\beta) - \frac{1}{\phi} D_\beta(\vecx||\vecmu)\right\}
\end{align}

where $D_\beta(\vecx||\vecmu)$ is defined in (\ref{eq:betadiv}) and 
\begin{align}
\label{eq:R}
R(\vecx,\beta)=\frac{\beta-1}{2}  \sum_i \ln x_i \,.
\end{align}

We call (\ref{eq:eda_pdf}) the Exponential Divergence with
Augmentation (EDA) distribution, because it applies an exponential over
an information divergence plus an augmentation term. 

The log-likelihood of the EDA density can be written as  
\begin{multline}
\label{eq:logL}
\ln p(\vecx;\vecmu,\beta,\phi) =  \sum_i \ln p(x_i;\mu_i,\beta,\phi)\\
= \sum_i \left[\frac{\beta-1}{2}\ln x_i-\frac{1}{\phi} D_\beta(x_i||\mu_i)-\ln Z(\mu_i,\beta,\phi)\right]
\end{multline}
due to the fact that $D_\beta(\vecx||\vecmu)$ in Eq.~(\ref{eq:betadiv}) and the augmentation term in \eqref{eq:R} are separable over
$x_i$, (i.e.  $x_i$ are independent given $\mu_i$).
The best $\beta$ is now selected by 
\begin{align}
\label{eq:MEDAL}
\beta^*=\arg\max_\beta\left[\max_\phi\ln p(\vecx;\vecmu,\beta,\phi)\right]\,,
\end{align}
where $\vecmu=\arg\min_{\veceta}~D_\beta(\vecx||\veceta)$. We call the new divergence selection method \emph{Maximum EDA Likelihood} (MEDAL). 

Let us look at the four special cases of Tweedie distribution: Gaussian ($\calN$), Poisson
($\calP\calO$), Gamma ($\calG$) and Inverse Gaussian
($\calI\calN$). They correspond to $\beta = 1, 0, -1, -2$.  
For simplicity of notation, we may drop the
subscript $i$ and write $x$ and $\mu$ for one entry in $\vecx$ and
$\vecmu$. Then, the log-likelihoods of the above four special cases are
\begin{align*}
\ln p_\calN(x;\mu,\phi) =& -\frac{1}{2}\ln(2\pi\phi)-\frac{1}{2\phi}(x-\mu)^2,\\
\ln p_{\calP\calO}(x;\mu) =& x\ln\mu -\mu - \ln\Gamma(x+1),\\
\approx & x\ln\mu -\mu -\ln(2\pi x)/2 - x \ln x + x,\\
\ln p_\calG(x;1/\phi,\phi\mu) =& (1/\phi-1)\ln x-\frac{x}{\phi\mu}\\
& - (1/\phi)\ln(\phi\mu) - \ln\Gamma(1/\phi),\\
\ln p_{\calI\calN}(x;\mu,1/\phi) =& -\frac{1}{2}\ln(2\pi\phi x^3) - \frac{1}{\phi} \left(\frac{1}{2}\frac{x}{\mu^2} - \frac{1}{\mu} + \frac{1}{2x}\right),
\end{align*}
where in the Poisson case we employ Stirling's
approximation\footnote{The case $\beta=0$ and $\phi\neq 1$ does not
correspond to Poisson distribution, but the transformation
$p_\text{EDM}(x;\mu,\phi,1)=p_{\calP\calO}(x/\phi;\mu/\phi)/\phi$ can
be used to evaluate the pdf.}.
To see the similarity of these four special cases with the general
expression for the EDA log-likelihood in Eq.~(\ref{eq:logL}), let us
look at one term in the sum there. It is a fairly straightforward
exercise to plug in the $\beta$-divergences from
Eqs.~(\ref{eq:eusq},\ref{eq:idiv},\ref{eq:isdiv},\ref{eq:indiv}) and the augmentation
term from Eq.~(\ref{eq:R}) and see that the log-likelihoods
coincide. The normalizing term $\ln Z(\mu,\beta,\phi)]$ for these
special cases can be
determined from the corresponding density.  

In general, the normalizing constant $Z(\vecmu,\beta,\phi)$ is intractable except
for a few special cases. Numerical evaluation of
$Z(\vecmu,\beta,\phi)$ can be implemented by standard statistical
software. Here we employ the approximation with Gauss-Laguerre
quadratures (details in Appendix \ref{sec:gausslaguerre}). 

Finally, let us note that in addition to the maximum likelihood estimator, Score Matching (SM)
\cite{hyvaerinen05estimation,hyvaerinen07extensions} can be applied to
estimation of $\beta$ as a density parameter (see Section
\ref{sec:expsyn}).
In a previous effort, Lu et al. \cite{lu12selecting} proposed a similar
exponential divergence (ED) distribution
\begin{align}
\label{eq:ED}
p_\text{ED}(\vecx;\vecmu,\beta)\propto
\exp\left[-D_\beta(\vecx||\vecmu)\right], 
\end{align}
but without the augmentation.
It is easy to show that ED also exists for all $\beta$ by changing
$q=1+\epsilon$ in the proof of Theorem \ref{theo:edaexist}.
We will empirically illustrate the discrepancy between ED and EDA
in Section \ref{sec:expsyn}, showing that the selection based on
ED is however inaccurate, especially for $\beta\leq0$.
%%%%%%%%%%%%%%%%%%%%%%%%%%%%%%%%%%%%%%%%%%%%%%%%
\subsection{Selecting $\alpha$-divergence}
\label{sec:selectalpha}
We extend the MEDAL method to $\alpha$-divergence selection. This is
done by relating $\alpha$-divergence to $\beta$-divergence with a
nonlinear transformation between $\alpha$ and $\beta$. Let
$y_i=x_i^\alpha/\alpha^{2\alpha}$, $m_i=\mu_i^\alpha/\alpha^{2\alpha}$
and $\beta=1/\alpha-1$ for $\alpha\neq0$. We have
\begin{align*}
D_\beta(y_i||m_i)
=& \frac{1}{\beta(\beta+1)} \left(y_i^{\beta+1} + \beta m_i^{\beta+1} - (\beta+1) y_i m_i^\beta\right)\\
=& \frac{-\alpha^2}{\alpha-1} \left(\frac{x_i}{\alpha^2} + \frac{1-\alpha}{\alpha}\frac{\mu_i}{\alpha^2} - \frac{1}{\alpha}\frac{x_i^\alpha}{\alpha^{2\alpha}}\frac{\mu_i^{1-\alpha}}{\alpha^{2(1-\alpha)}} \right)\\
=& D_\alpha(x_i||\mu_i)
\end{align*}

This relationship allows us to evaluate the likelihood of $\mu$ and
$\alpha$ using $y_i$ and $\beta$:
\begin{align*}
p(x_i;\mu_i,\alpha,\phi) 
&= p(y_i;m_i,\beta,\phi) \left| \frac{dy_i}{dx_i} \right|\\
&= p(y_i;m_i,\beta,\phi) \frac{x_i^{\alpha-1}}{\alpha^{2(\alpha-1/2)}}\\
&= p(y_i;m_i,\beta,\phi) y_i^{-\beta} |\beta+1|
\end{align*}
In vectorial form, the best $\alpha$ for $D_\alpha(\vecx||\vecmu)$ is
then given by $\alpha^*=1/(\beta^*+1)$ where
\begin{align}
\nonumber
\beta^*=&\arg\max_\beta\Big\{\max_\phi\big[\ln p(\vecy;\vecm,\beta)\\
&-\beta\ln y_i+\ln|\beta+1|\big]\Big\},
\end{align}
where $\vecm=\arg\min_{\veceta} D_\beta(\vecy||\veceta)$. This
transformation method can handle all $\alpha$ except
$\alpha\rightarrow0$ since it corresponds to $\beta\rightarrow\infty$.

%%%%%%%%%%%%%%%%%%%%%%%%%%%%%%%%%%%%%%%%%%%%%%%%
\subsection{Selecting $\gamma$- and R\'enyi divergences}
\label{sec:nonseparable}
Above we presented the selection methods for two families where the
divergence is separable over the tensor entries. Next we consider
selection among $\gamma$- and R\'enyi divergence families where their
members are not separable.
Our strategy is to reduce $\gamma$-divergence to $\beta$-divergence
with a \emph{connecting scalar}. This is formally given by the
following result.
\begin{theorem}
\label{theo:connectgamma2beta}
For $\vecx\geq\veczero$ and $\tau\in\bbR$,
\begin{align}
\arg\min_{\vecmu\geq\veczero} D_{\gamma\rightarrow \tau}(\vecx||\vecmu) =
\arg\min_{\vecmu\geq\veczero} \left[\min_{c>0}D_{\beta\rightarrow
\tau}(\vecx||c\vecmu)\right]
\end{align}
\end{theorem}
The proof is done by zeroing the derivative right hand side with
respect to $c$ (details in Appendix \ref{sec:connectionproof}).

Theorem \ref{theo:connectgamma2beta} states that with a positive
scalar, the learning problem formulated by a $\gamma$-divergence is
equivalent to the one by the corresponding $\beta$-divergence. The
latter is separable and can be solved by the methods described in the
Section \ref{sec:selectbeta}. An example is between normalized
KL-divergence (in $\gamma$-divergence) and the non-normalized
KL-divergence (in $\beta$-divergence) with the optimal connecting
scalar $c=\frac{\sum_i x_i}{\sum_i\mu_i}$.  Example applications on
selecting the best $\gamma$-divergence are given in Section
\ref{sec:expgamma}.

Similarly, we can also reduce a R\'enyi divergence to its
corresponding $\alpha$-divergence with the same proof technique (see
Appendix \ref{sec:connectionproof}).
\begin{theorem}
\label{theo:connectrenyi2alpha}
For $\vecx\geq\veczero$ and $\tau>0$,
\begin{align}
\arg\min_{\vecmu\geq\veczero} D_{\rho\rightarrow \tau}(\vecx||\vecmu) =
\arg\min_{\vecmu\geq\veczero} \left[\min_{c>0}D_{\alpha\rightarrow
\tau}(\vecx||c\vecmu)\right].
\end{align}
\end{theorem}

%%%%%%%%%%%%%%%%%%%%%%%%%%%%%%%%%%%%%%%%%%%%%%%%
\section{Experiments}
\label{sec:exp}
In this section we demonstrate the proposed method on various data
types and learning tasks. First we provide the results on synthetic
data, whose density is known, to compare the behavior of MTL, MEDAL and the score matching
method \cite{lu12selecting}. Second, we illustrate the advantage of
the EDA density over ED.  Third, we apply
our method on $\alpha$- and $\beta$-divergence selection in
Nonnegative Matrix Factorization (NMF) on real-world data including
music and stock prices. Fourth, we test MEDAL in selecting
non-separable cases (e.g.\@ $\gamma$-divergence) for Projective NMF
and s-SNE visualization learning tasks across synthetic data, images,
and a dolphin social network.

\begin{figure*}[t]
\begin{center}
\begin{tabular}{cccc}
  \epsfig{figure=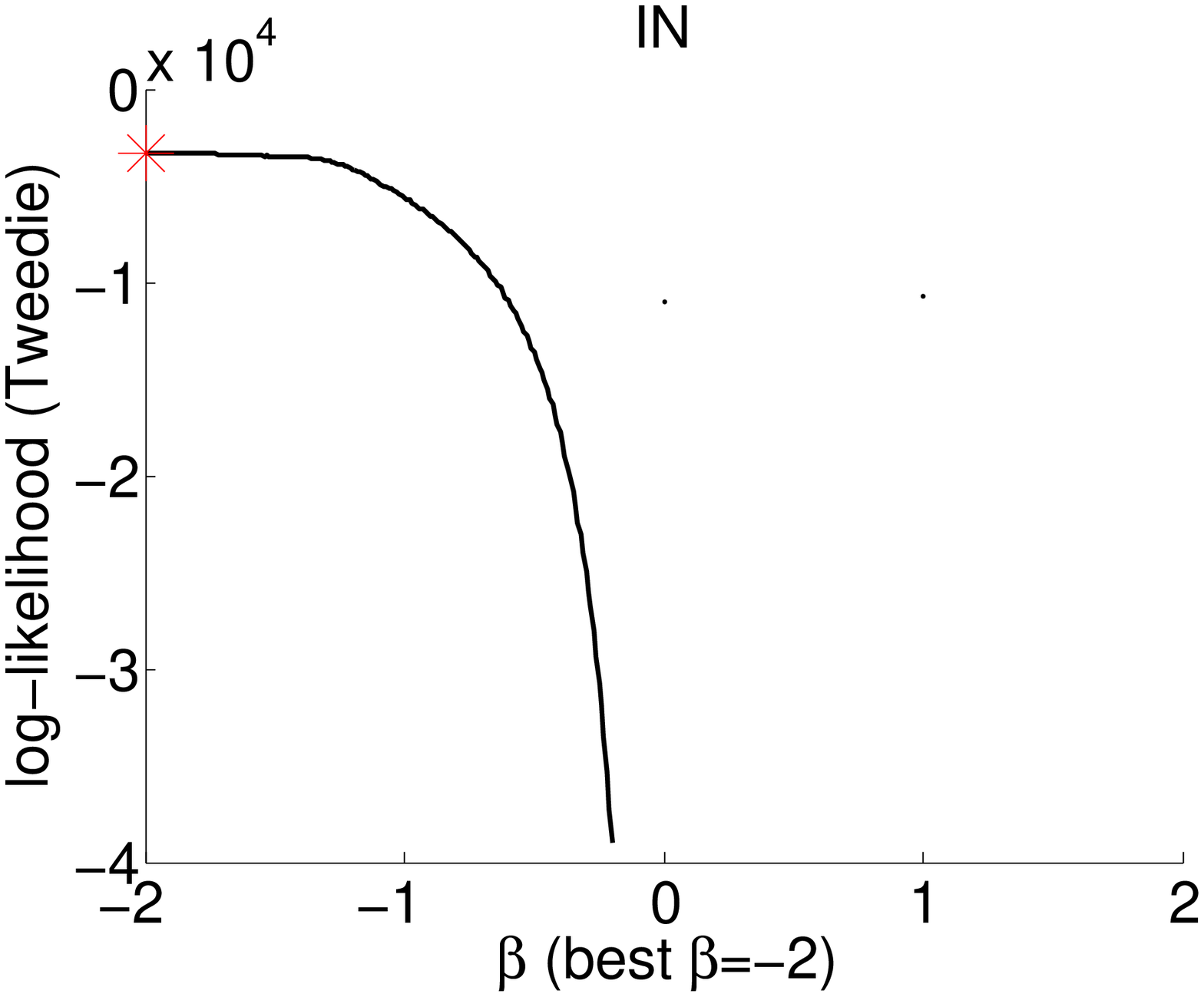,width=4cm}&
  \epsfig{figure=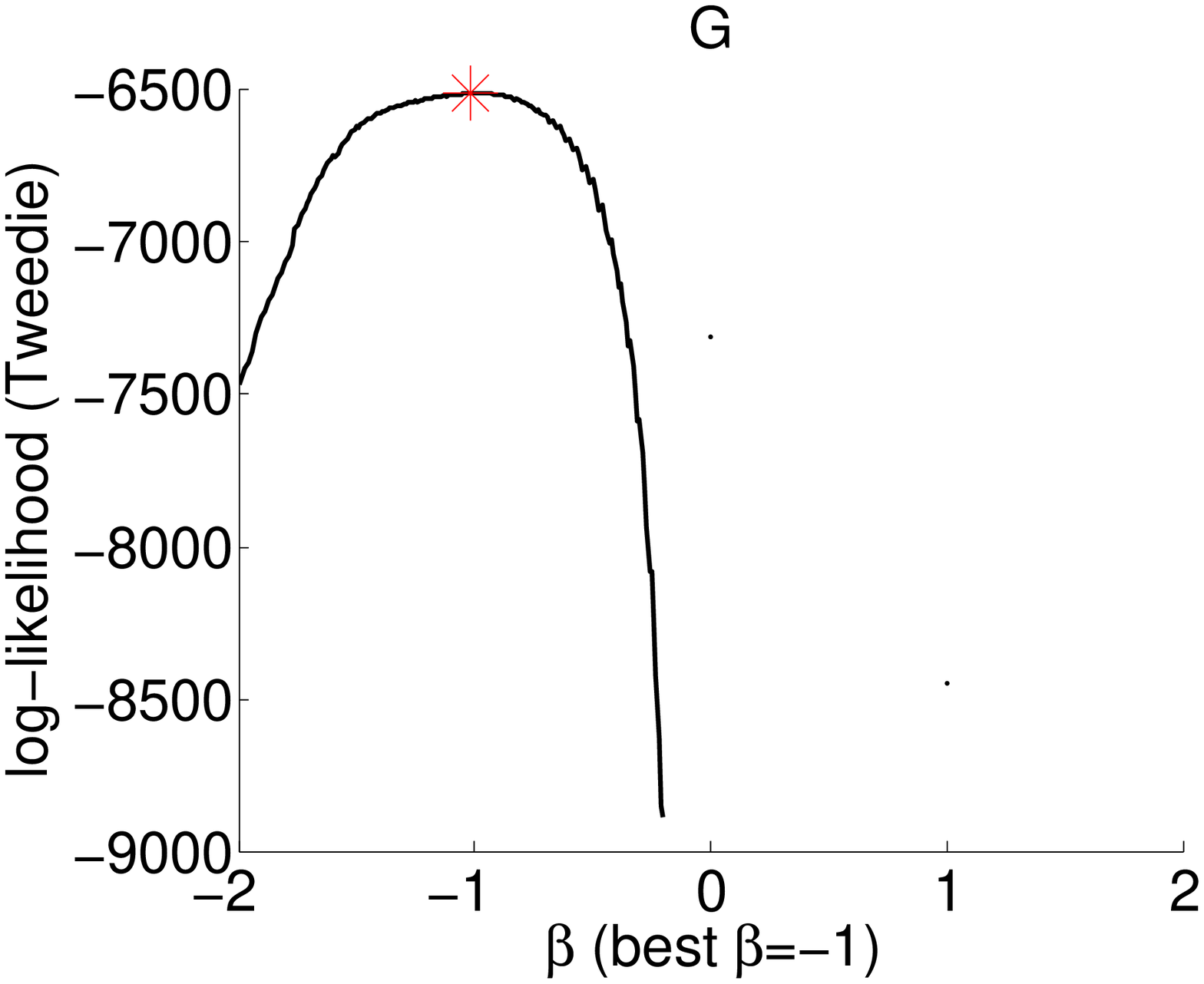,width=4cm}&
  \epsfig{figure=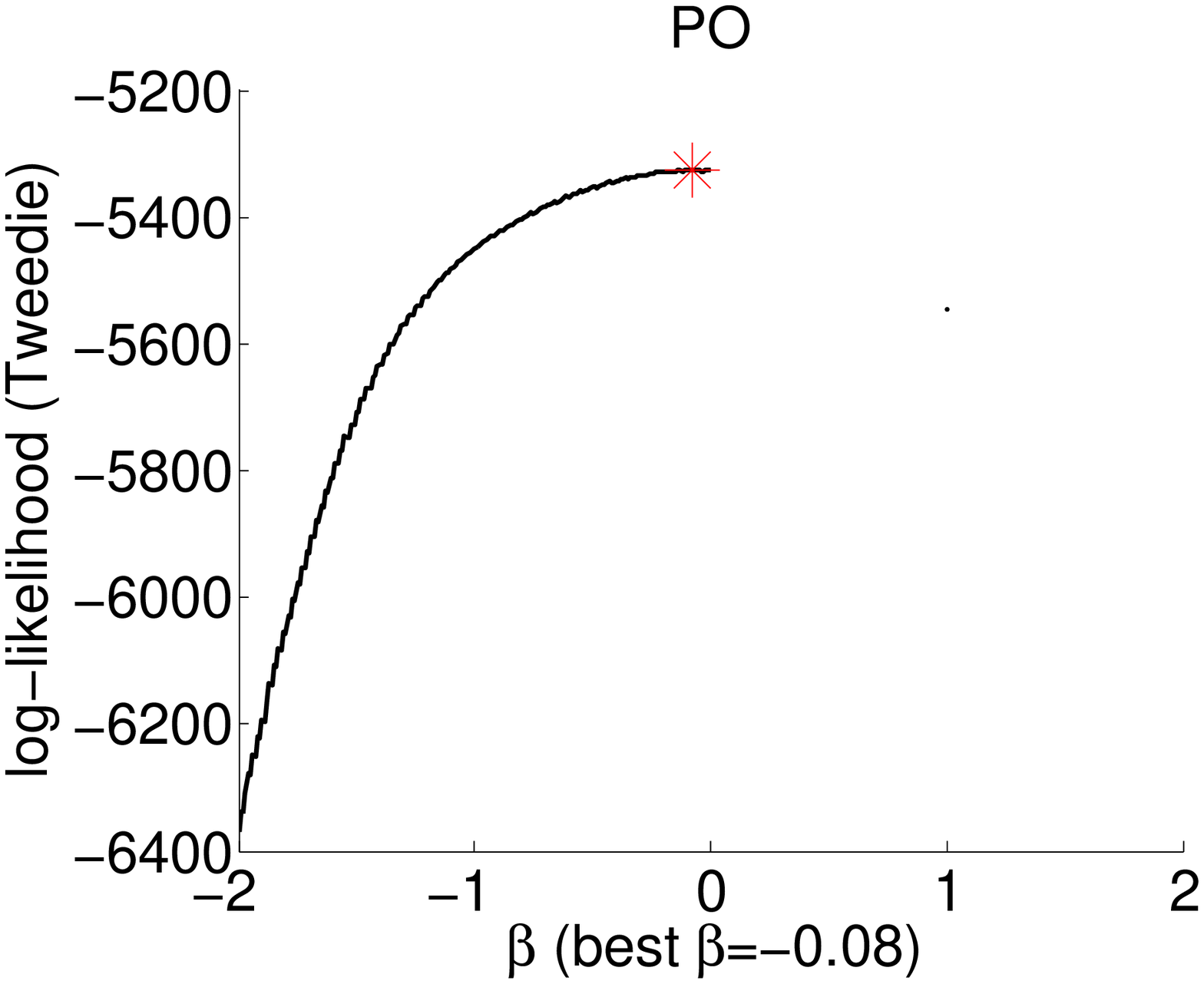,width=4cm}&
  \epsfig{figure=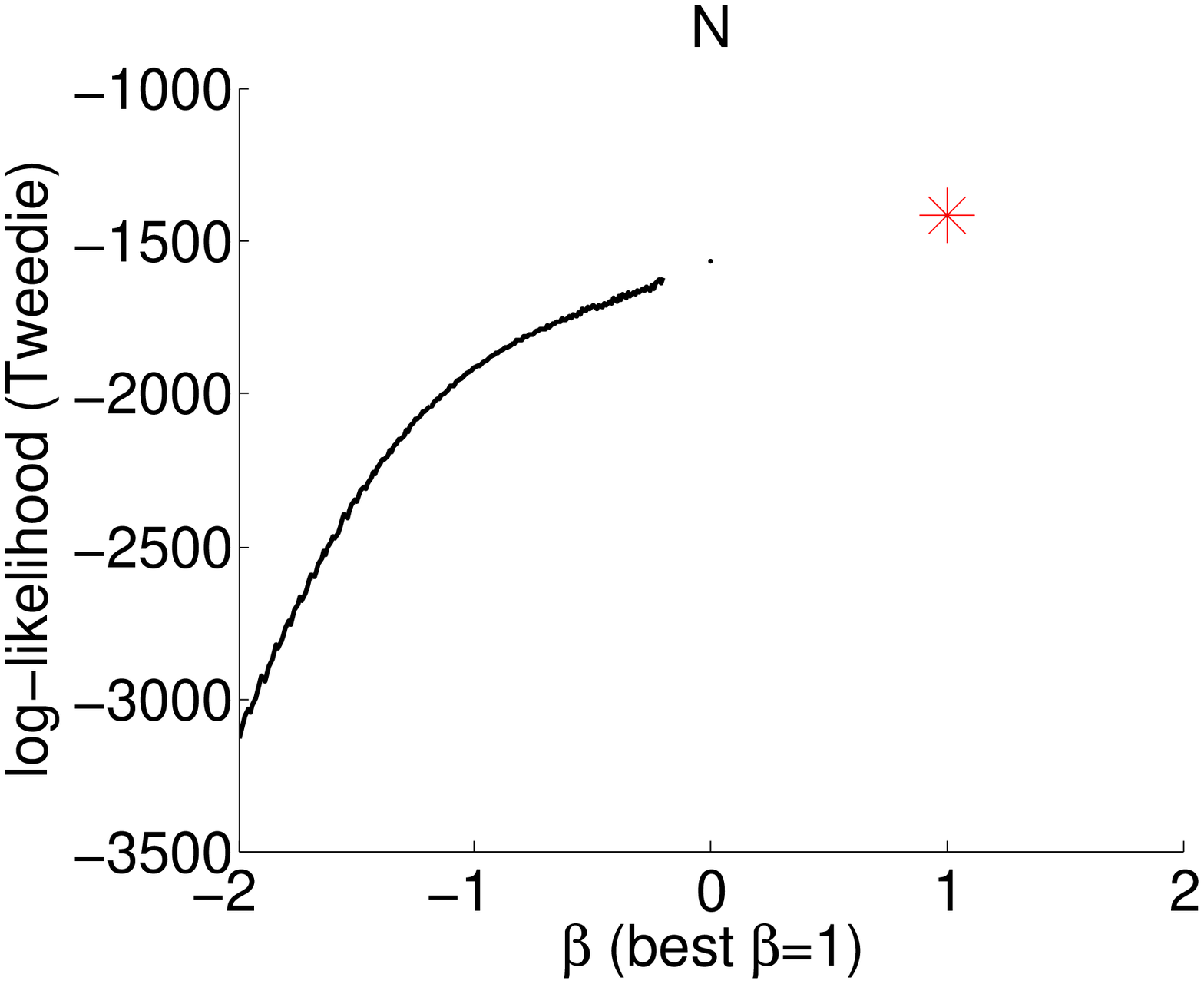,width=4cm}\\
  \epsfig{figure=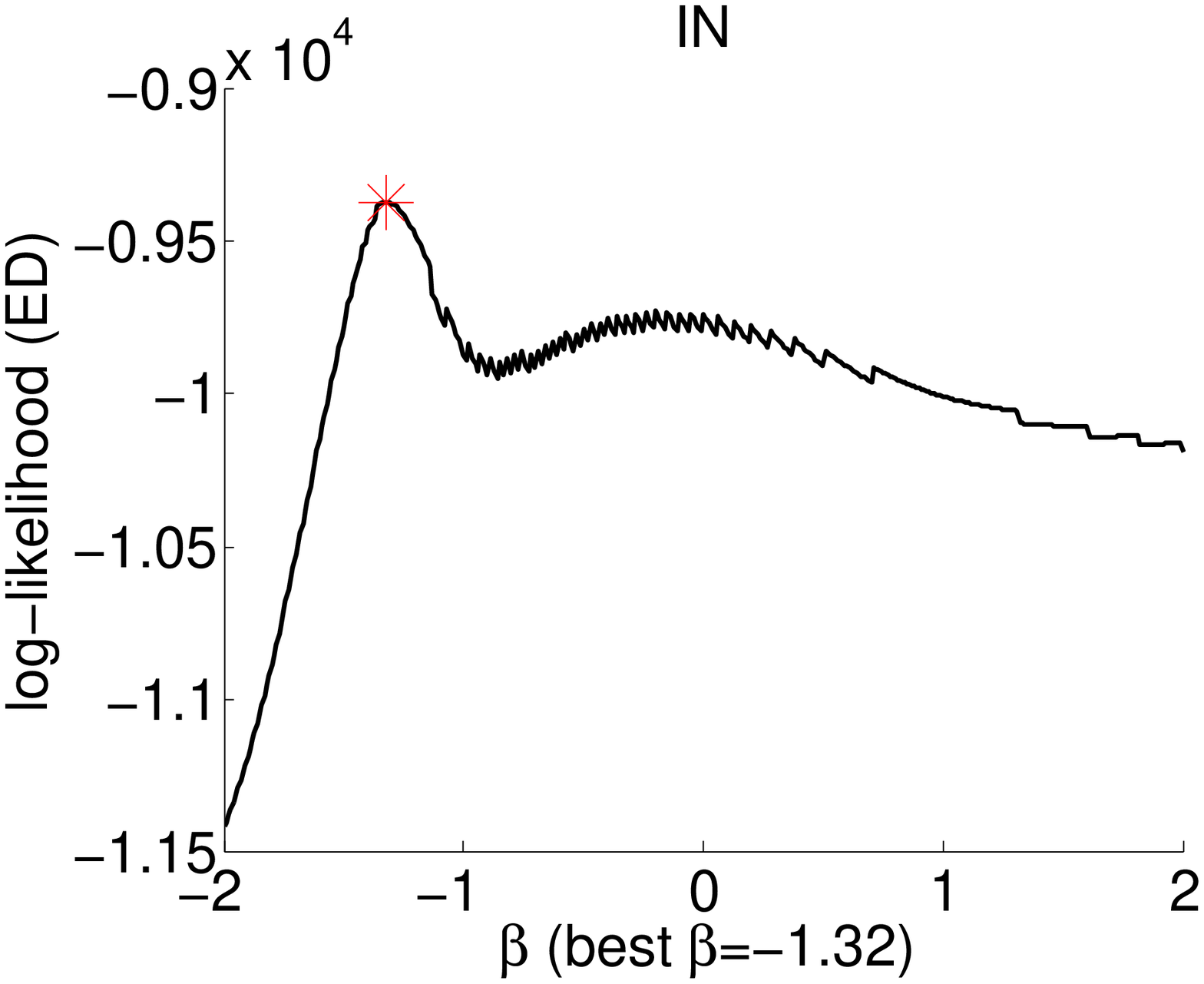,width=4cm}&
  \epsfig{figure=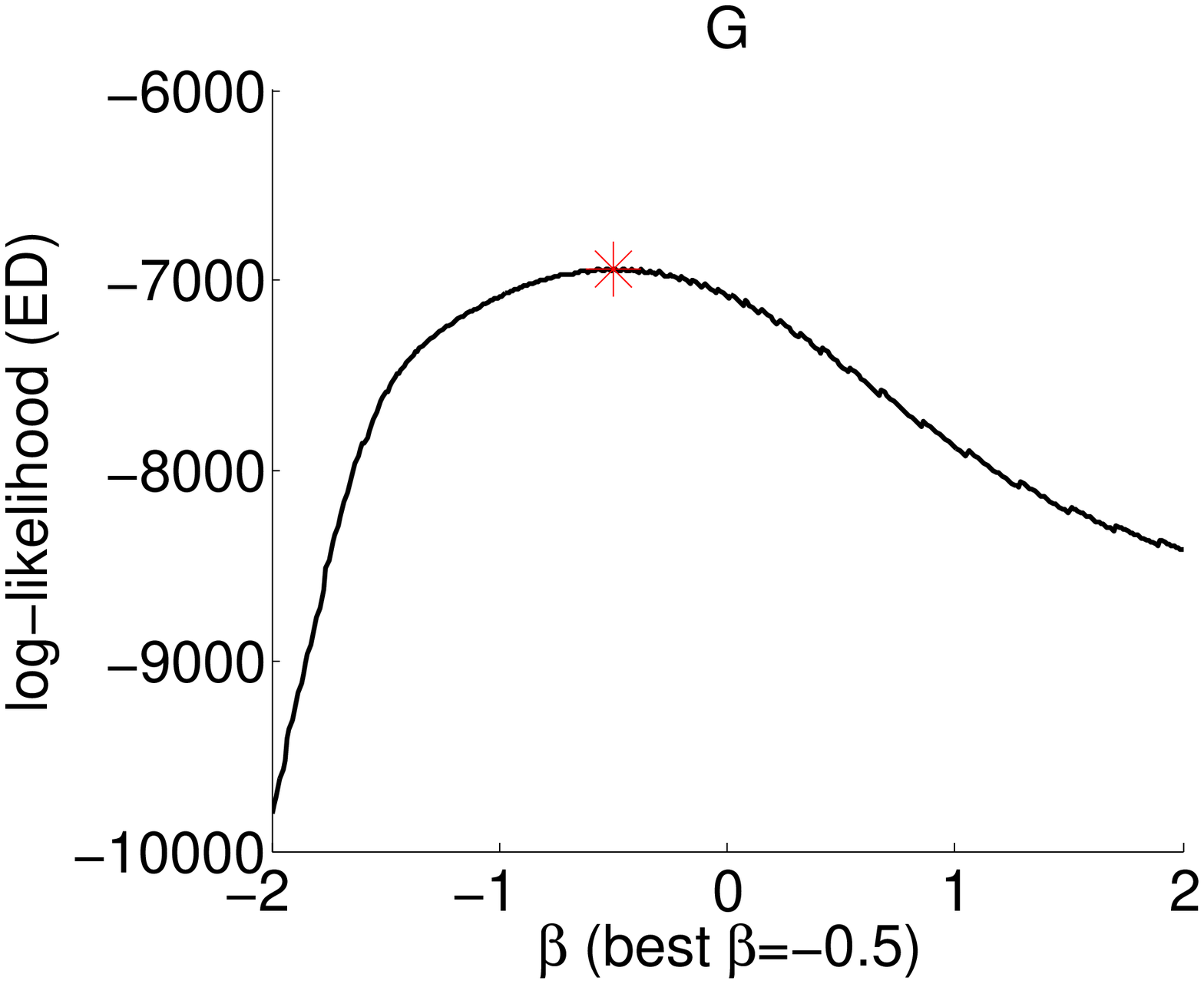,width=4cm}&
  \epsfig{figure=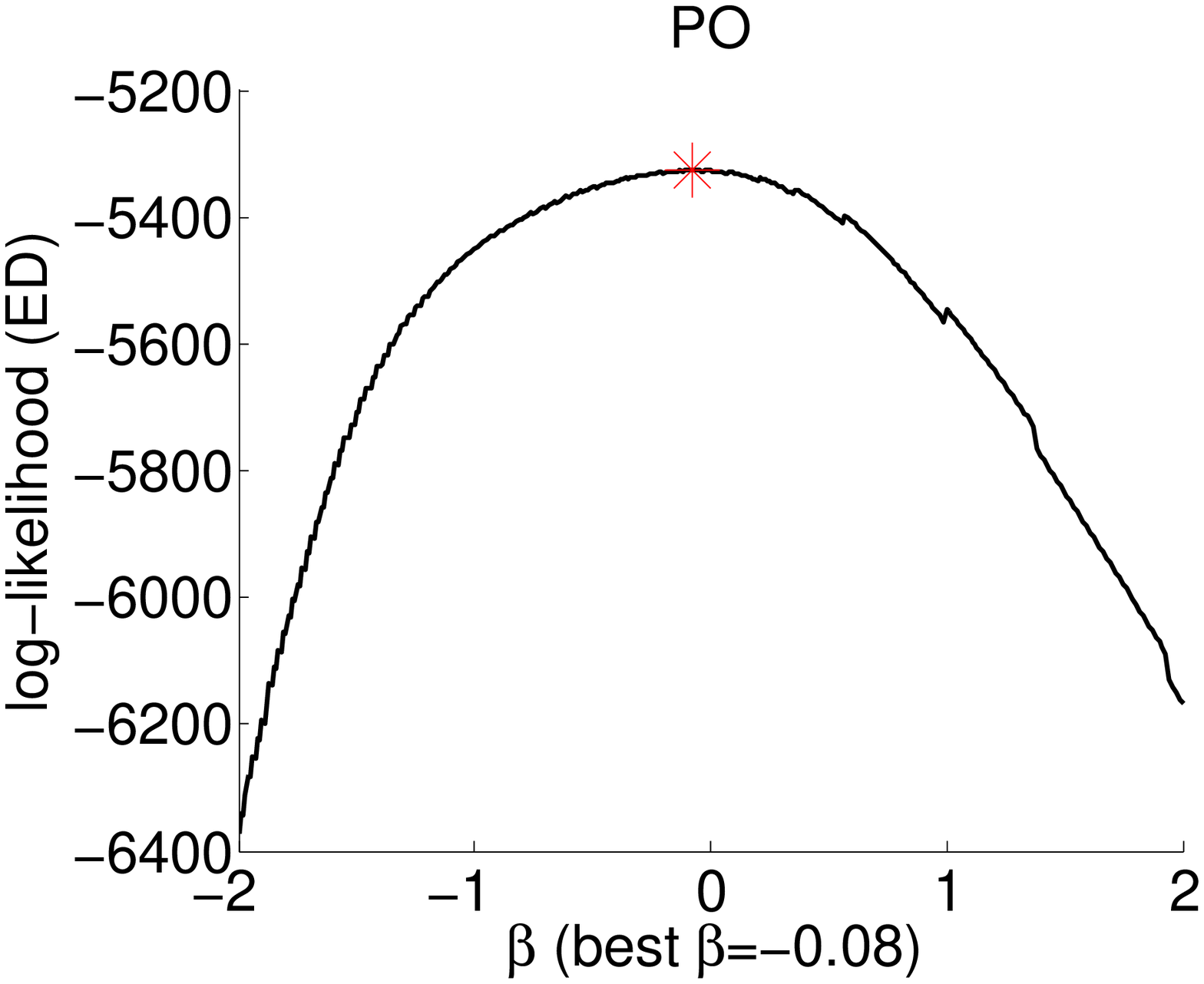,width=4cm}&
  \epsfig{figure=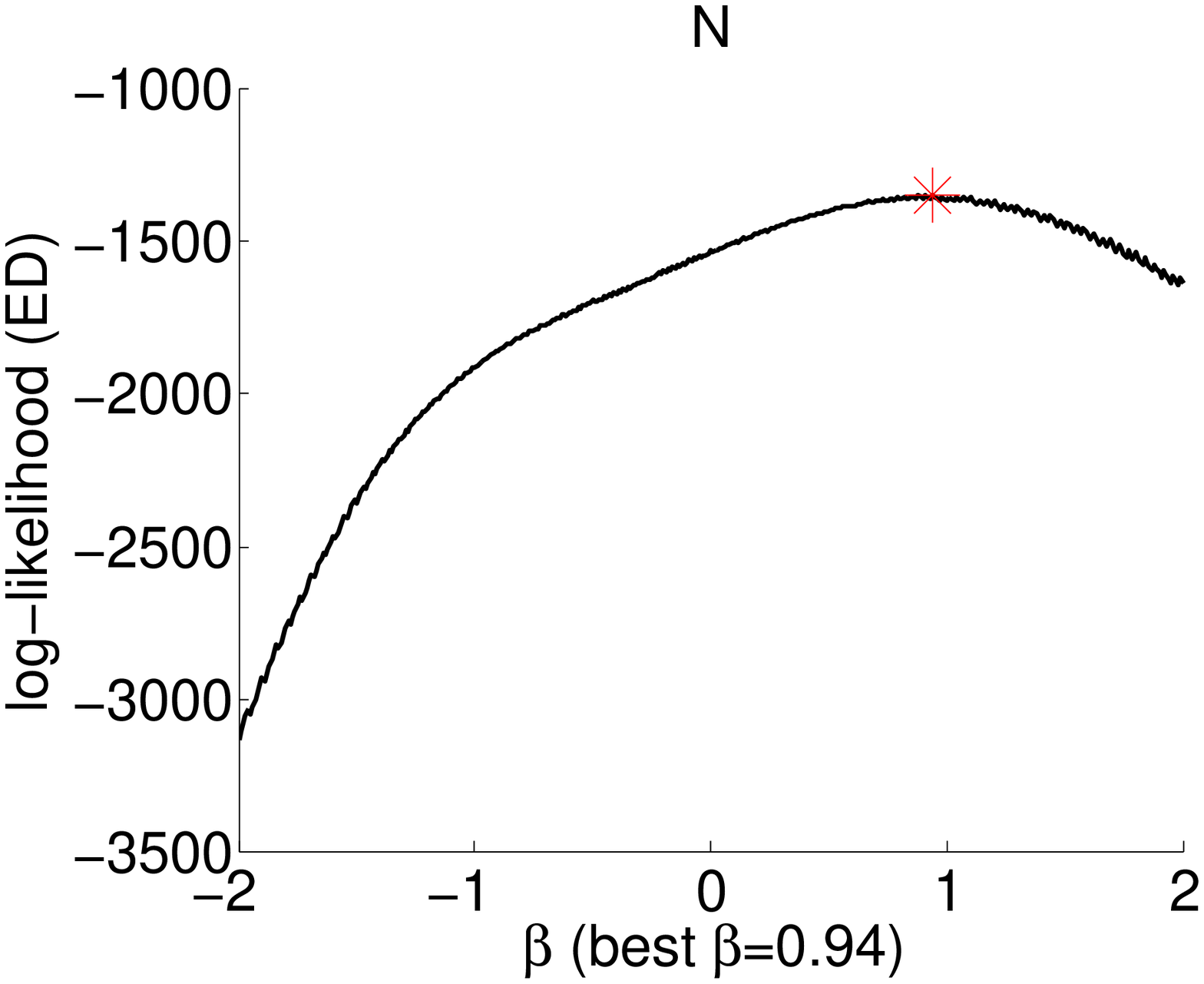,width=4cm}\\
  \epsfig{figure=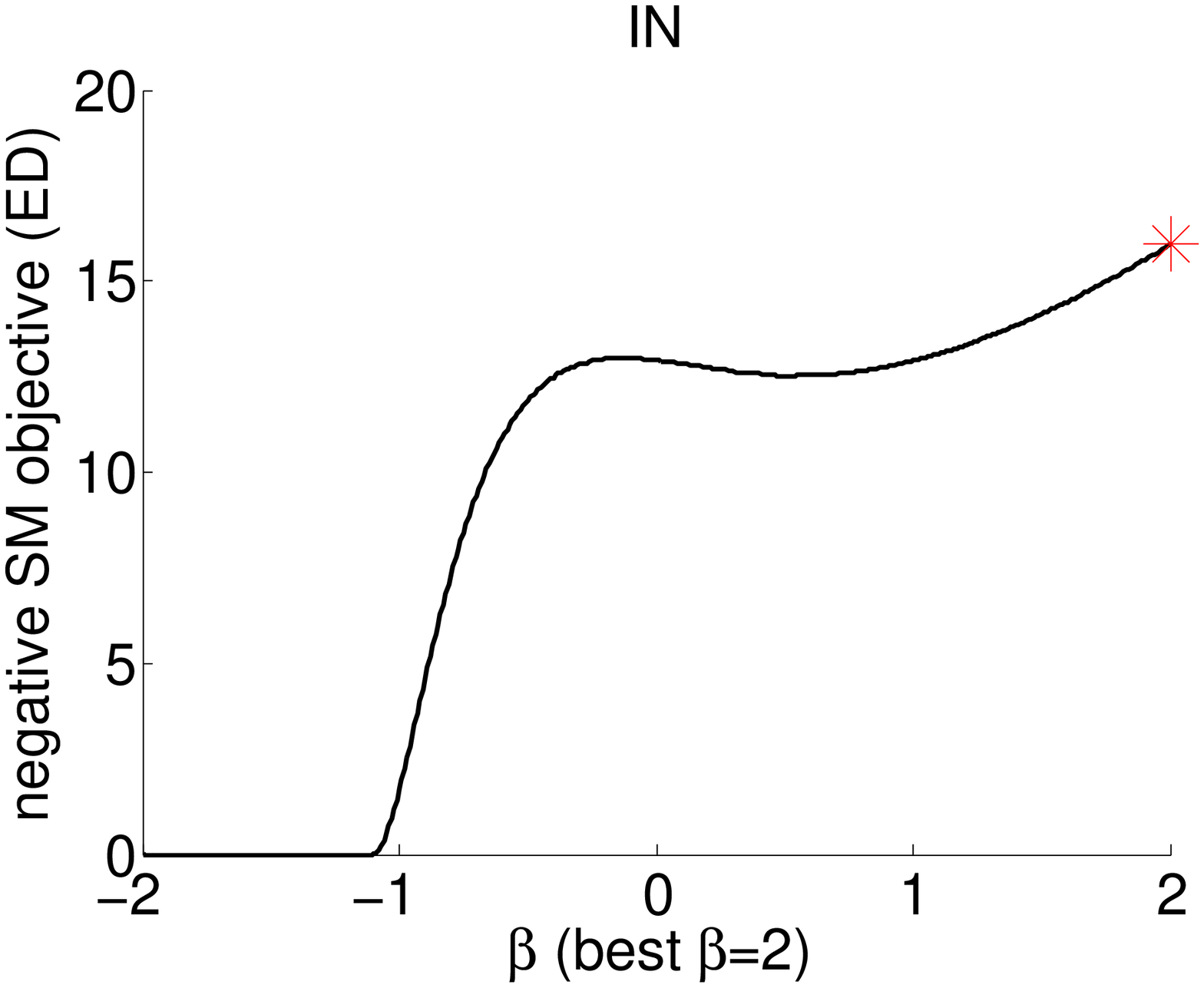,width=4cm}&
  \epsfig{figure=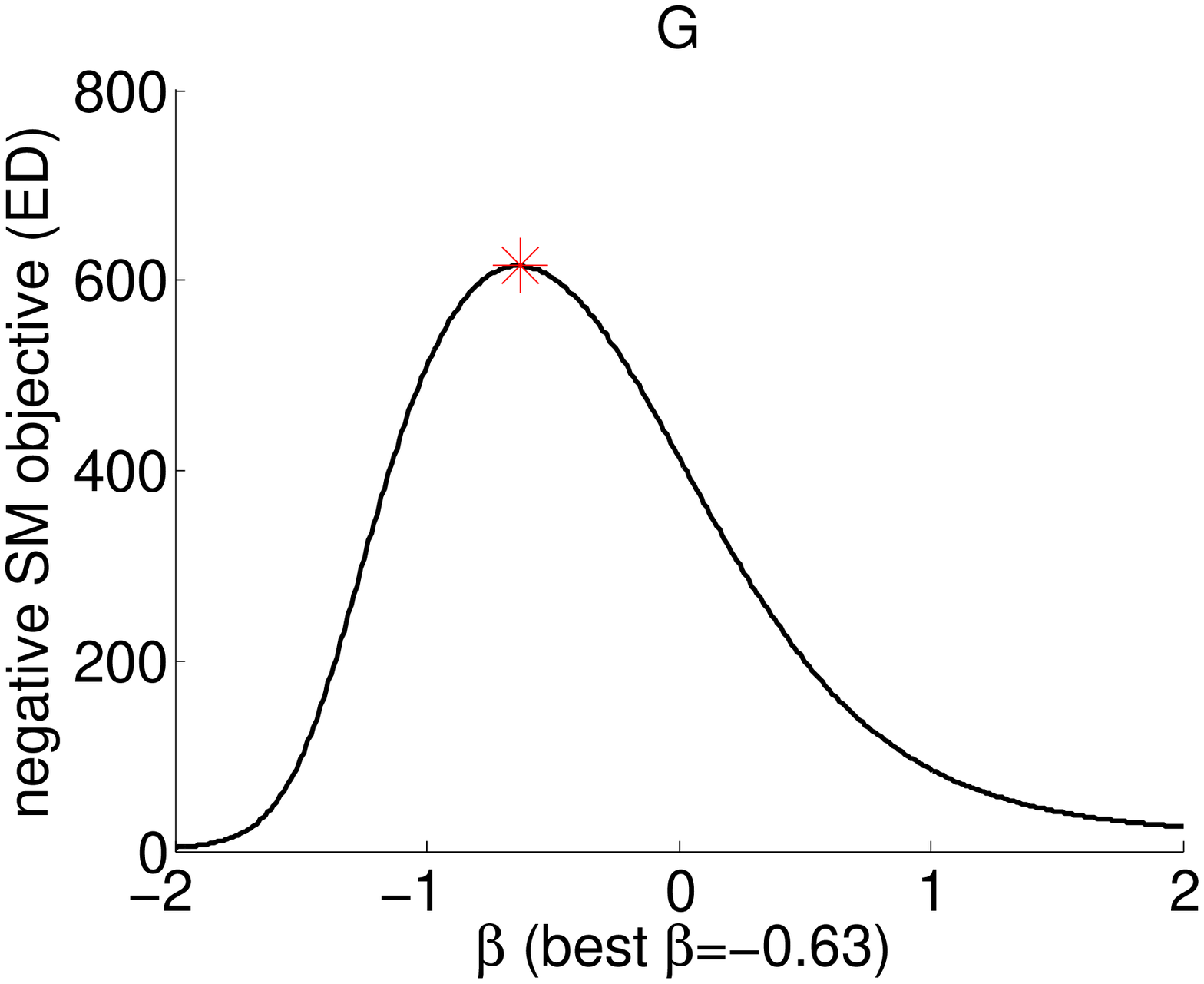,width=4cm}&
  \epsfig{figure=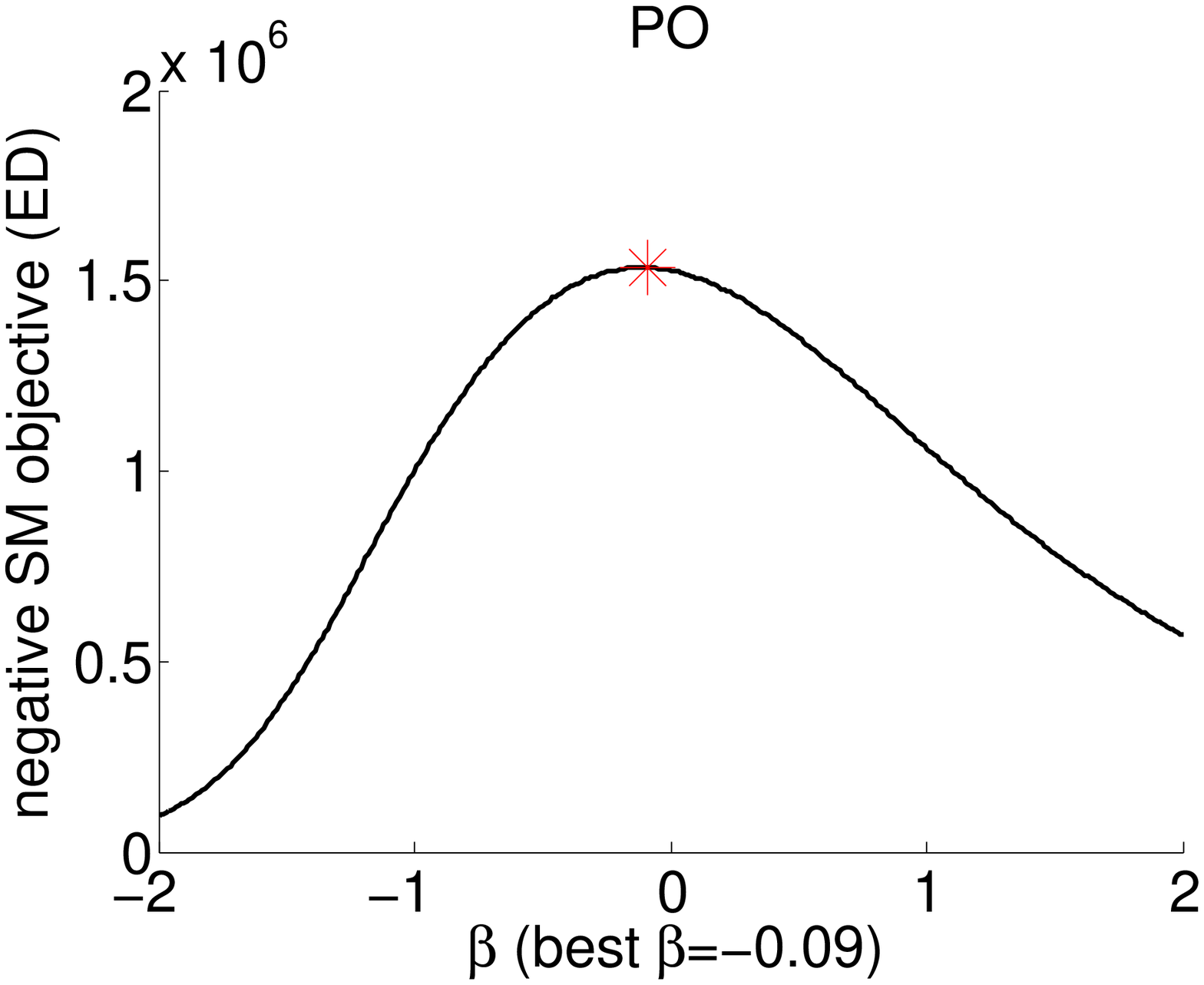,width=4cm}&
  \epsfig{figure=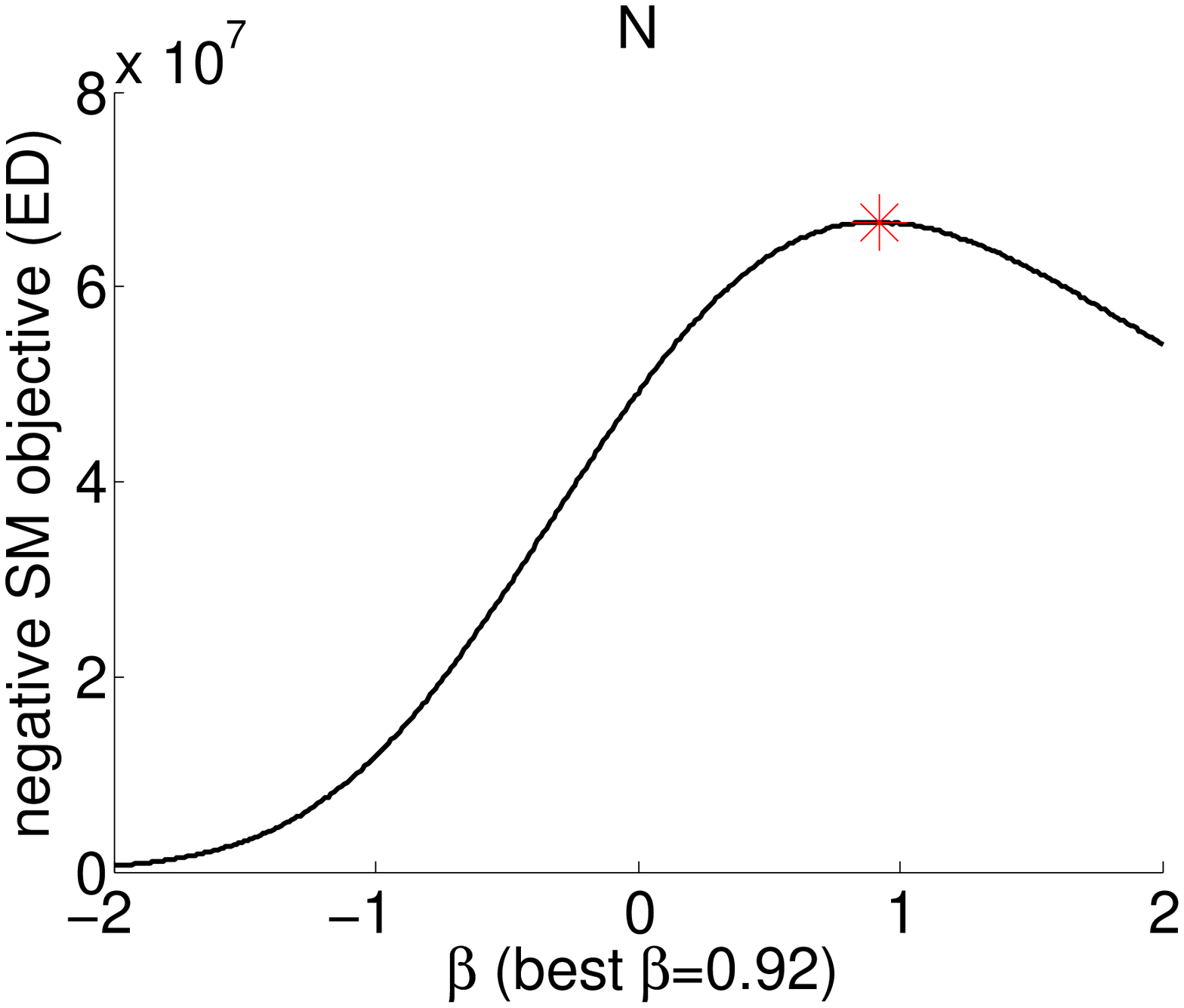,width=4cm}\\
  \epsfig{figure=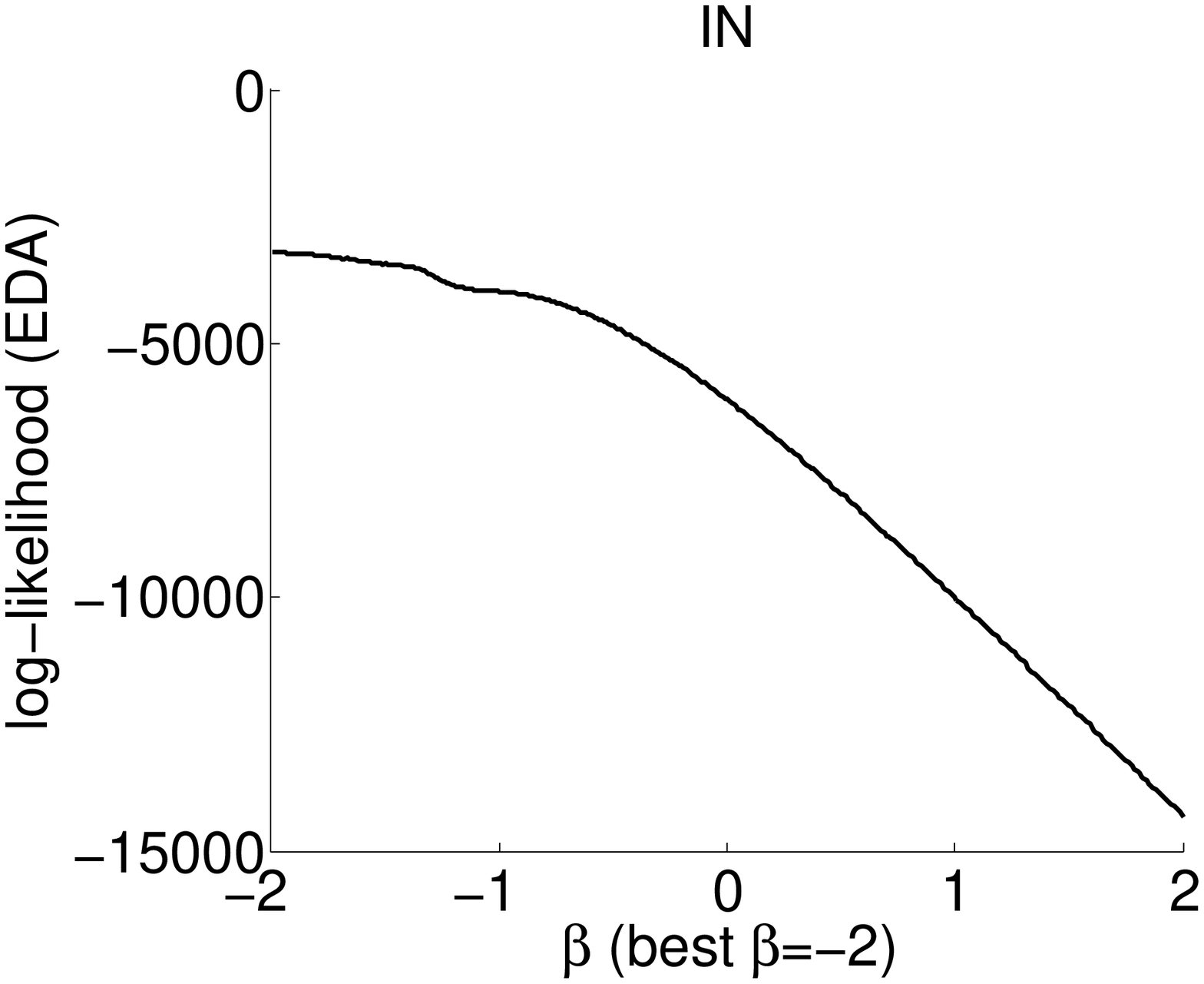,width=4cm}&
  \epsfig{figure=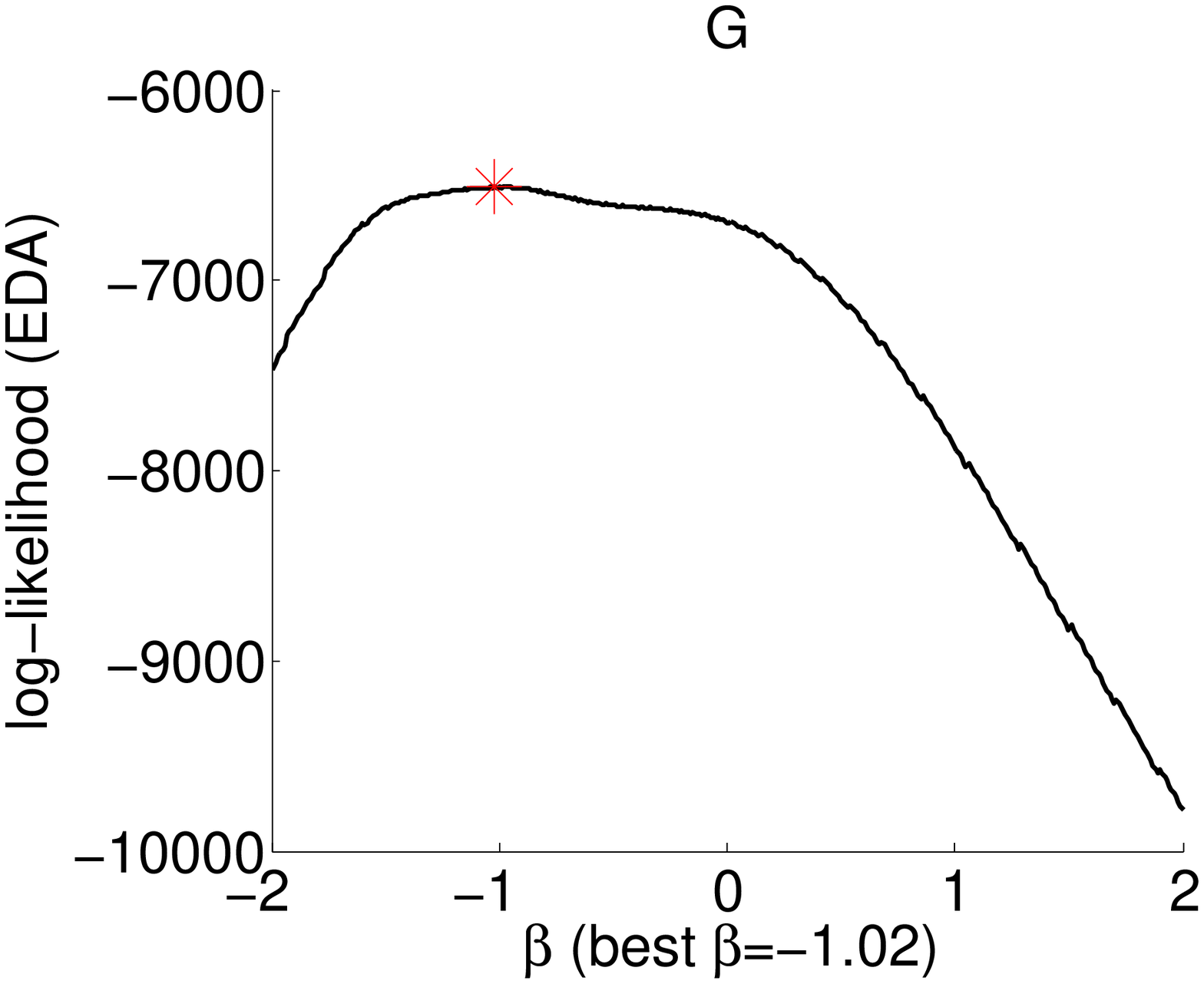,width=4cm}&
  \epsfig{figure=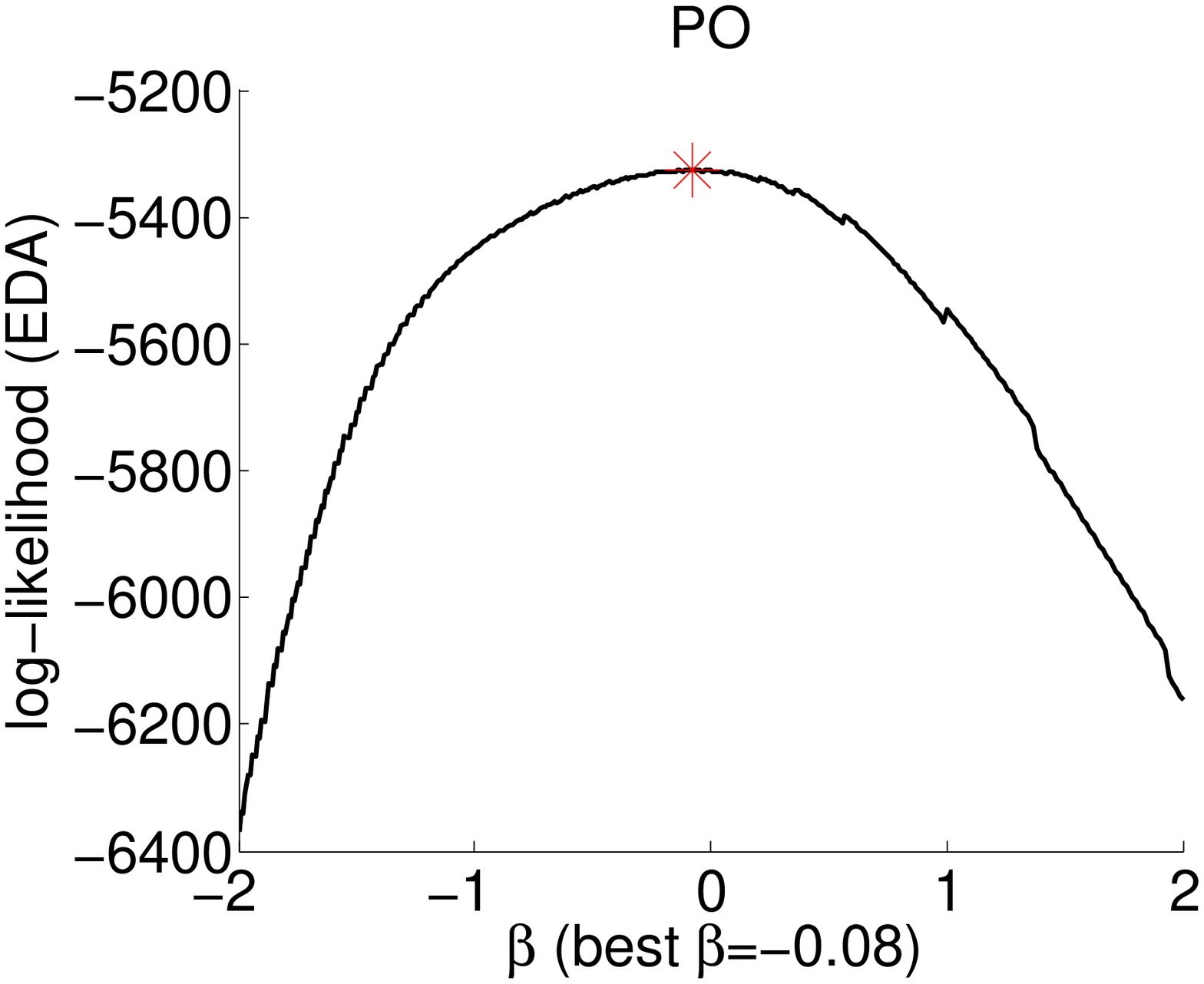,width=4cm}&
  \epsfig{figure=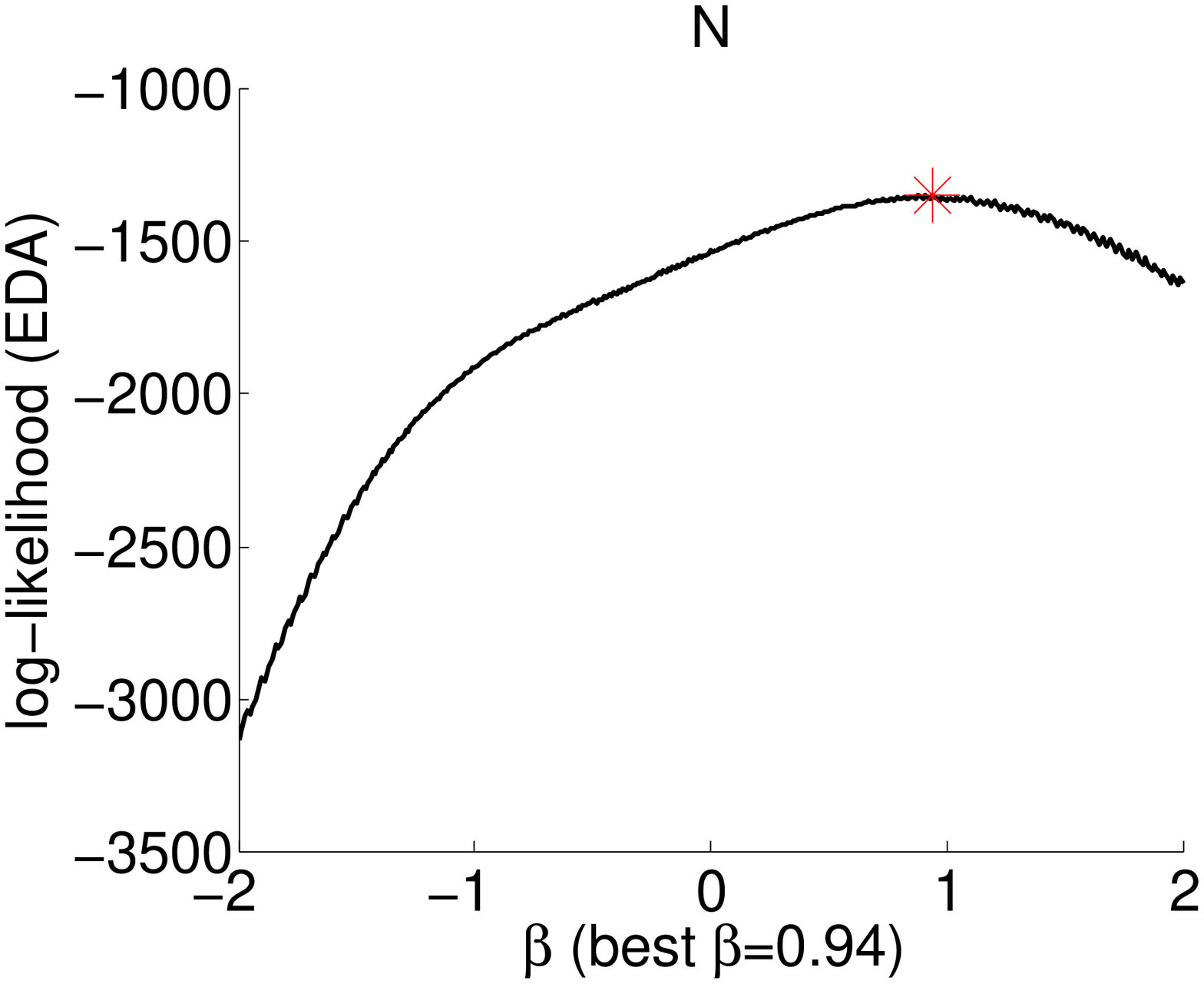,width=4cm}\\
  \epsfig{figure=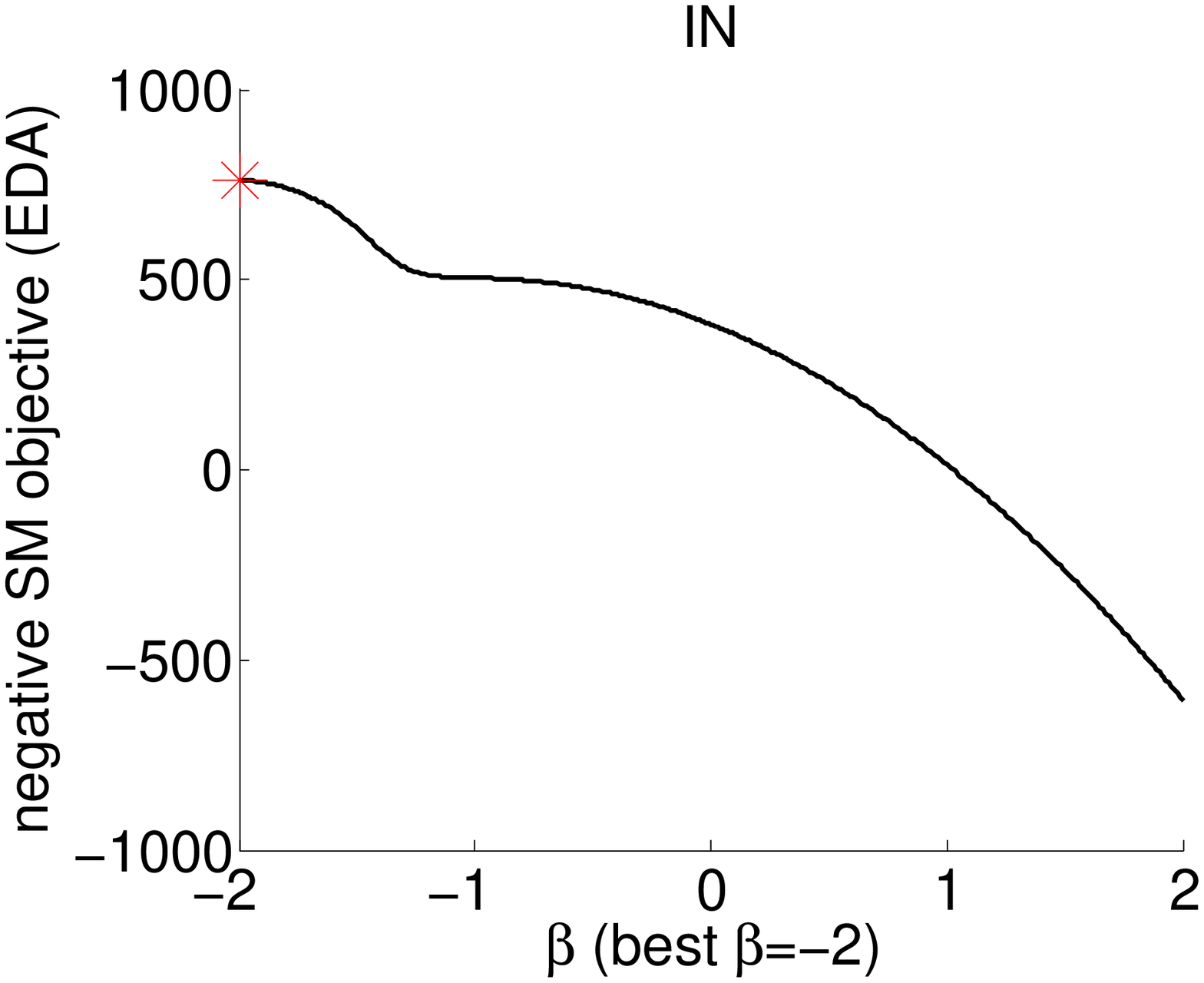,width=4cm}&
  \epsfig{figure=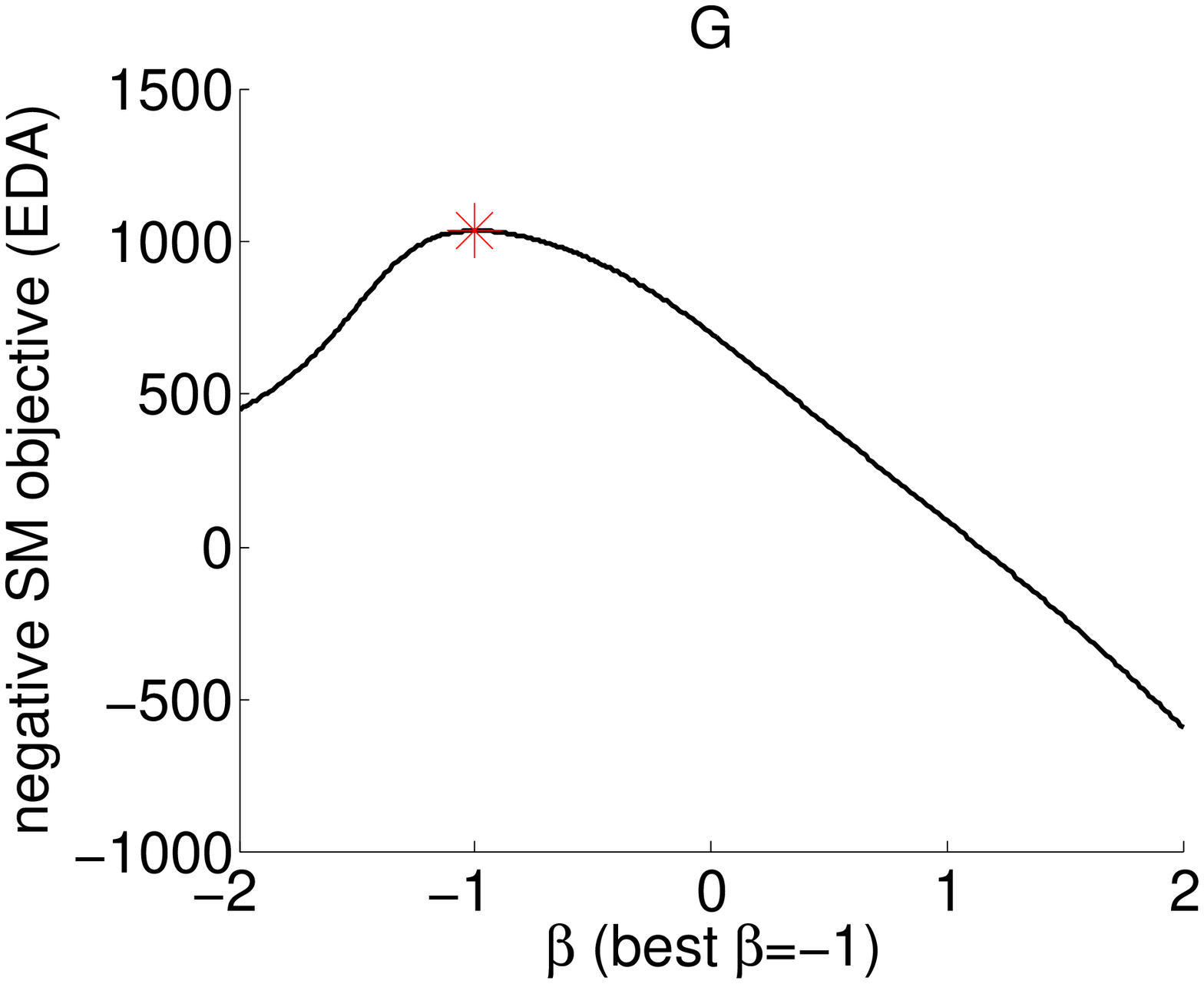,width=4cm}&
  \epsfig{figure=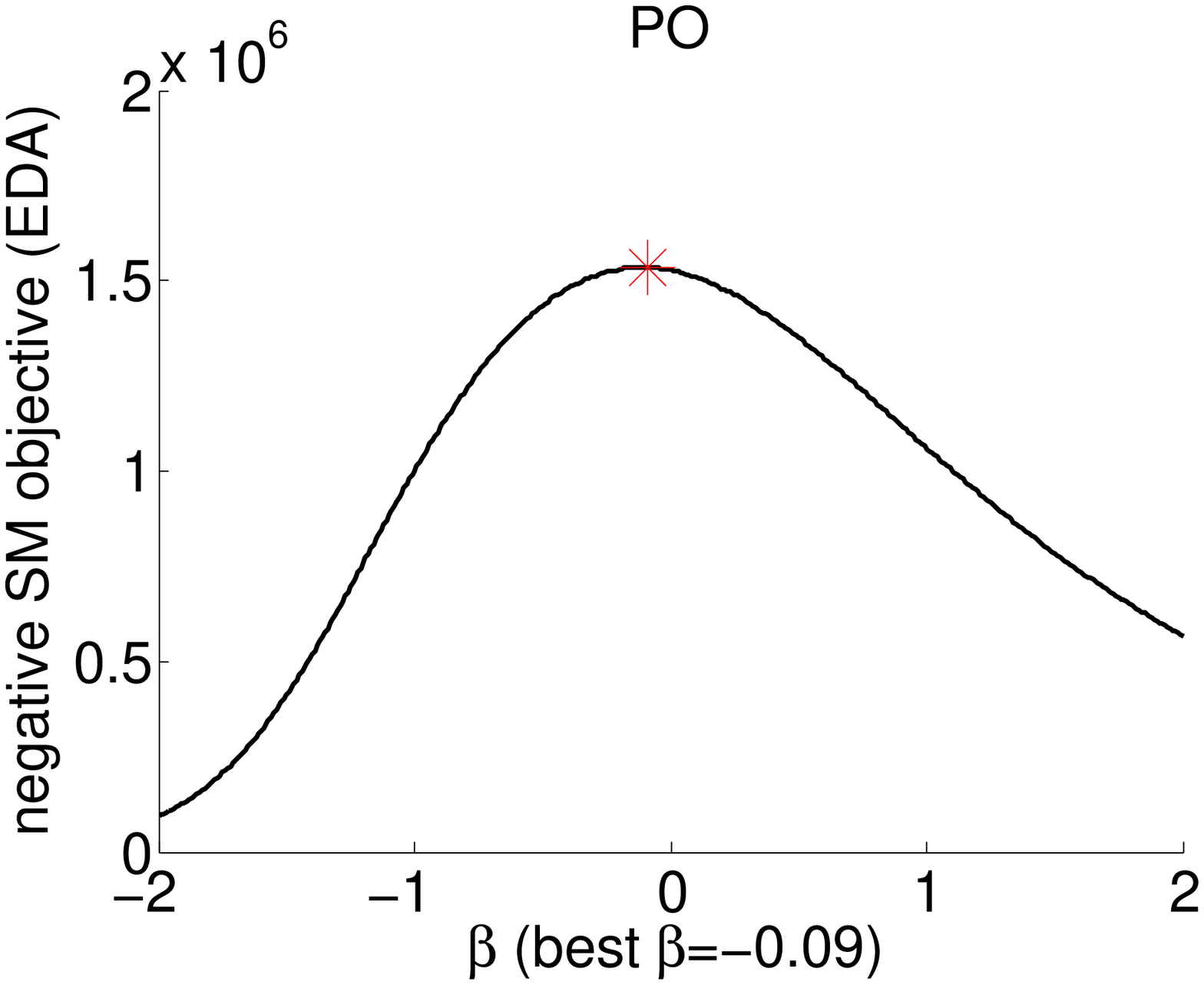,width=4cm}&
  \epsfig{figure=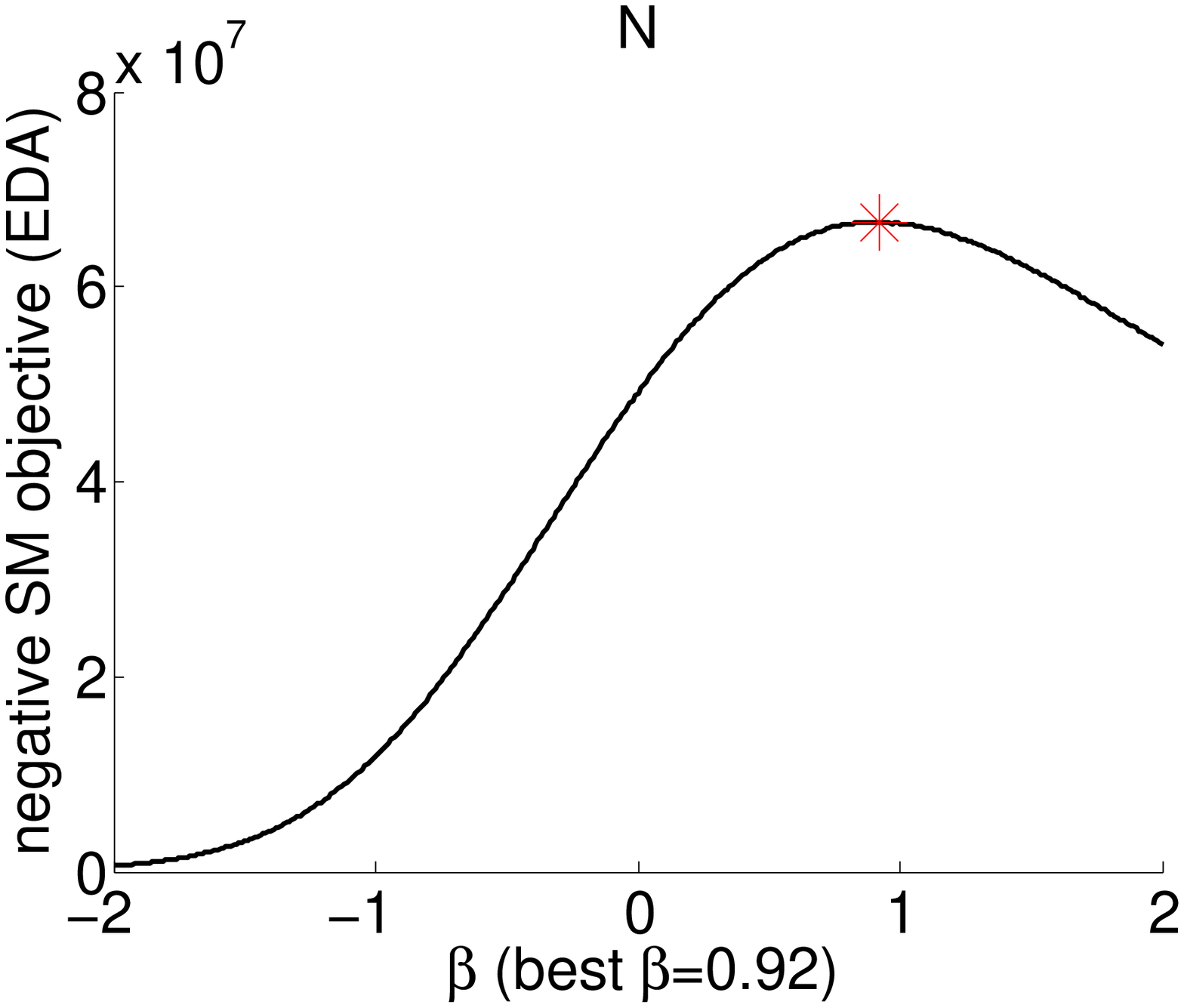,width=4cm}\\
  a) Inverse Gaussian & b) Gamma & c) Poisson & d) Gaussian\\
\end{tabular}
\end{center}
\caption{$\beta$ selection using (from top to bottom) 
Tweedie likelihood, ED likelihood, negative SM objective of ED,
EDA likelihood, and negative SM objective of EDA. Data were 
generated using Tweedie distribution with $\beta=-2,-1,0,1$ (from left to right).}
\label{fig:logL_vs_beta}
\end{figure*}

\subsection{Synthetic data}
\label{sec:expsyn}
\subsubsection{$\beta$-divergence selection}
\label{sec:expsynbeta}
We use here scalar data generated from the four special cases of Tweedie
distributions, namely, Inverse Gaussian, Gamma, Poisson, and Gaussian
distributions.  We simply fit the best Tweedie, EDA or ED density to
the data using either the
maximum likelihood method or score matching (SM). 

In Fig.~\ref{fig:logL_vs_beta} (first row), the results of the Maximum
Tweedie Likelihood (MTL) are shown. The $\beta$ value
that maximizes the likelihood in Tweedie distribution is consistent
with the true parameters, i.e., -2, -1, 0 and 1 respectively for the
above distributions. Note that Tweedie distributions are not defined
for $\beta\in(0,1)$, but $\beta$-divergence is defined in this region,
which will lead to discontinuity in the log-likelihood over $\beta$.

The second and third rows in Fig.~\ref{fig:logL_vs_beta} present
results of the exponential divergence density ED given in
Eq.~(\ref{eq:ED}). The  
log-likelihood and negative score matching objectives
\cite{lu12selecting} on the same four datasets are shown. The estimates are
consistent with the ground truth Gaussian and Poisson data. However,
for Gamma and Inverse Gaussian data, both $\beta$ estimates deviate
from the ground truth.  Thus, estimators based on ED do not give
as accurate estimates as the MTL method.
The ED distribution \cite{lu12selecting} has an advantage that it is
defined also for $\beta\in(0,1)$. In the above, we have seen that
$\beta$ selection by using ED is accurate when $\beta\rightarrow0$ or
$\beta=1$. However, as explained in Section~\ref{sec:medal}, in the
other cases ED and Tweedie distributions are not the same because the
terms containing the observed variable in these distributions are not
exactly the same as those of the Tweedie distributions.

EDA, the augmented ED density introduced in
Section~\ref{sec:selectbeta}, not only has both the advantage of
continuity but also gives very accurate estimates for $\beta<0$.  The
MEDAL log-likelihood curves over $\beta$ based on EDA are given in
Fig.~\ref{fig:logL_vs_beta} (fourth row).  In the $\beta$ selection of
Eq.~(\ref{eq:MEDAL}), the $\phi$ value that maximizes the likelihood
with $\beta$ fixed is found by a grid search.  The likelihood values
are the same as those of special Tweedie distributions and there are
no abrupt changes or discontinuities in the likelihood surface. We
also estimated $\beta$ for the EDA density using Score Matching, and
curves of the negative SM objective are presented in the bottom row of
Fig.~\ref{fig:logL_vs_beta}. They also recover the ground truth
accurately.

\subsubsection{$\alpha$-divergence selection}
\label{sec:expsynalpha}
There is only one known generative
model for which the maximum likelihood estimator corresponds to the
minimizer of the corresponding $\alpha$ divergence. It is the Poisson distribution. 
\begin{figure*}[t]
\begin{center}
\begin{tabular}{cccc}
  \epsfig{figure=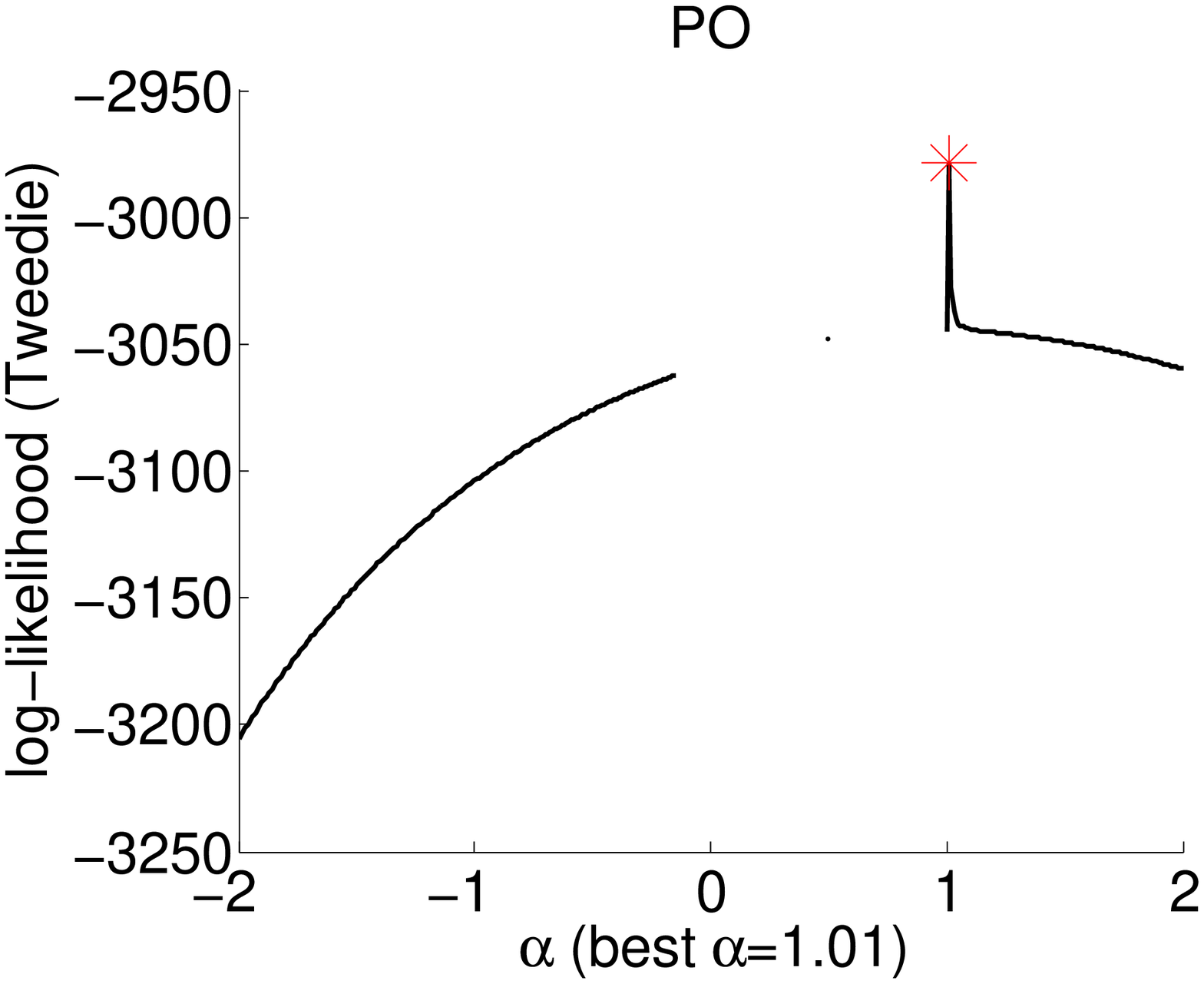,width=4cm}&
  \epsfig{figure=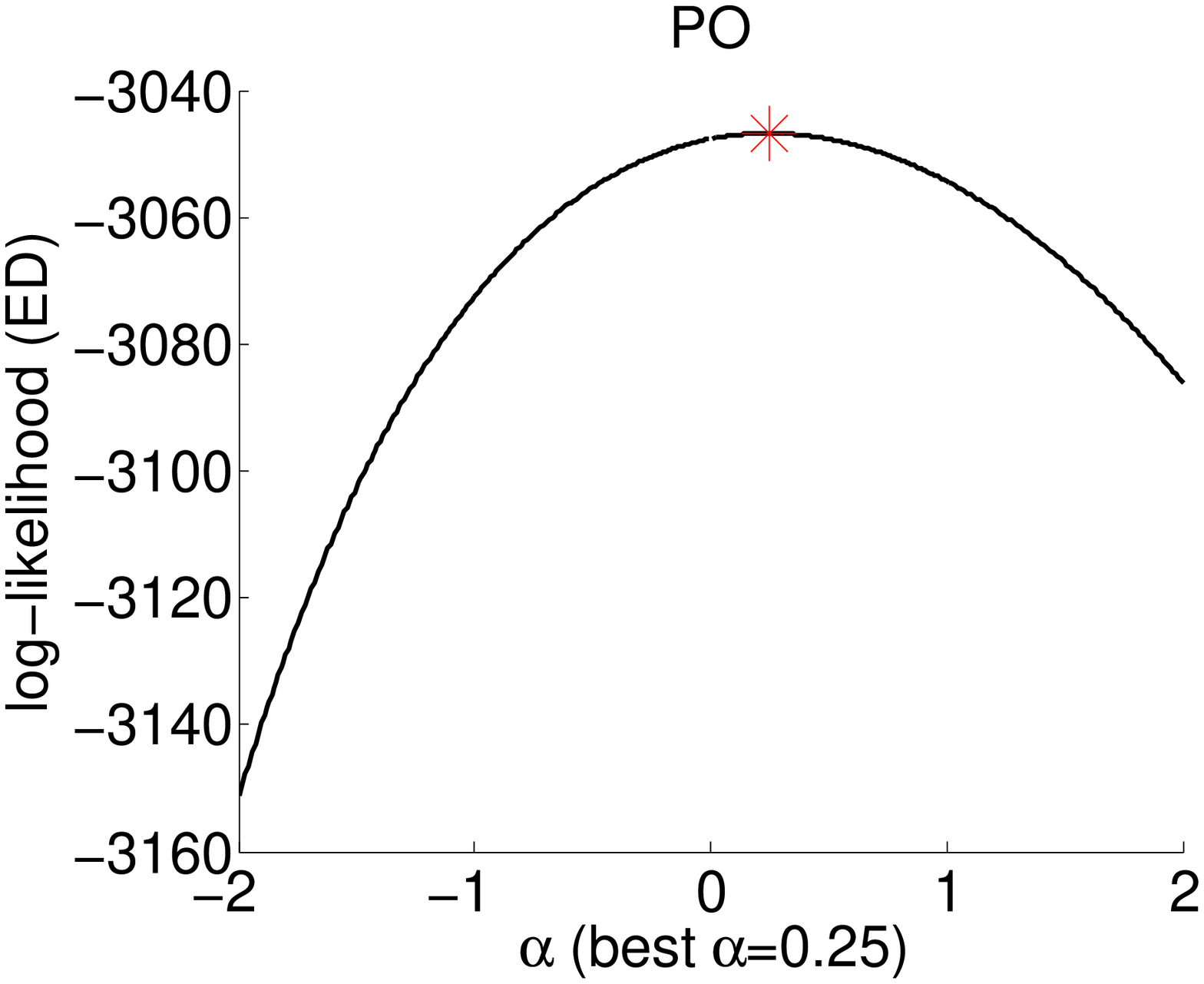,width=4cm}&
  \epsfig{figure=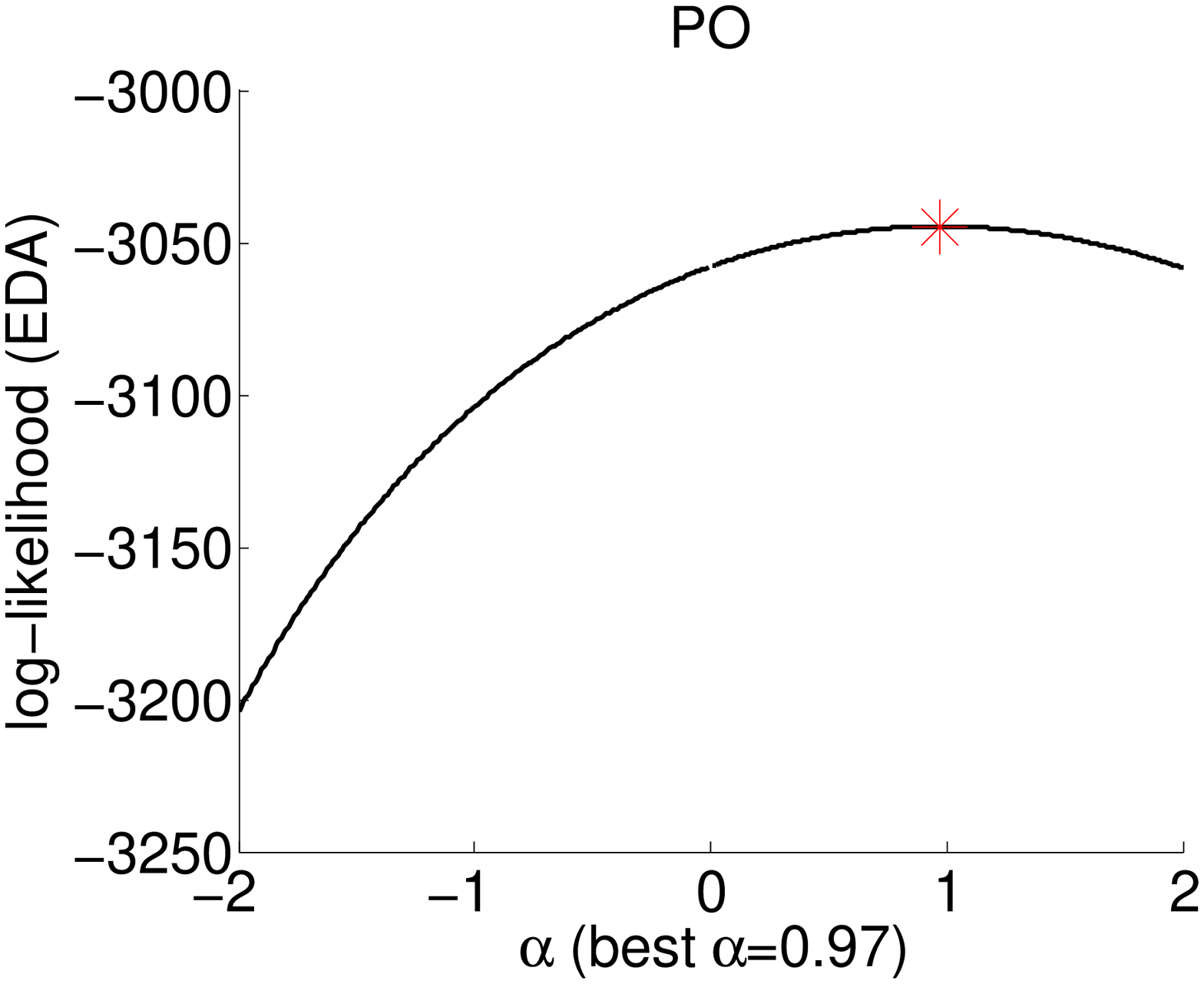,width=4cm}&
  \epsfig{figure=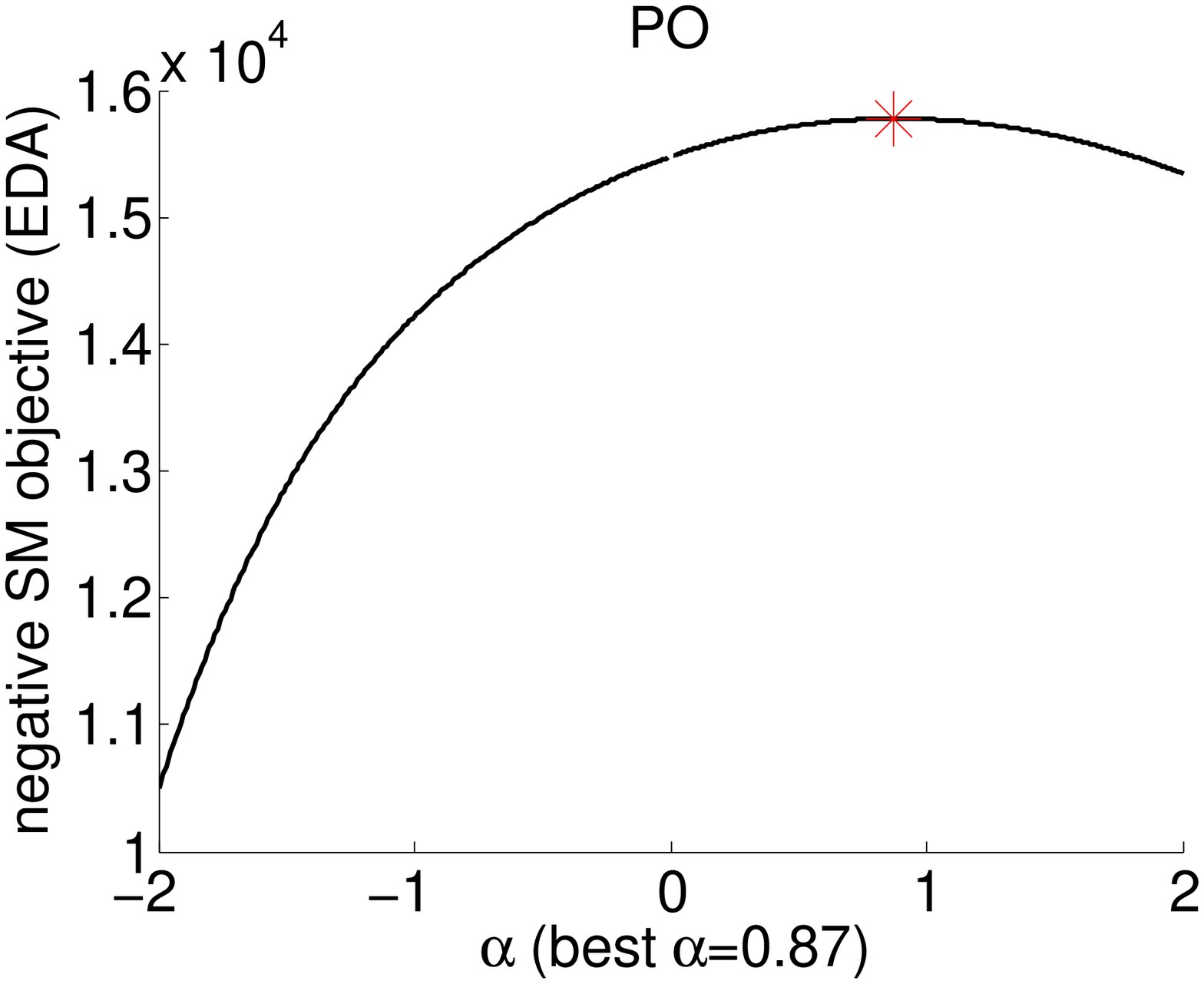,width=4cm}\\
  (a) & (b) & (c) & (d)
\end{tabular}
\end{center}
\caption{Log-likelihood of (a) Tweedie, (b) ED, and (c)
EDA distributions for $\alpha$-selection. In the Tweedie plot, blanks correspond to $\beta=1/\alpha-1$ values for which a Tweedie distribution pdf does not exist or cannot be evaluated, i.e., $\beta\in(0,1)\cup(1,\infty)$. In (d), negative SM objective function values are plotted for EDA.}
\label{fig:logL_vs_alpha}
\end{figure*}
We thus reused the Poisson-distributed  data of the previous
experiments with  the $\beta$-divergence.
In Fig.~\ref{fig:logL_vs_alpha}a, we present the log-likelihood
objective over $\alpha$ obtained with Tweedie distribution (MTL) and the
transformation from Section~\ref{sec:selectalpha}. The ground truth
$\alpha\rightarrow1$ is successfully recovered with MTL. However,
there are no likelihood estimates for $\alpha\in(0.5,1)$,
corresponding to $\beta\in(0,1)$ for which no Tweedie distributions
are defined.  Moreover, to our knowledge there are no studies
concerning the pdf's of Tweedie distributions with $\beta>1$. For that
reason, the likelihood values for $\alpha\in[0,0.5)$ are left blank in
the plot.

It can be seen from Fig.~\ref{fig:logL_vs_alpha}b and \ref{fig:logL_vs_alpha}c, that the augmentation in the MEDAL method also helps in
$\alpha$ selection. Again, both ED and EDA solve most of the
discontinuity problem except $\alpha=0$. Selection using ED fails to
find the ground truth which equals 1, which is however successfully
found by the MEDAL method. SM on EDA recovers the ground truth as well (Fig.~\ref{fig:logL_vs_alpha}d).
%%%%%%%%%%%%%%%%%%%%%%%%%%%%%%%%%%%%%%%%%%%%%%%%
\subsection{Divergence selection in NMF}
\label{sec:expnmf}
The objective in nonnegative matrix factorization (NMF) is to find a
low-rank approximation to the observed data by expressing it as a
product of two nonnegative matrices, i.e., $\matV\approx \Vh = \matW \matH$
with $\matV\in \bbR_+^{F\times N}$ , $\matW\in \bbR_+^{F\times K}$ and $\matH\in
\bbR_+^{K\times N}$. This objective is pursued through the minimization of
an information divergence between the data and the approximation, i.e.,
$D(\matV||\Vh)$. The divergence can be any appropriate one for the
data/application such as $\beta$, $\alpha$, $\gamma$, R\'{e}nyi, etc.
Here, we chose the $\beta$ and $\alpha$ divergences to illustrate the
MEDAL method for realistic data.  

The optimization of $\beta$-NMF was implemented using
the standard multiplicative update
rules~\cite{cichocki09nonnegative,fevotte11algorithms}.
Similar multiplicative update rules are also available for
$\alpha$-NMF~\cite{cichocki09nonnegative}. Alternatively, the
algorithm for $\beta$-NMF can be used for $\alpha$-divergence
minimization as well, using the transformation explained in
Section~\ref{sec:selectalpha}.

%%%%%%%%%%%%%%%%%%%%%%%%%%%%%%%%%%%%%%%%%%%%%%%%
\subsubsection{A Short Piano Excerpt}
\label{sec:expnmfpiano}
We consider the piano data used in~\cite{fevotte09nonnegative}. It is
an audio sequence recorded in real conditions, consisting of four
notes played all together in the first measure and in all possible
pairs in the subsequent measures. A power spectrogram with analysis
window of size 46~ms was computed, leading to $F=513$ frequency bins
and $N=676$ time frames. These make up the data matrix $\matV$, for
which a matrix factorization $\Vh=\matW\matH$ with low rank $K = 6$ is sought for. 

In Fig.~\ref{fig:piano_logLs}a and \ref{fig:piano_logLs}b, we show the
log-likelihood values of the MEDAL method for $\beta$ and $\alpha$,
respectively. For each parameter value $\beta$ and $\alpha$, the
multiplicative algorithm for the respective divergence is run for 100
iterations and likelihoods are evaluated with mean values calculated
from the returned matrix factorizations. For each value of $\beta$ and
$\alpha$, the highest likelihood w.r.t. $\phi$ (see
Eq.~(\ref{eq:MEDAL})) is found by a grid search.

The found maximum likelihood estimate $\beta=-1$ corresponds to Itakura-Saito divergence, which is in harmony with the empirical results presented
in~\cite{fevotte09nonnegative} and the common belief that IS
divergence is most suitable for audio spectrograms. The optimal
$\alpha$ value value was 0.5 corresponding to Hellinger distance. We
can also see that the log likelihood value associated with
$\alpha=0.5$ is still much less than the one for $\beta=-1$. SM also finds $\beta=-1$ as can be seen from Fig.~\ref{fig:piano_logLs}c.
\begin{figure}[t]
\begin{center}
\begin{tabular}{cc}
  \epsfig{figure=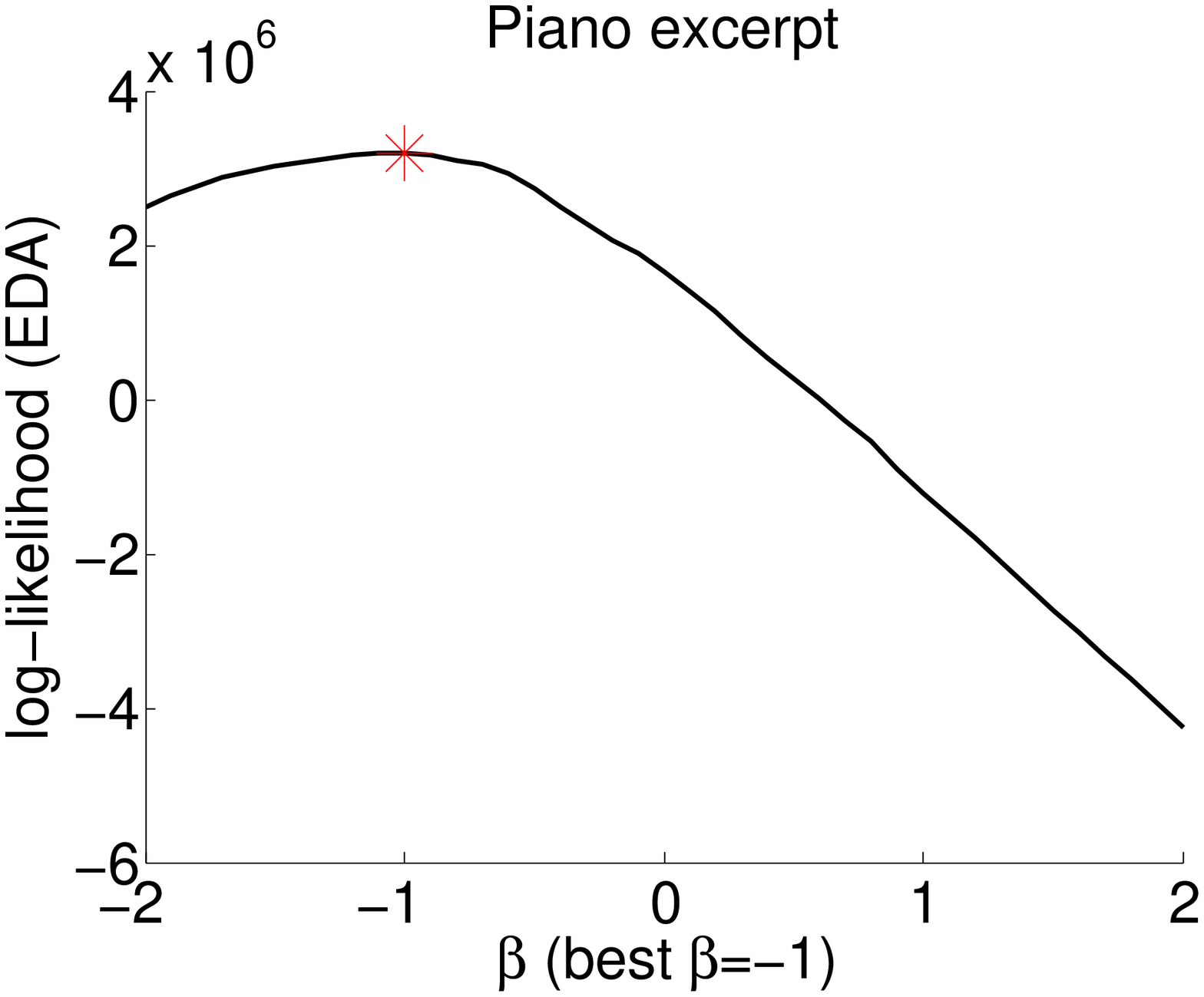,width=4cm}&
  \epsfig{figure=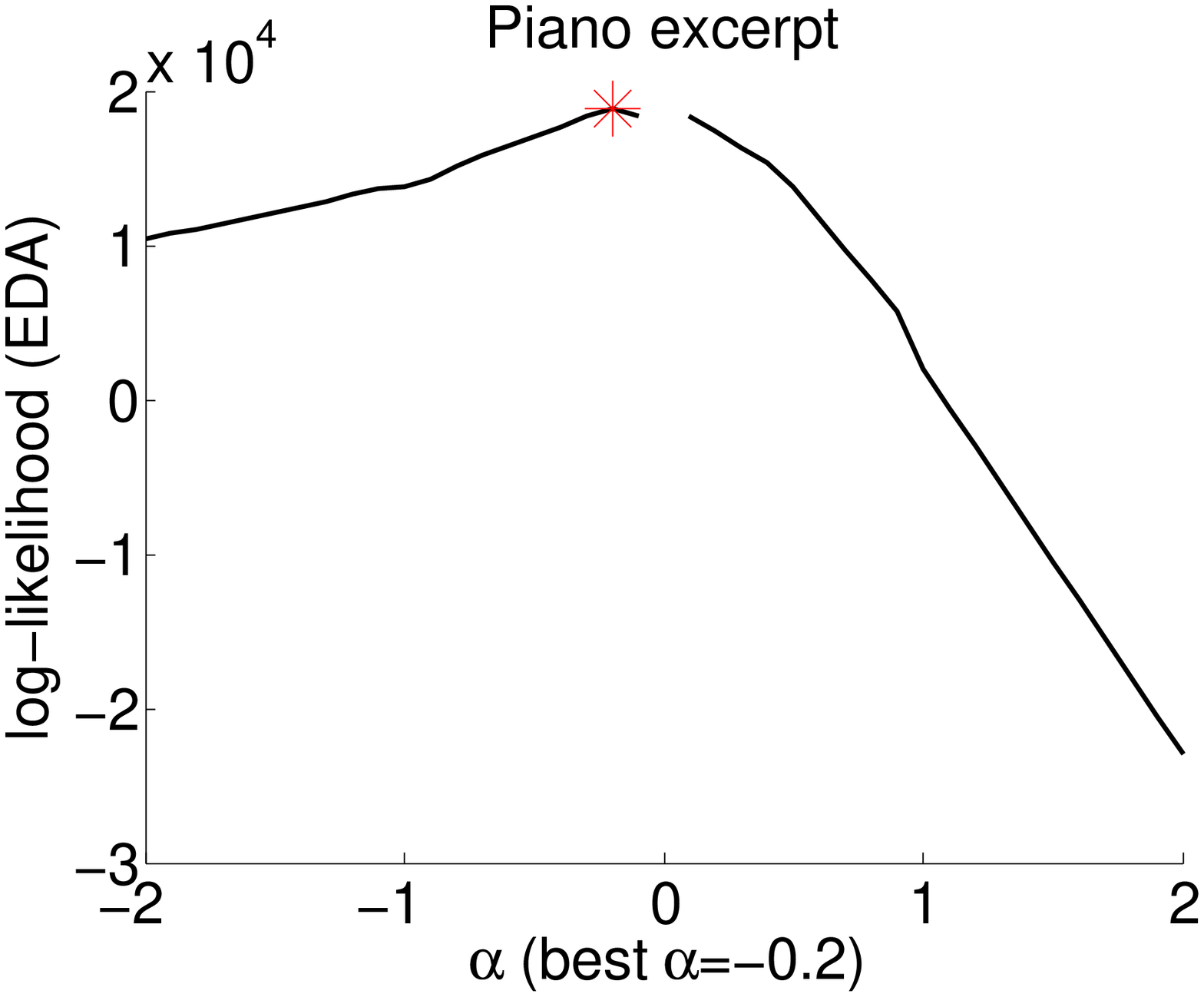,width=4cm}\\
    a) $\beta$ div. & b) $\alpha$ div.\\
  \epsfig{figure=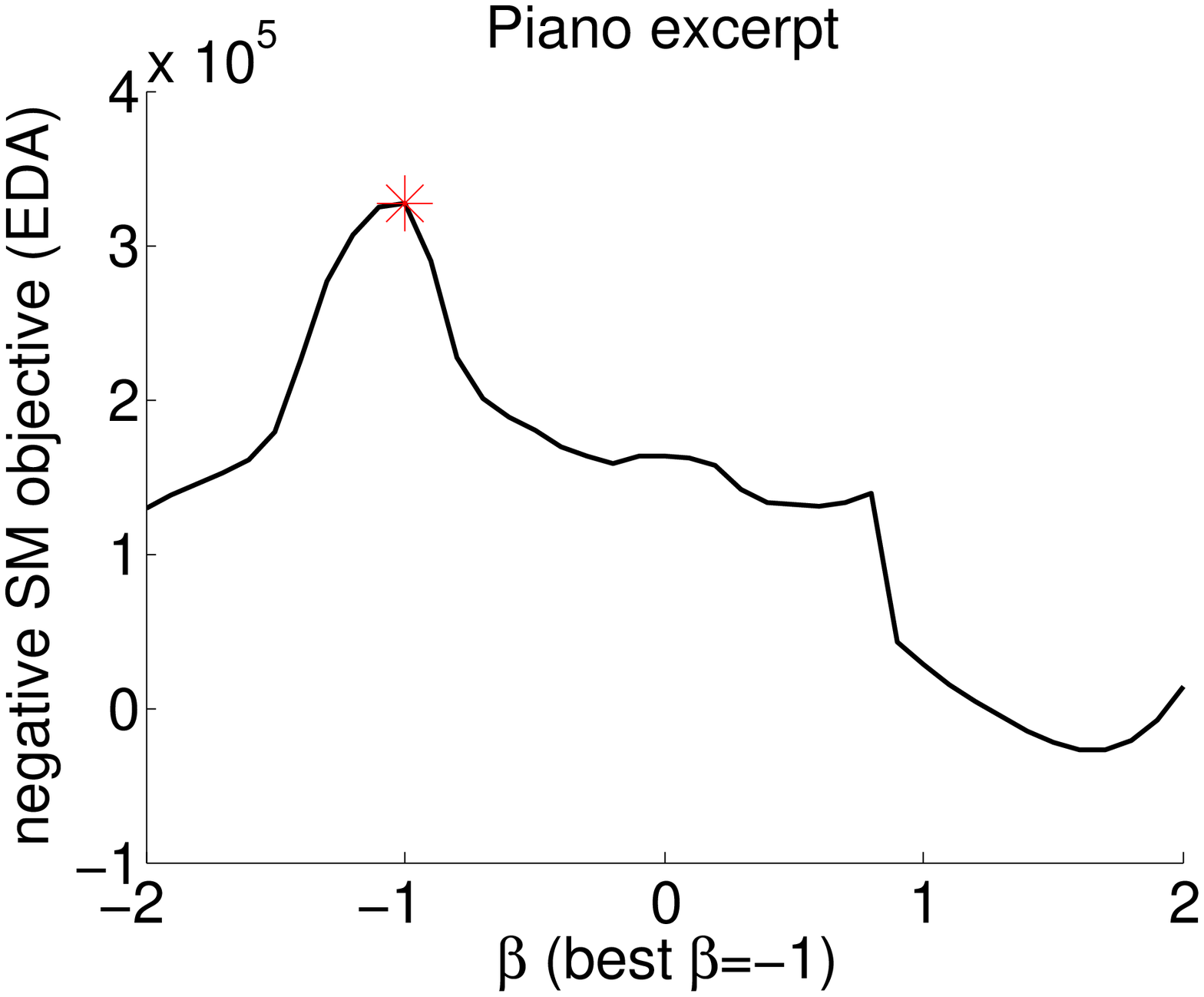,width=4cm}&\\
    c) $\beta$ div. &\\
\end{tabular}
\end{center}
\caption{(a, b) Log likelihood values for $\beta$ and $\alpha$ for the spectrogram of a short piano excerpt with $F=513$, $N=676$, $K=6$. (c) Negative SM objective for $\beta$.}
\label{fig:piano_logLs}
\end{figure}
%%%%%%%%%%%%%%%%%%%%%%%%%%%%%%%%%%%%%%%%%%%%%%%%
\subsection{Stock Prices}
Next, we repeat the same experiment on a stock price dataset which
contains Dow Jones Industrial Average. There are 30 companies included
in the data. They are major American companies from various sectors
such as services (e.g.,\@ Walmart), consumer goods (e.g., General Motors)
and healthcare (e.g.,\@ Pfizer). The data was collected from 3rd
January 2000 to 27th July 2011, in total 2543 trading dates. We
set $K=5$ in NMF and masked 50\% of the data by following \cite{tan2013pami}. The \texttt{stock}
data curves are displayed in Fig.~\ref{fig:stock} (left). 

The EDA likelihood curve with $\beta\in[-2,2]$ is shown in Figure~\ref{fig:stock} (bottom left). We can see that the best divergence selected
by MEDAL is $\beta=0.4$. The
corresponding best $\phi=0.006$. These results are in harmony with the findings of Tan and F\'evotte~\cite{tan2013pami} using the remaining 50\% of the data as validation set, where they found that $\beta \in [0, 0.5]$ (mind that our
$\beta$ values equal theirs minus one)
performs well for a large range of $\phi$'s. Differently, our method is
more advantageous because we do not need additional criteria nor data for
validations. In Figure~\ref{fig:stock} (bottom right), negative SM objective function is plotted for $\beta\in[-2,2]$. With SM, the optimal $\beta$ is found to be 1.
%The EDA likelihood curve with $\beta\in[-2,2]$ is shown in Figure
%\ref{fig:stock} (right). We can see that the best divergence selected
%by MEDAL is $\beta=1$ (i.e.\@ squared Euclidean distance). The
%corresponding best $\phi=0.457$. These results are very close to those
%found by Tan and F\'evotte~\cite{tan2013pami} using validations, where they found
%the same best divergence when $K=5$ and $\phi=0.1$ (mind that our
%$\beta$ values equal theirs minus one). Differently, our method is
%more advantageous because we do not need additional criteria for
%validations.

\begin{figure}[t]
\begin{center}
\center{
  \includegraphics[width=0.23\textwidth]{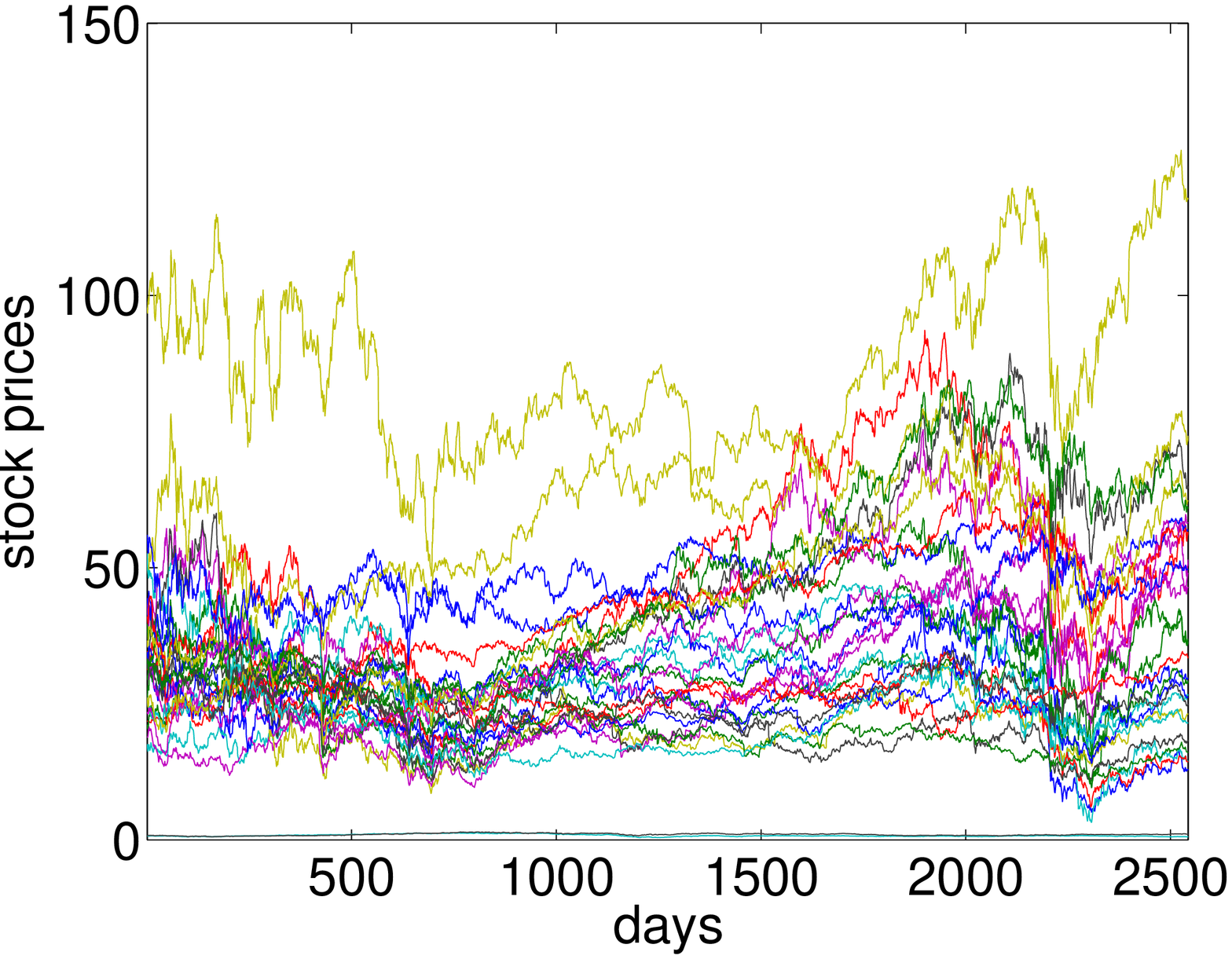}}\\
  \includegraphics[width=0.23\textwidth]{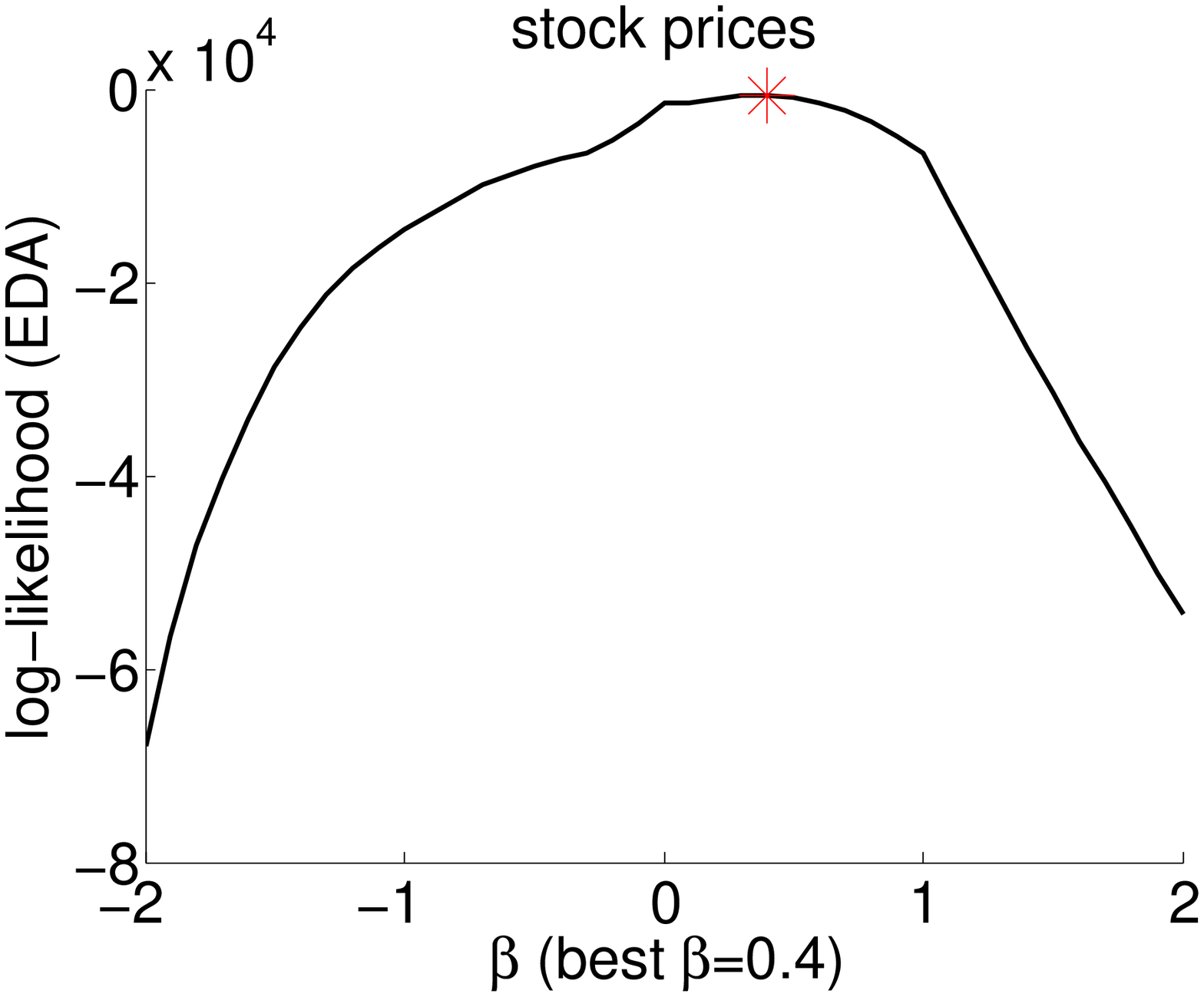}
  \includegraphics[width=0.23\textwidth]{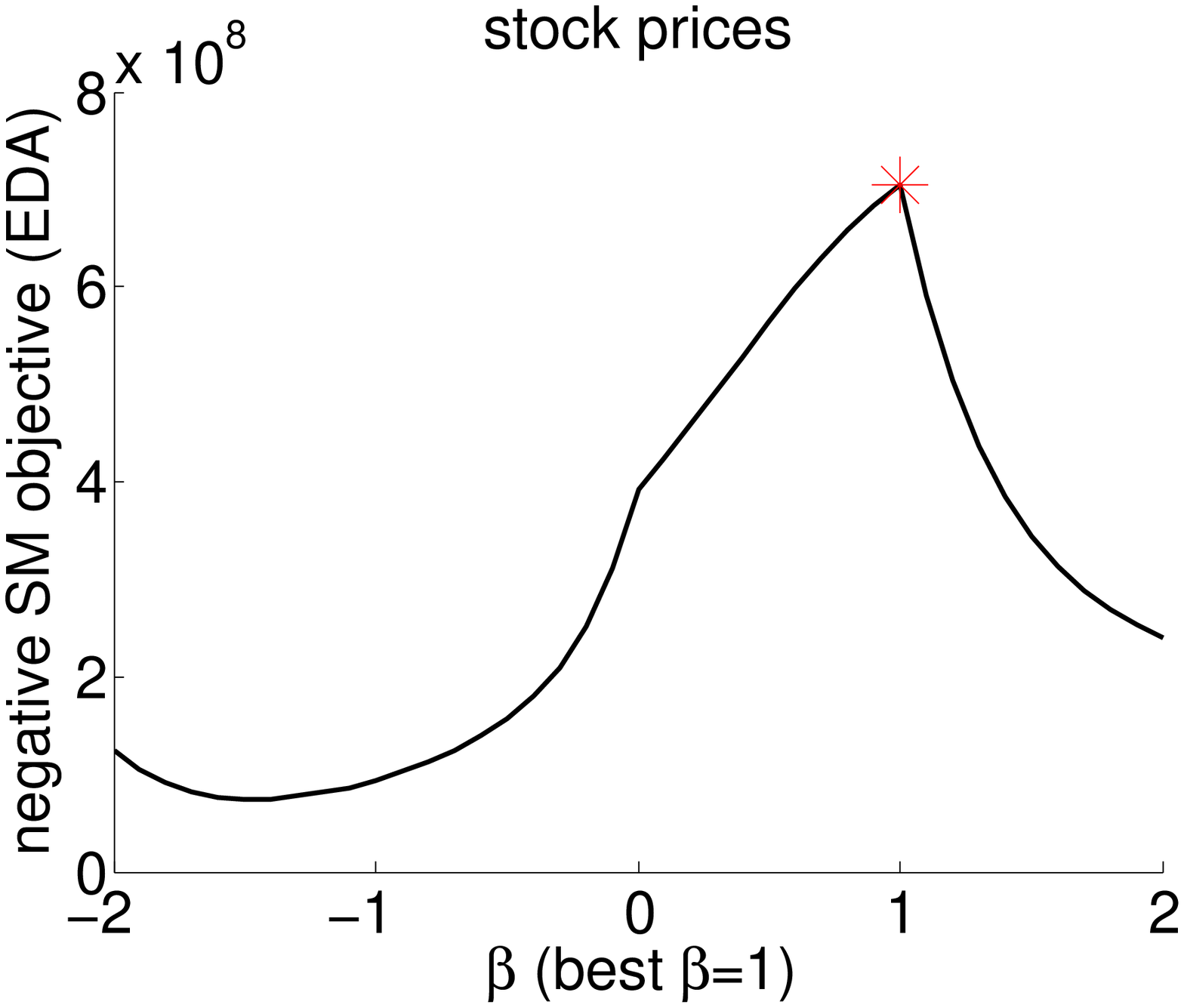}
\end{center}
\caption{Top: the \texttt{stock} data. Bottom left: the EDA log-likelihood
for $\beta\in[-2,2]$. Bottom right: negative SM objective function
for $\beta\in[-2,2]$.}
\label{fig:stock}
\end{figure}

%%%%%%%%%%%%%%%%%%%%%%%%%%%%%%%%%%%%%%%%%%%%%%%%

\subsection{Selecting $\gamma$-divergence}
\label{sec:expgamma}
In this section we demonstrate that the proposed method can be applied
to applications beyond NMF and to non-separable divergence families. To
our knowledge, no other existing methods can handle these two cases.

\subsubsection{Multinomial data}
\label{sec:expgammamn}
We first exemplify $\gamma$-divergence selection for synthetic
data drawn from a multinomial distribution. We generated a
1000-dimensional stochastic vector $\vecp$ from the uniform distribution.
Next we drew $\vecx\sim\text{Multinomial}(n,\vecp)$ with $n=10^7$. The
MEDAL method is applied to find the best $\gamma$-divergence for
the approximation of $\vecx$ by $\vecp$.

Fig.~\ref{fig:gamma} (1st row, left) shows the MEDAL log-likelihood.
The peak appears when $\gamma=0$, which indicates that the normalized
KL-divergence is the most suitable one among the $\gamma$-divergence
family. Selection using score matching of EDA gives the best $\gamma$
also close to zero (Fig.~\ref{fig:gamma} 1st row, right). The result
is expected, because the maximum likelihood estimator of $\vecp$ in
multinomial distribution is equivalent to minimizing the KL-divergence
over $\vecp$. Our finding also justifies the usage of KL-divergence in
topic models with the multinomial distribution
\cite{plsi,blei2001lda}.

\subsubsection{Projective NMF}
\label{sec:expgammapnmf}
Next we apply the MEDAL method to Projective Nonnegative Matrix
Factorization (PNMF) \cite{yuan05pnmf,TNN2010} based on
$\gamma$-divergence \cite{gammadiv,ICANN2011ROZ}. Given a nonnegative
matrix $\matV\in\bbR_+^{F\times N}$, PNMF seeks a low-rank nonnegative
matrix $\matW\in\bbR_+^{F\times K}$ ($K<F$) that minimizes
$D_\gamma\left(\matV||\widehat{\matV}\right)$, where
$\widehat{\matV}=\matW\matW^T\matV$. PNMF is able to produce a highly
orthogonal $\matW$ and thus finds its applications in part-based
feature extraction and clustering analysis, etc. Different from
conventional NMF (or linear NMF) where each factorizing matrix only
appears once in the approximation, the matrix $\matW$ occurs twice in
$\widehat{\matV}$.  Thus it is a special case of Quadratic Nonnegative
Matrix Factorization (QNMF) \cite{QNMF}.

We choose PNMF for two reasons: 1) we demonstrate the MEDAL
performance on QNMF besides the linear NMF already shown in Section
\ref{sec:expnmf}; 2) PNMF contains only one variable matrix in
learning, without the issue of how to interleave the updates of
different variable matrices.

We first tested MEDAL on a synthetic dataset. We generated a diagonal
blockwise data matrix $\matV$ of size $50\times30$, where two blocks
are of sizes $30\times20$ and $20\times10$. The block entries are
uniformly drawn from $[0,10]$. We then added uniform noise from $[0,1]$
to the all matrix entries.
For each $\gamma$, we ran the multiplicative algorithm of PNMF by Yang
and Oja \cite{TNN2010,TNN2011ROZ} to obtain $\matW$ and $\widehat{\matV}$. The MEDAL
method was then applied to select the best $\gamma$.
The resulting approximated log-likelihood for $\gamma\in[-2,2]$ is
shown in Fig.~\ref{fig:gamma} (2nd row). We can see MEDAL and score
matching of EDA give similar results, where the best $\gamma$ appear
at $-0.76$ and $-0.8$, respectively. Both resulting $W$'s give perfect
clustering accuracy of data rows.

We also tested MEDAL on the \texttt{swimmer} dataset
\cite{swimmerdata} which is popularly used in the NMF field. Some
example images from this dataset are shown in Fig.~\ref{fig:swimmer}
(left). We vectorized each image in the dataset as a column and
concatenated the columns into a $1024\times256$ data matrix $\matV$.
This matrix is then fed to PNMF and MEDAL as in the case for the
synthetic dataset. Here we empirically set the rank to $K=17$
according to Tan and F\'evotte \cite{vincent2009ardnmf} and Yang et
al.  \cite{LVA2010}.  The matrix $\matW$ was initialized by PNMF based
on Euclidean distance to avoid poor local minima. The resulting
approximated log-likelihood for $\gamma\in[-1,3]$ is shown in Figure
\ref{fig:gamma} (3rd row, left).  We can see a peak appearing around
$1.7$. Zooming in the region near the peak shows the best
$\gamma=1.69$. The score matching objective over $\gamma$ values
(Fig.~\ref{fig:gamma} 3rd row, right) shows a similar peak and the
best $\gamma$ very close to the one given by MEDAL. Both methods
result in excellent and nearly identical basis matrix ($\matW$) of the
data, where the swimmer body as well as four limbs at four angles are
clearly identified (see Fig.~\ref{fig:swimmer} bottom row).

\begin{figure}[t]
\begin{center}
\includegraphics[width=0.26\textwidth]{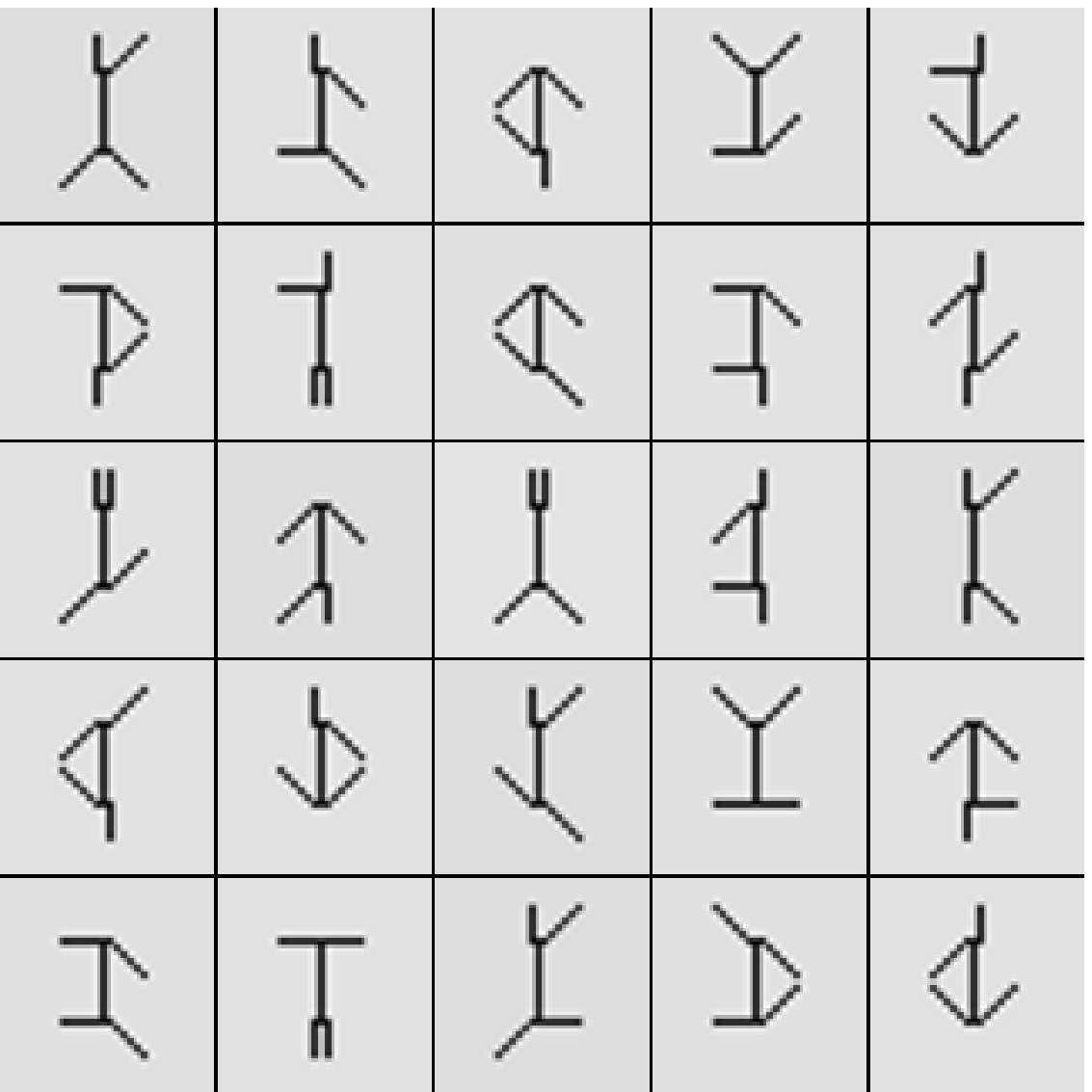}\\
\vspace{2mm}
\includegraphics[width=0.21\textwidth]{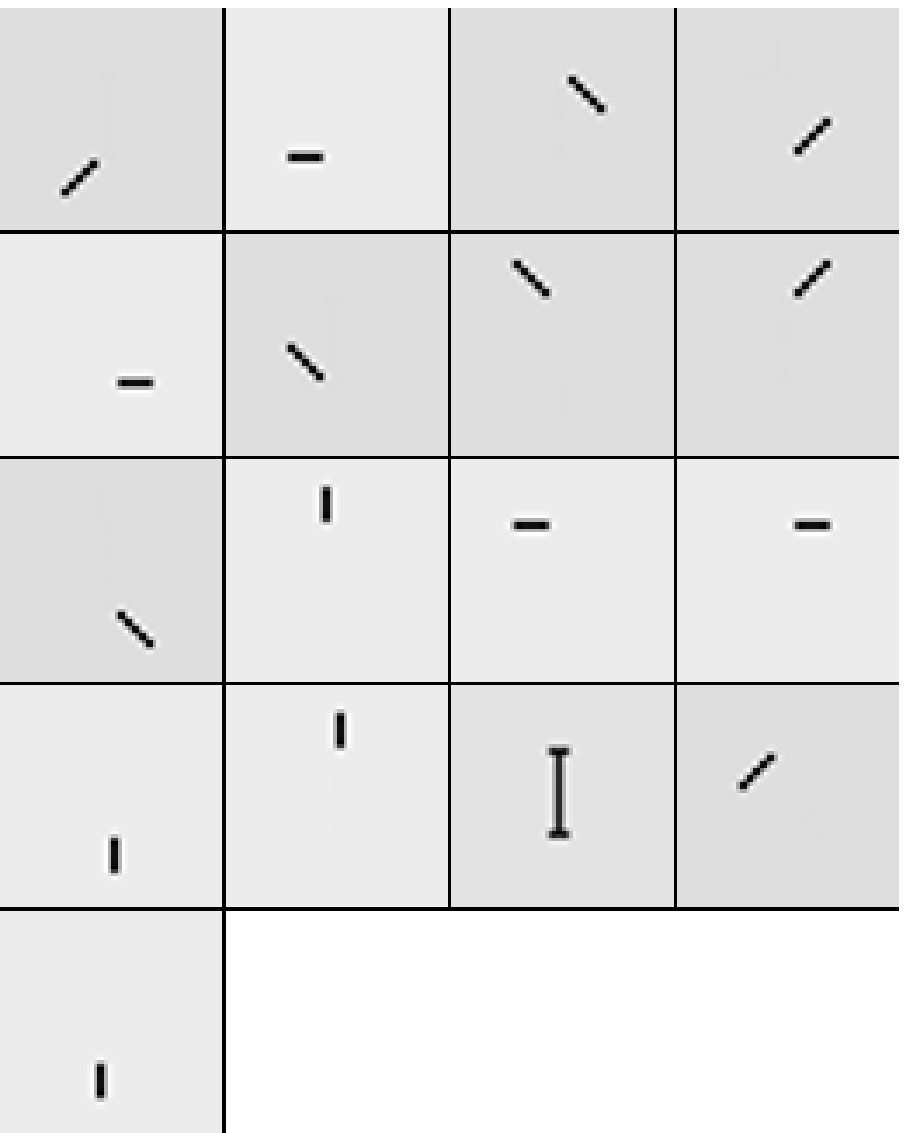}~
\includegraphics[width=0.21\textwidth]{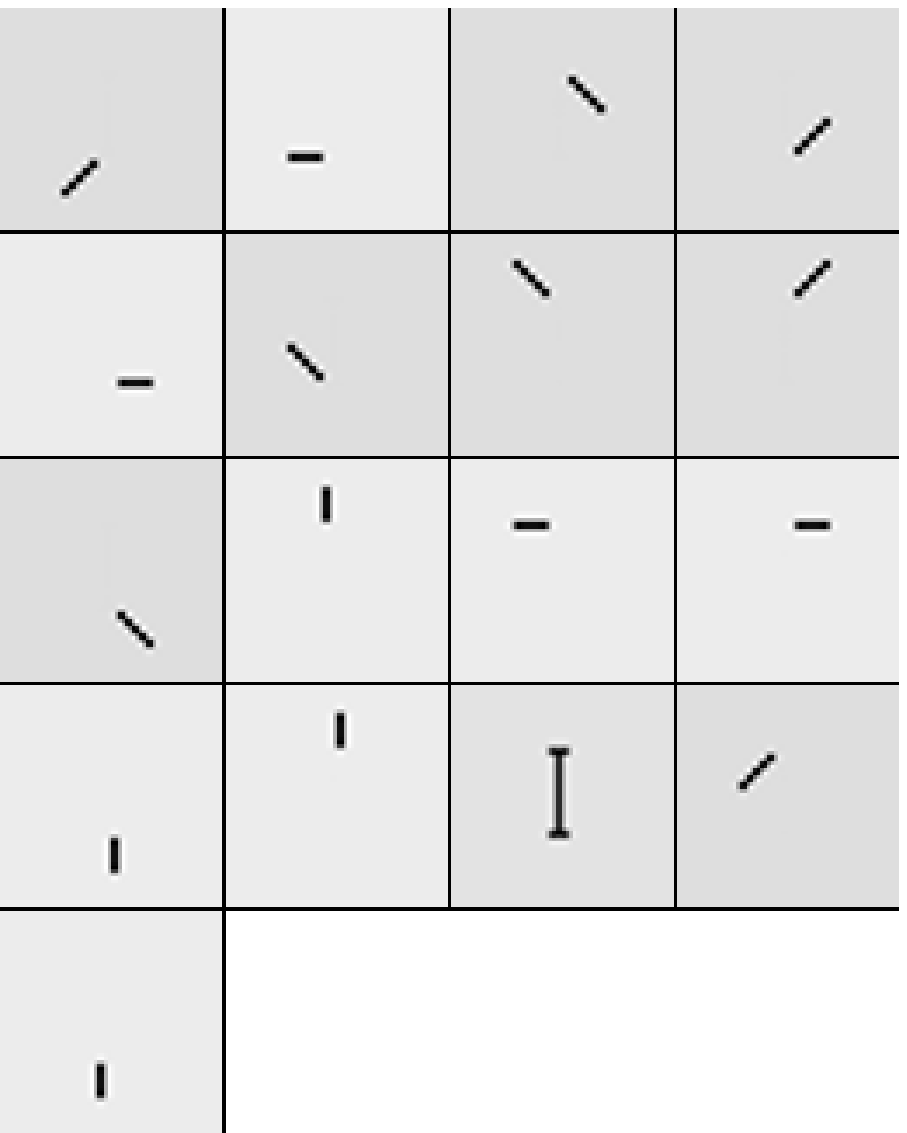}
\end{center}
\caption{Swimmer dataset: (top) example images; (bottom) the best PNMF
basis ($\matW$) selected by using (bottom left) MEDAL and (bottom
right) score matching of EDA. The visualization reshapes each column
of $\matW$ to an image and displays it by the Matlab function
\emph{imagesc}.}
\label{fig:swimmer}
\end{figure}

\subsubsection{Symmetric Stochastic Neighbor Embedding}
\label{sec:expgammassne}
Finally, we show an application beyond NMF, where MEDAL is used to
find the best $\gamma$-divergence for the visualization using
Symmetric Stochastic Neighbor Embedding (s-SNE)
\cite{hinton2002sne,maaten2008tsne}.

Suppose there are $n$ multivariate data samples
$\left\{\vecx_i\right\}_{i=1}^n$ with $\vecx_i\in\bbR^D$ and their
pairwise similarities are represented by an $n\times n$ symmetric
nonnegative matrix $\matP$ where $P_{ii}=0$ and $\sum_{ij}P_{ij}=1$.
The s-SNE visualization seeks a low-dimensional embedding
$\matY=\left[\vecy_1,\vecy_2,\dots,\vecy_n\right]^T\in\bbR^{n\times
d}$ such that pairwise similarities in the embedding approximate those
in the original space.  Generally $d=2$ or $d=3$ for easy visualization. Denote
$q_{ij}=q(\|\vecy_i-\vecy_j\|^2)$ with a certain kernel function $q$,
for example $q_{ij}=\left(1+\|\vecy_i-\vecy_j\|^2\right)^{-1}$. The
pairwise similarities in the embedding are then given by
$Q_{ij}=q_{ij}/\sum_{kl:k\neq l}q_{kl}$. The s-SNE target is that
$\matQ$ is as close to $\matP$ as possible. To measure the
dissimilarity between $\matP$ and $\matQ$, the conventional s-SNE uses
the Kullback-Leibler divergence $D_\text{KL}(\matP||\matQ)$. 
Here we generalize s-SNE to
the whole family of $\gamma$-divergences as dissimilarity measures and select the
best divergence by our MEDAL method.

We have used a real-world \texttt{dolphins} dataset\footnote{available
at \url{http://www-personal.umich.edu/~mejn/netdata/}}. It is the
adjacency matrix of the undirected social network between 62 dolphins.
We smoothed the matrix by PageRank random walk in order to find its
macro structures. The smoothed matrix was then fed to s-SNE based on
$\gamma$-divergence, with $\gamma\in[-2,2]$. The EDA log-likelihood is
shown in Fig.~\ref{fig:gamma} (4th row, left). By the MEDAL
principle the best divergence is $\gamma=-0.6$ for s-SNE and the
\texttt{dolphins} dataset. Score matching of EDA also indicates the
best $\gamma$ is smaller than 0. The resulting visualizations created
by s-SNE with the respective best $gamma$-divergence are shown in
Fig.~\ref{fig:dolphins}, where the node layouts by both methods are
very similar. In both visualizations we can clearly see two
dolphin communities.

\begin{figure}[t]
\begin{center}
\begin{tabular}{cc}
\includegraphics[width=0.23\textwidth]{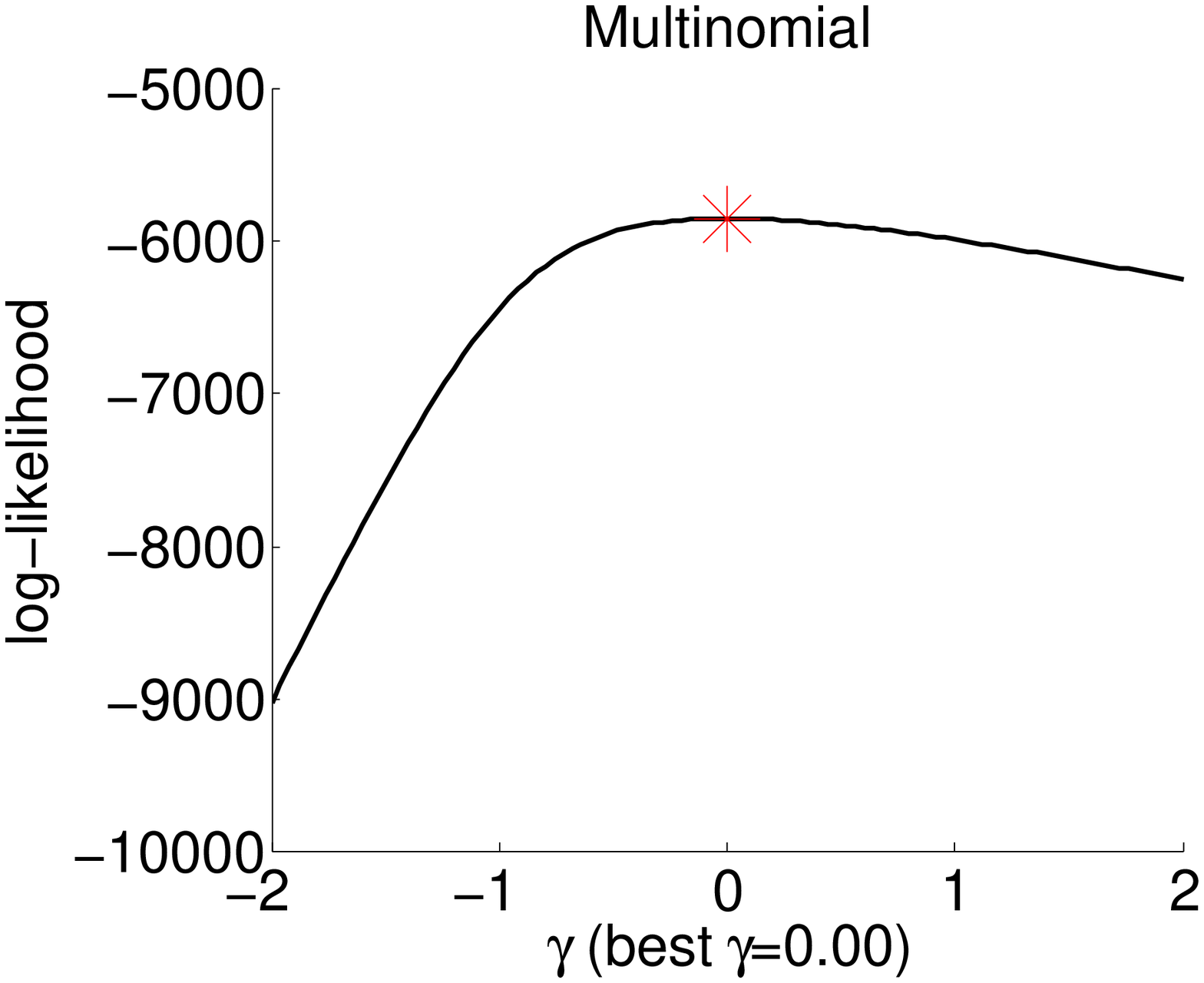}&
\includegraphics[width=0.23\textwidth]{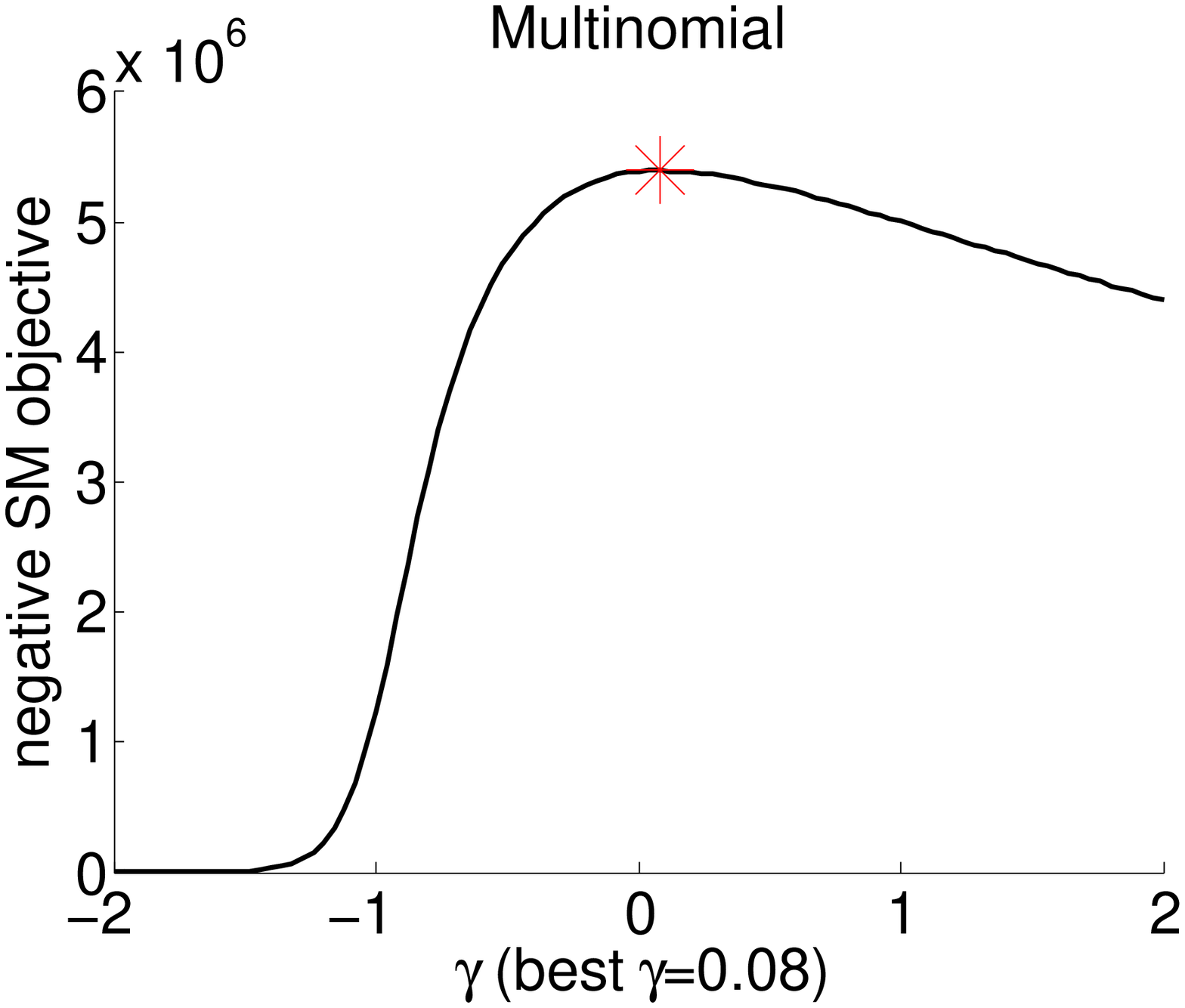}\\
\includegraphics[width=0.23\textwidth]{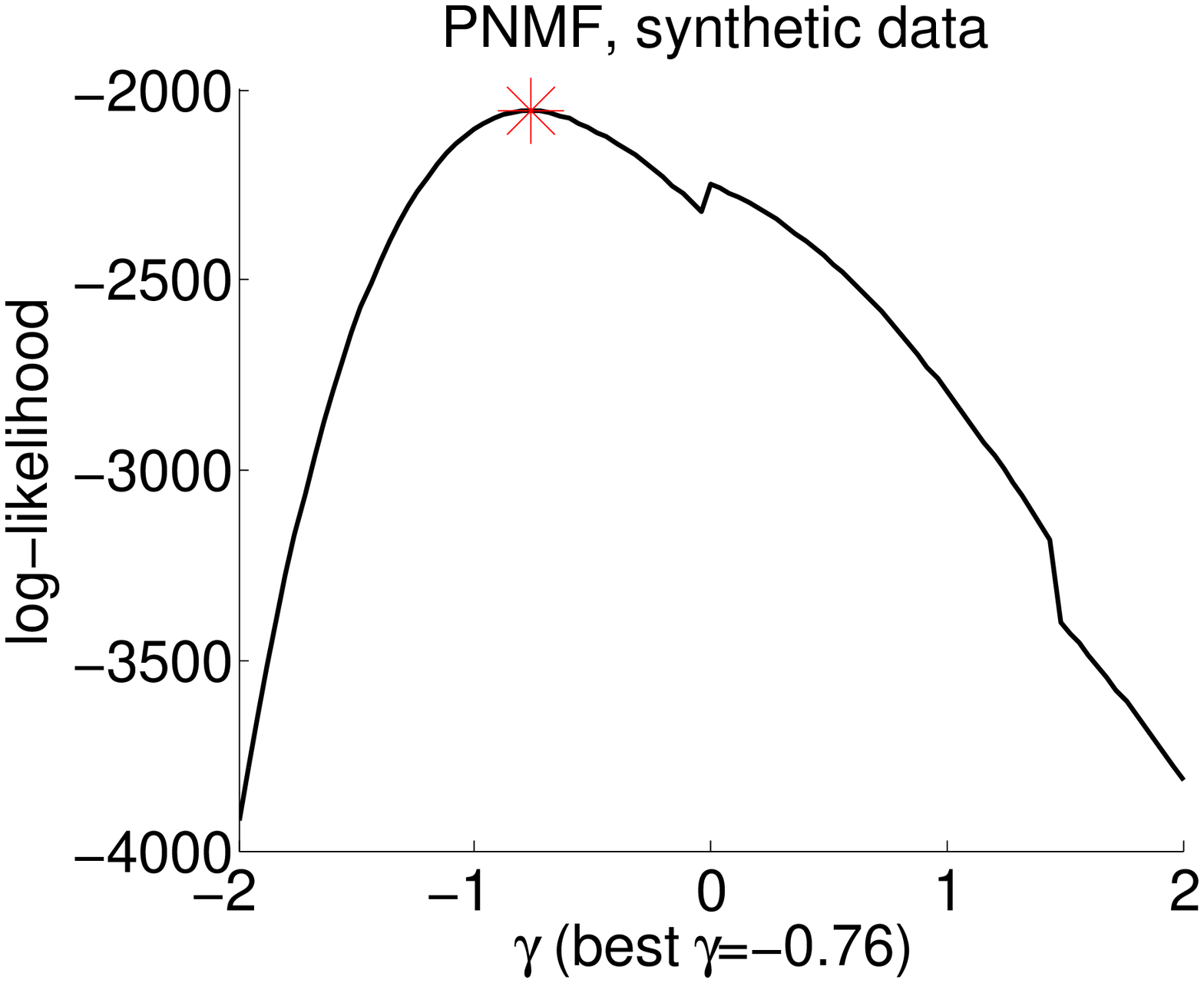}&
\includegraphics[width=0.23\textwidth]{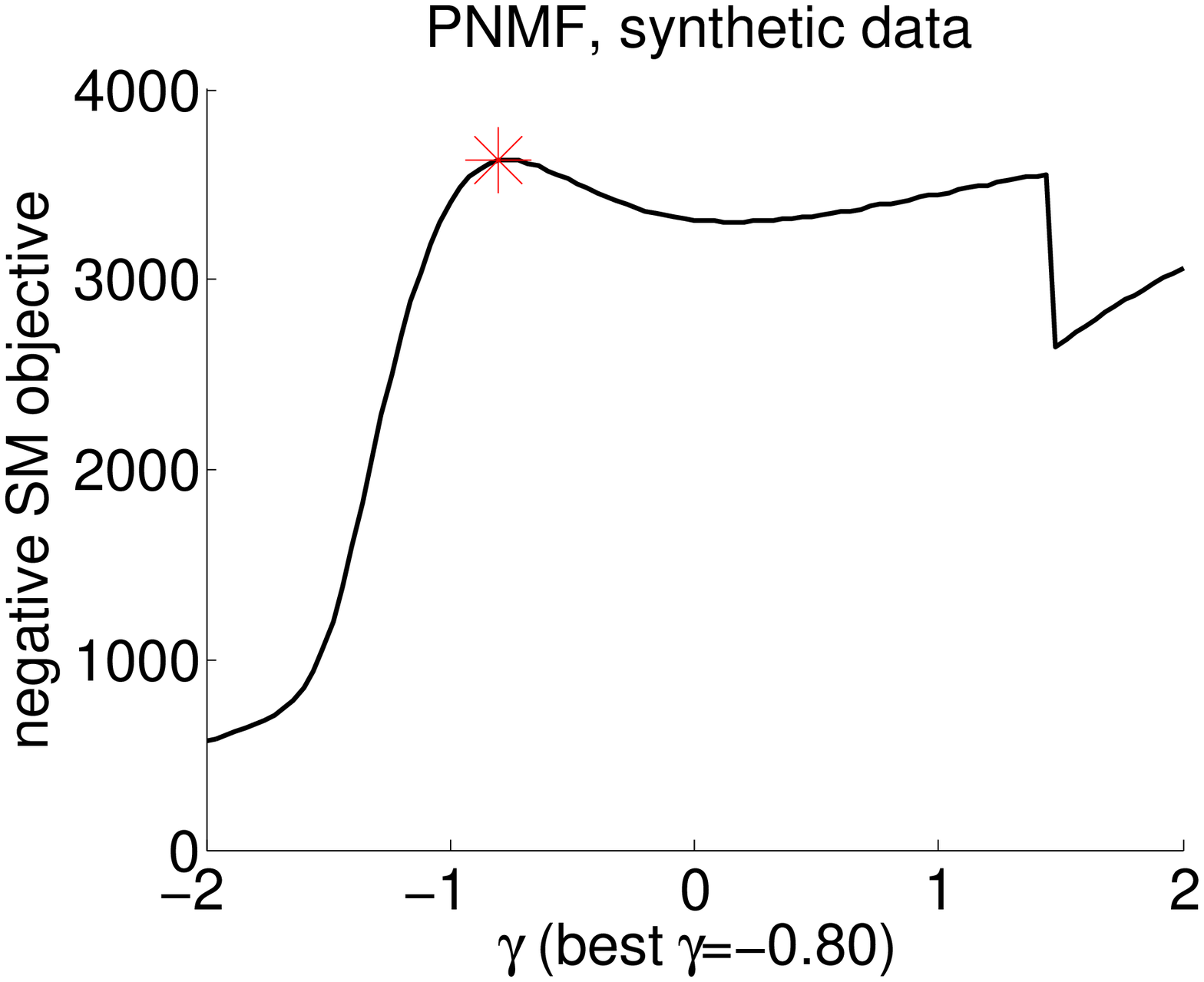}\\
\includegraphics[width=0.23\textwidth]{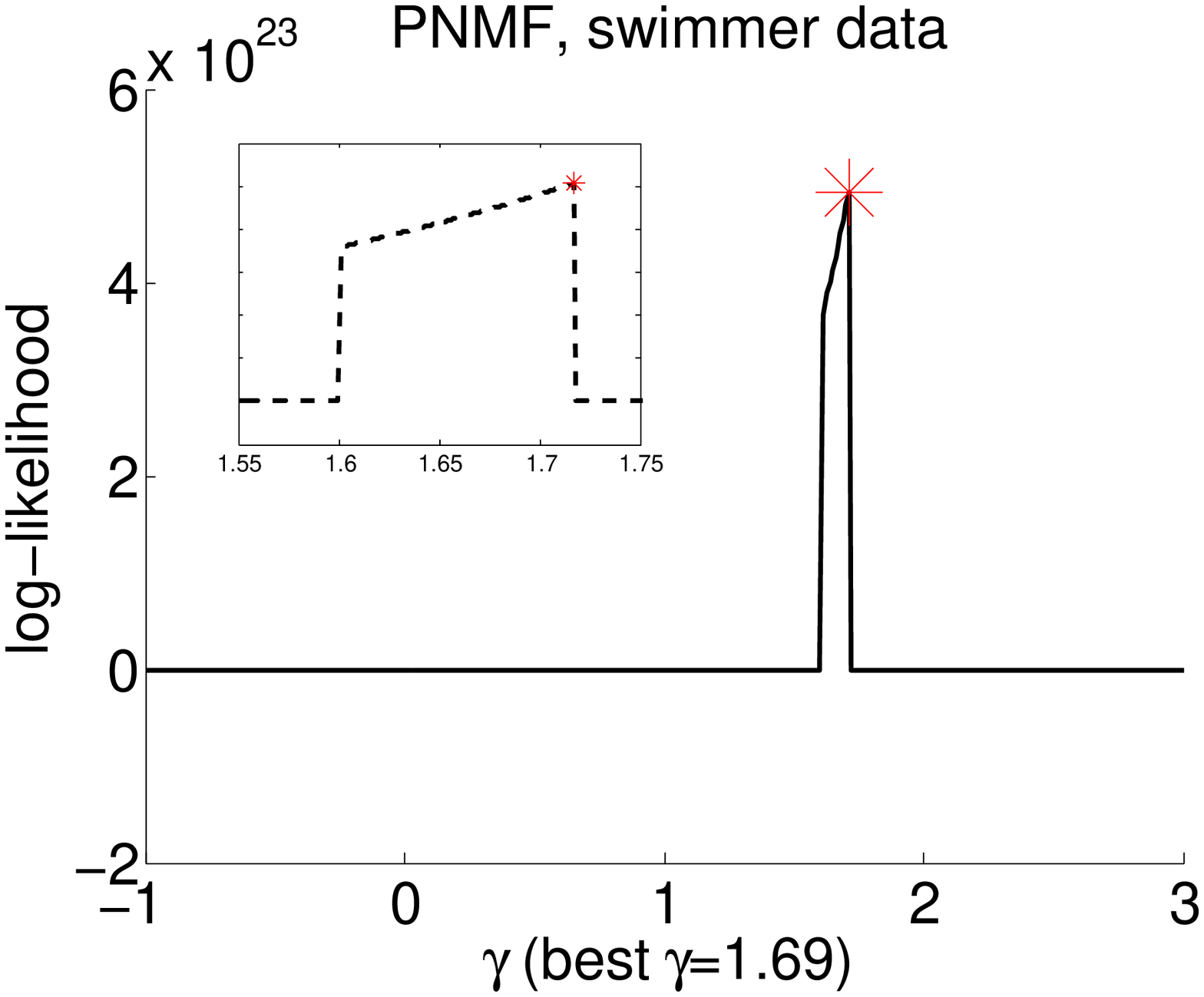}&
\includegraphics[width=0.23\textwidth]{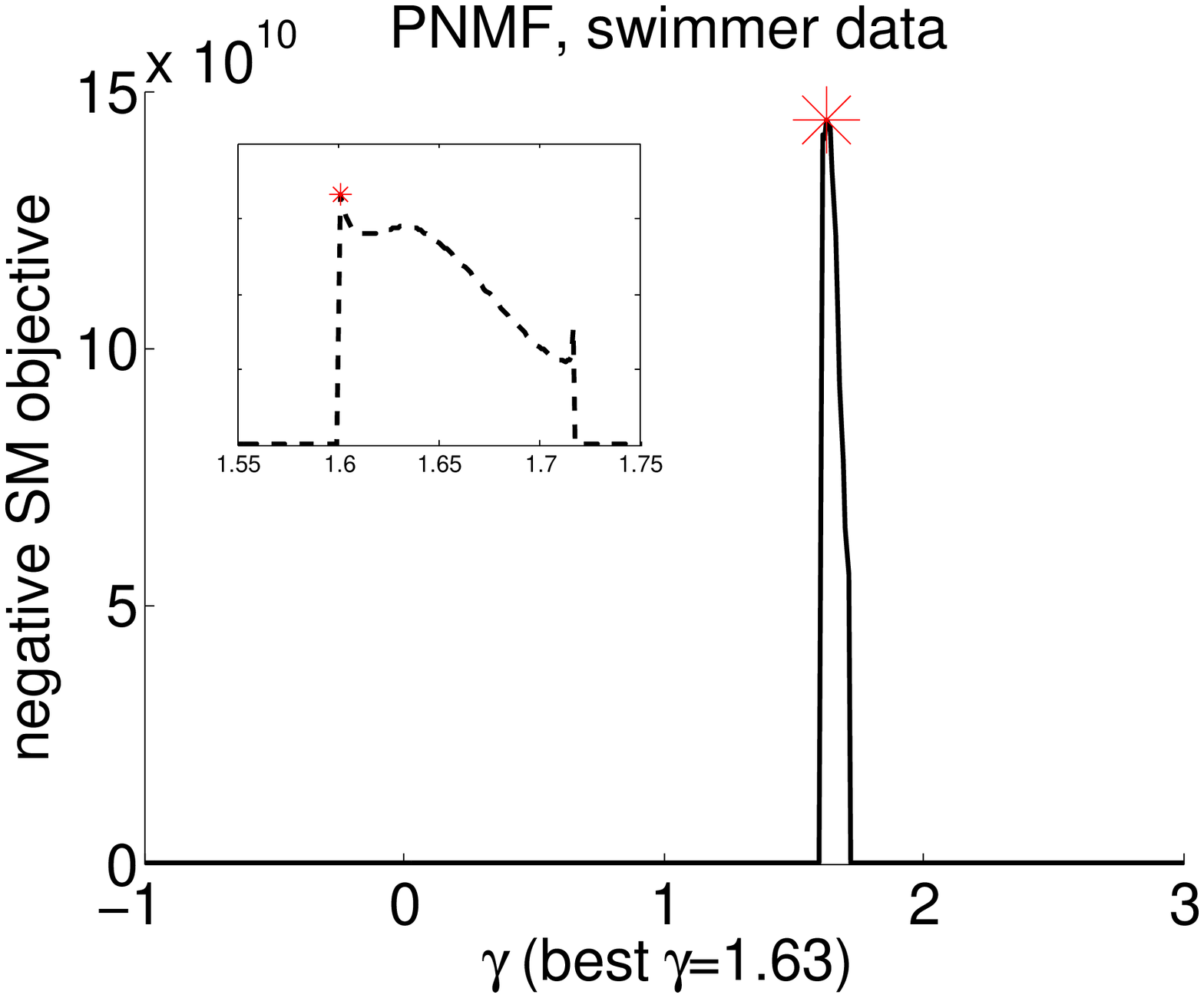}\\
\includegraphics[width=0.23\textwidth]{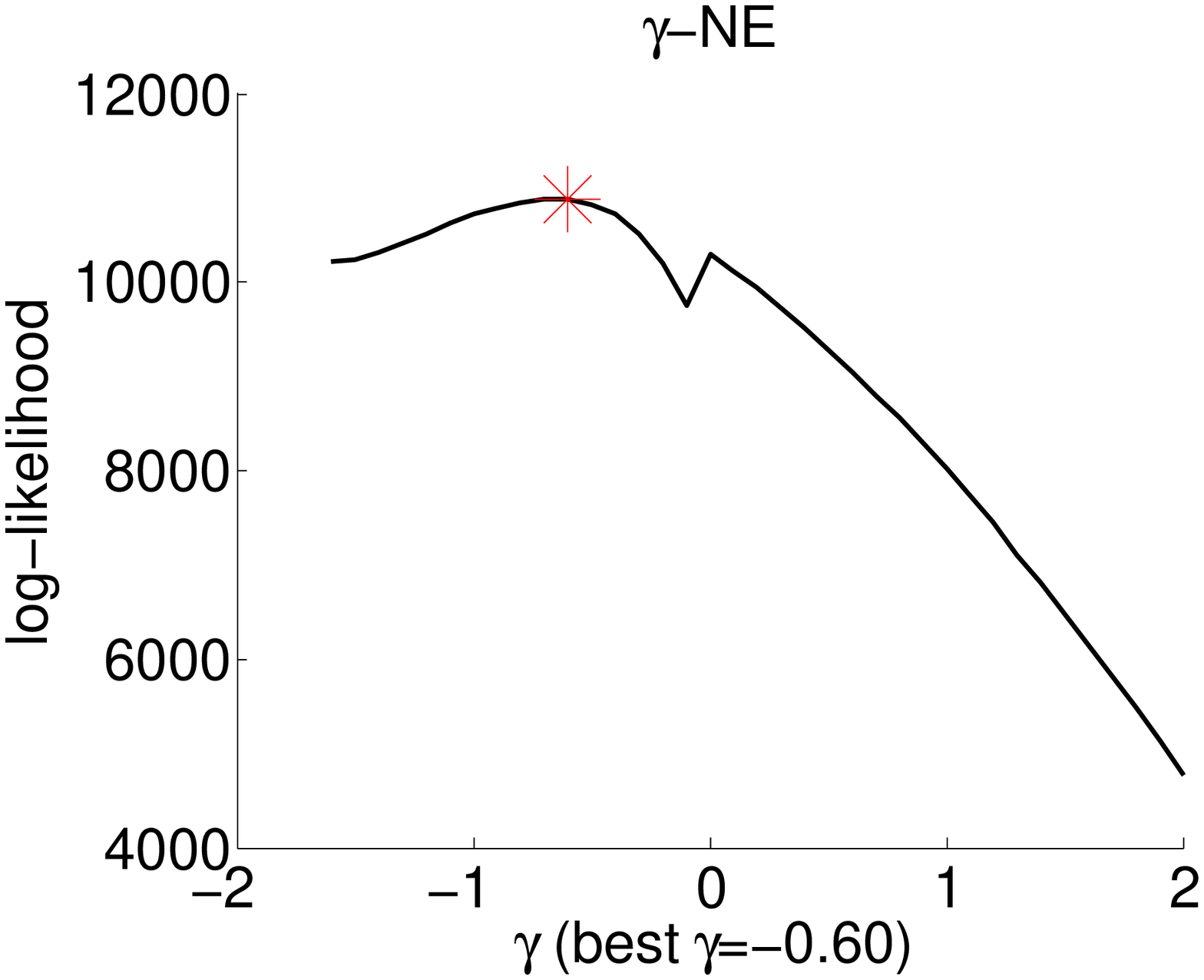}&
\includegraphics[width=0.23\textwidth]{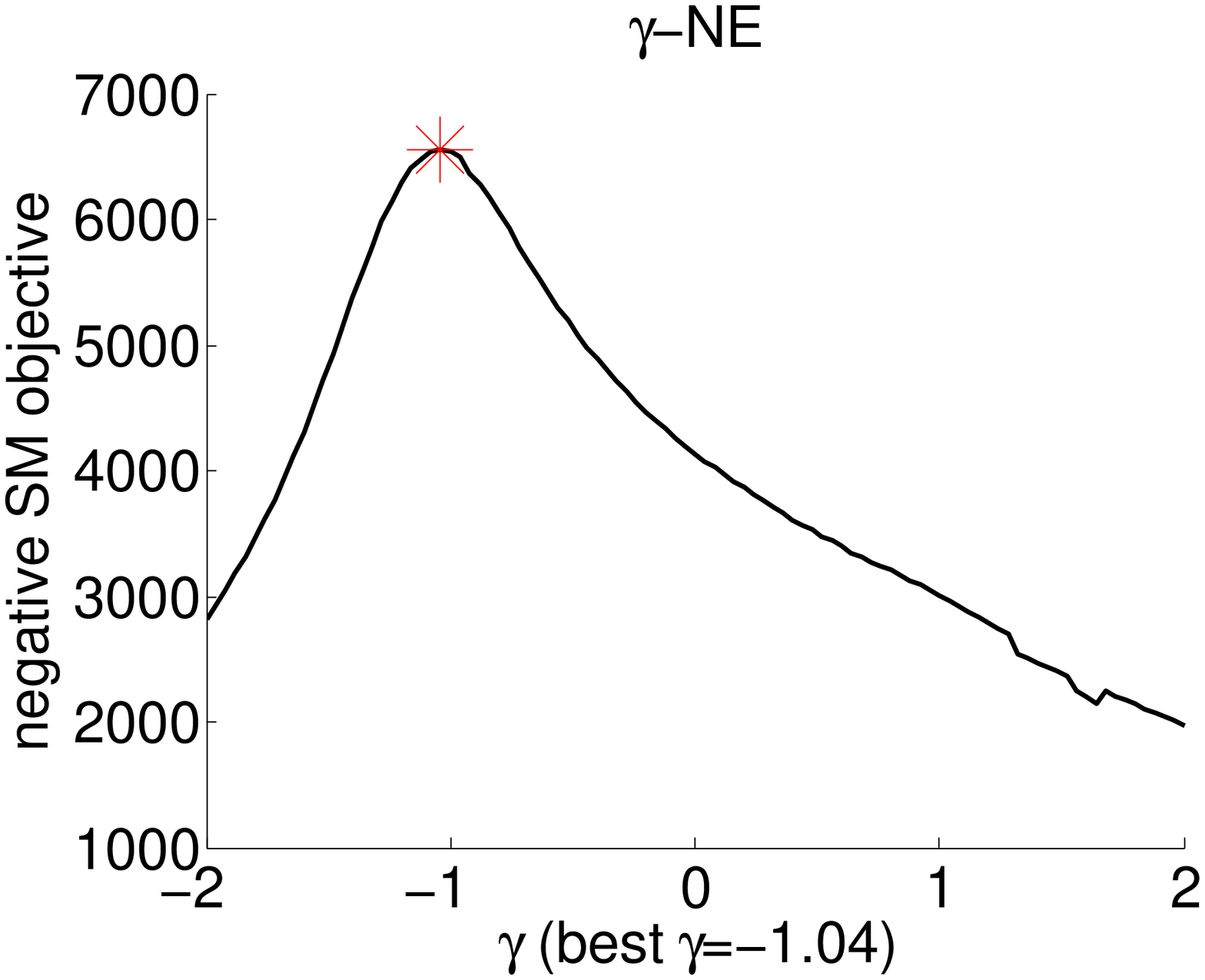}
\end{tabular}
\end{center}
\caption{Selecting the best $\gamma$-divergence: (1st row) for
multinomial data, (2nd row) in PNMF for synthetic data, (3rd row) in
PNMF for the \texttt{swimmer} dataset, and (4th row) in s-SNE for the
\texttt{dolphins} dataset; (left column) using MEDAL and (right
column) using score matching of EDA. The red star highlights the peak
and the small subfigures in each plot shows the zoom-in around the
peak. The sub-figures in the 3rd row zoom in the area near the peaks.}
\label{fig:gamma}
\end{figure}

\begin{figure}[t]
\begin{center}
\includegraphics[width=0.4\textwidth]{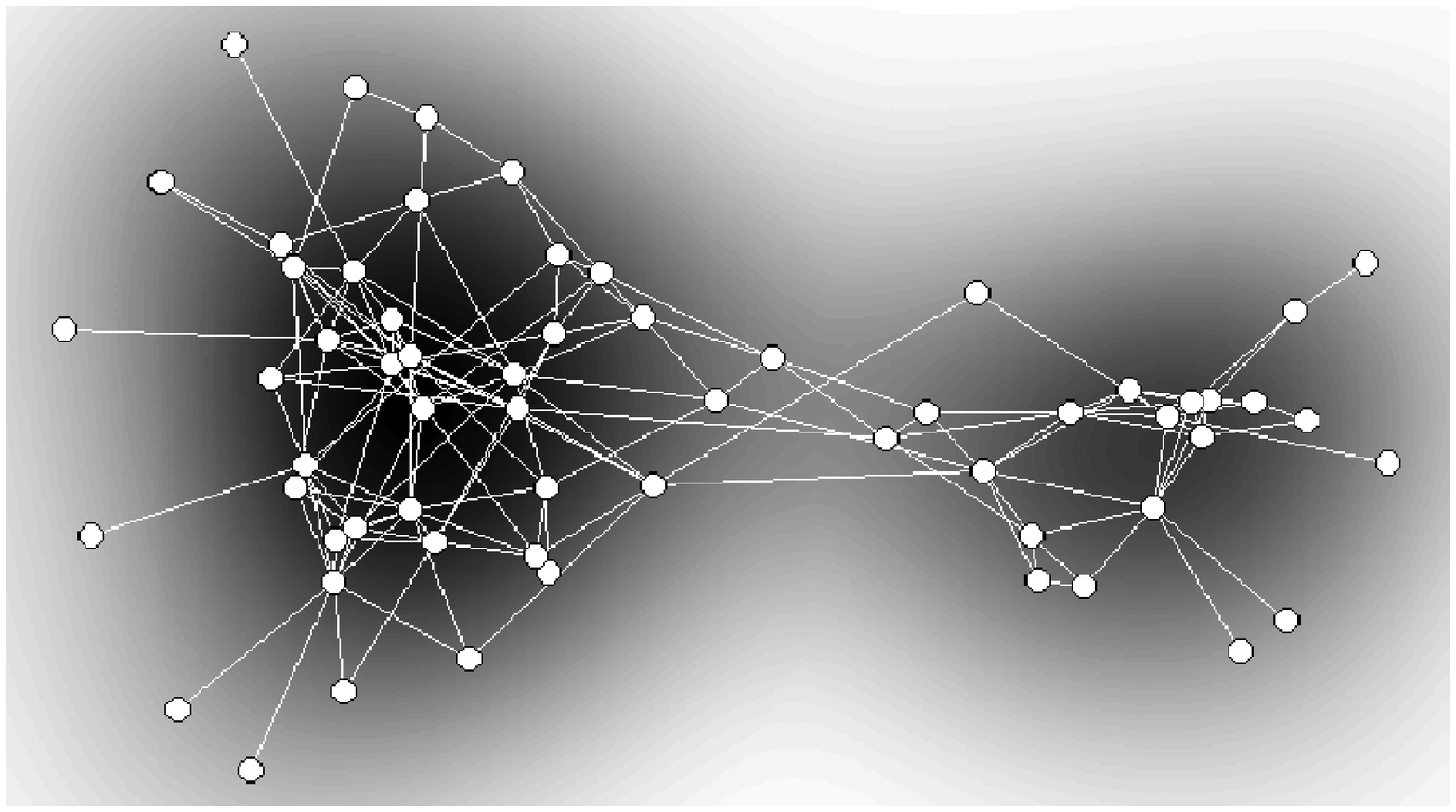}\\
\vspace{2mm}
\includegraphics[width=0.4\textwidth]{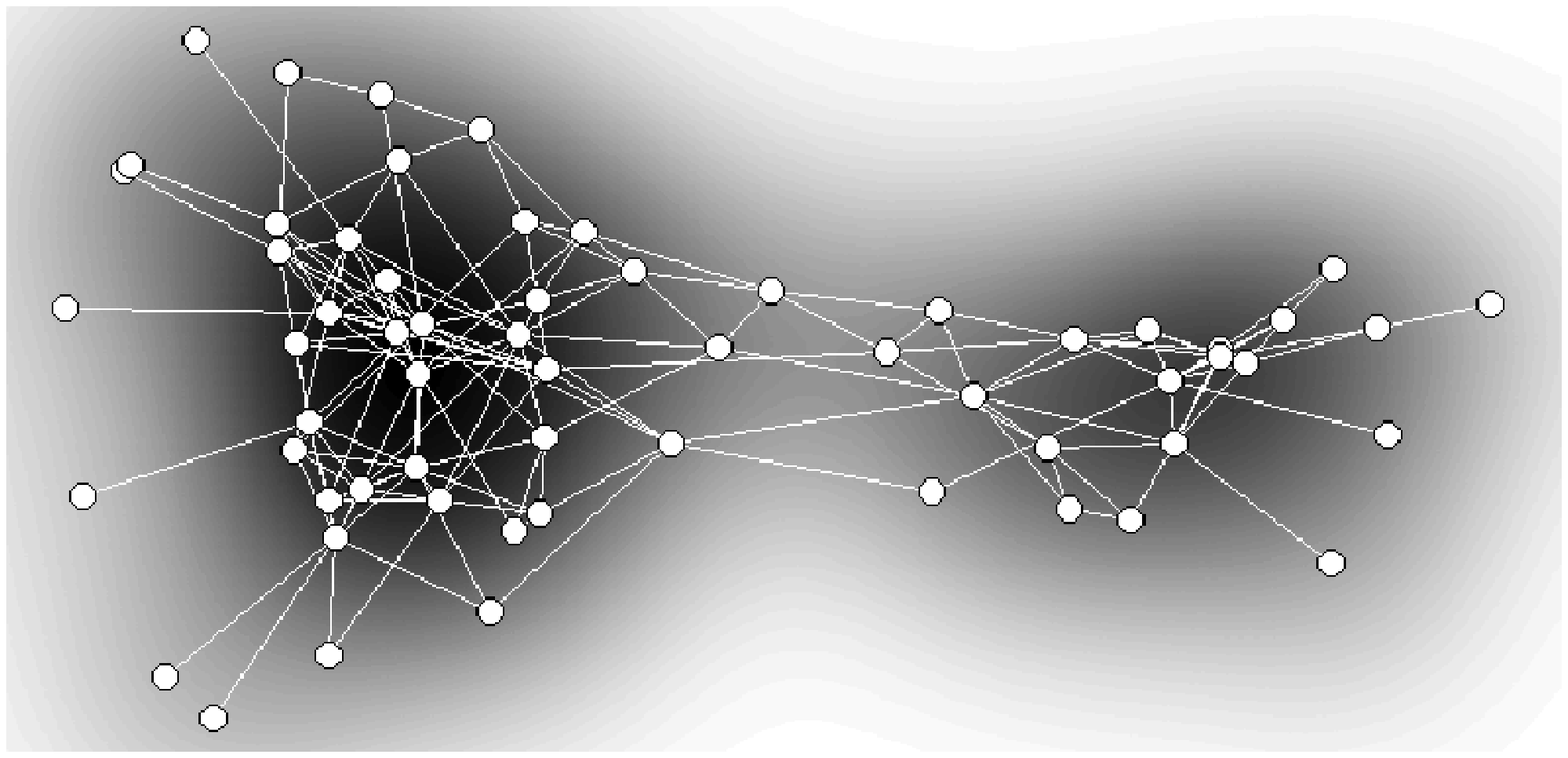}
\end{center}
\caption{Visualization of the \texttt{dolphins} social network with
the best $\gamma$ using (top) MEDAL and (bottom) score matching of EDA.
Dolphins and their social connections are shown by circles and lines,
respectively. The background illustrates the node density by the
Parzen method \cite{parzen}.}
\label{fig:dolphins}
\end{figure}

%%%%%%%%%%%%%%%%%%%%%%%%%%%%%%%%%%%%%%%%%%%%%%%%
\section{Conclusions}
\label{sec:conclusions}
We have presented a new method called MEDAL to automatically select the best
information divergence in a parametric family. Our selection method is
built upon a statistical learning approach, where the divergence is
learned as the result of standard density parameter estimation. Maximizing
the likelihood of the Tweedie distribution is a straightforward way for
selecting $\beta$-divergence, which however has some shortcomings.
We have proposed a novel distribution, the Exponential Divergence with
Augmentation (EDA),  which overcomes these shortcomings and
thus can give a more robust selection for the parameter over a wider range. The new method
has been extended to $\alpha$-divergence selection by a nonlinear
transformation. Furthermore, we have provided new results that connect
the $\gamma$- and $\beta$-divergences, which enable us to extend the selection
method to non-separable cases. The extension also holds for R\'enyi
divergence with similar relationship to $\alpha$-divergence. As a
result, our method can be applied to most commonly used information
divergences in learning.

We have performed extensive experiments to show the accuracy and
applicability of the new method. Comparison on synthetic data has
illustrated that our method is superior to Maximum Tweedie Likelihood,
i.e., it finds the ground truth as accurately as MTL, while being 
defined on all values of $\beta$ and being less prone to numerical
problems (no abrupt changes in the likelihood). 
We also showed that a previous estimation approach by Score Matching on Exponential Divergence distribution (ED, i.e., EDA before augmentation) is not accurate, especially for $\beta<0$. In the application to NMF,
we have provided experimental results on various kinds of data including audio and stock prices. In the non-separable cases, we have
demonstrated selecting $\gamma$-divergence for synthetic data,
Projective NMF, and visualization by s-SNE. In those cases where the
correct parameter value is known in advance for the synthetic data, or there is a wide
consensus in the application community on the correct parameter value for
real-world data, the MEDAL method gives expected results.  
These results show that the presented method has not only broad applications
but also accurate selection performance. In the case of new kinds of
data, for which the appropriate information divergence is not known,
the MEDAL method provides a disciplined and rigorous way to compute
the optimal parameter values. 

In this paper we have focused on information divergence for vectorial
data. There exist other divergences for higher-order tensors, for
example, LogDet divergence and von Newmann divergence (see e.g.\@
\cite{kulis2009lowrank}) that are defined over eigenvalues of
matrices.  Selection among these divergences remains an open problem.

Here we mainly consider a positive data matrix and selecting the divergence
parameter in $(\infty,+\infty)$.
Tweedie distribution has no support for zero entries when $\beta<0$
and thus gives zero likelihood of the whole matrix/tensor by
independence.
In future work, extension of EDA to accommodate nonnegative data
matrices could be developed for $\beta\geq0$.

MEDAL is a two-phase method: the $\beta$ selection is based on the
optimization result of $\vecmu$. Ideally, both variables should be
selected by optimizing the same objective. For maximum log-likelihood
estimator, this requires that the negative log-likelihood equals the
$\beta$-divergence, which is however infeasible for all $\beta$ due to
intractability of integrals. Non-ML estimators could be used to attack
this open problem.

The EDA distribution family includes the exact Gaussian, Gamma, and
Inverse Gaussian distributions, and approximated Poisson distribution.
In the approximation we used the first-order Stirling expansion. One
could apply higher-order expansions to improve the approximation
accuracy. This could be implemented by further augmentation with
higher-order terms around $\beta\rightarrow0$.

%%%%%%%%%%%%%%%%%%%%%%%%%%%%%%%%%%%%%%%%%%%%%%%%
\section{Acknowledgment}
\label{sec:acknowledgement}
This work was financially supported by the Academy of Finland (Finnish
Center of Excellence in Computational Inference Research COIN, grant
no 251170; Zhirong Yang additionally by decision number 140398).

%%%%%%%%%%%%%%%%%%%%%%%%%%%%%%%%%%%%%%%%%%%%%%%%
\appendices
\section{Infinite series expansion in Tweedie distribution}
\label{sec:tweedieexpansion}
In the series expansion, an EDM random variable is represented as a
sum of $G$ independent Gamma random variables $x=\sum_g^G y_g$, where
$G$ is Poisson distributed with parameter $\lambda=
\frac{\mu^{2-p}}{\phi (2-p)}$; and the shape and scale parameters of
the Gamma distribution are $-a$ and $b$, with $a=\frac{2-p}{1-p}$
and $b=\phi (p-1)\mu^{p-1}$.

The pdf of the Tweedie distribution is obtained analytically at $x=0$
as $e^{-\frac{\mu^{2-p}}{\phi(2-p)}}$.
For $x>0$ the function $f(x,\phi,p)=\frac{1}{x}\sum_{j=1}^\infty
W_j(x,\phi,p)$, where for $1<p<2$
\begin{align}
W_j = \frac{x^{-ja} (p-1)^{j a}} {\phi^{j(1-a)}(2-p)^j
j!  \Gamma(-j a)}
\end{align}
and for $p>2$
\begin{align}
W_j=\frac{1}{\pi}\frac{\Gamma(1+j
a) \phi^{j(a-1)} (p-1)^{j a}} {\Gamma(1+j) (p-1)^j x^{j a}} (-1)^j
\sin(-\pi j a).
\end{align}

This infinite summation needs approximation in practice. Dunn and
Smyth~\cite{dunn05series} described an approach to select a subset of
these infinite terms to accurately approximate $f(x,\phi,p)$. In their approach, Stirling's approximation of the
Gamma functions are used to find the index $j$ which gives the highest
value of the function.  Then, in order to find the most significant
region, the indices are progressed in both directions until negligible
terms are reached.

\section{Gauss-Laguerre quadratures}
\label{sec:gausslaguerre}
This method (e.g.\@ \cite{gausslaguerre}) can evaluate definite integrals of the form
\begin{align}
\int_0^\infty
e^{-z}f(z)dz \approx \sum_i^n f(z_i) w_i,
\end{align}
where $z_i$ is the $i$th root of the $n$-th order Laguerre polynomial
$L_n(z)$, and the weights are given by
\begin{align}
w_i =
\frac{z_i}{(n+1)^2L_n^2(z_i)}.
\end{align}
The recursive definition of $L_n(z)$ is given by
\begin{align}
L_{n+1}(z) = \frac{1}{n+1}\left[(2n+1-z)L_n(z)-nL_{n-1}(z)\right],
\end{align}
with $L_0(z)=1$ and $L_1(z)=1-z$.
In our experiments, we used the Matlab implementation by
Winckel\footnote{available at \url{http://www.mathworks.se/matlabcentral/fileexchange/}} with $n=5000$.

\section{Proofs of Theorems \ref{theo:connectgamma2beta} and \ref{theo:connectrenyi2alpha}}
\label{sec:connectionproof}
\begin{lemma}
\label{lem:log}
$\arg\min_z af(z)=\arg\min_z a\ln f(z)$ for $a\in\bbR$ and $f(z)>0$.
\end{lemma}
The proof of the lemma is simply by the monotonicity of $\ln$.

Next we prove Theorem \ref{theo:connectgamma2beta}. For
$\beta\in\bbR\backslash\{-1,0\}$, zeroing
$\fracpartial{D_\beta(\vecx||c\vecmu)}{c}$ gives
\begin{align}
c^*=\frac{\sum_ix_i\mu_i^\beta}{\sum_i\mu_i^{1+\beta}}.
\end{align}
Putting it back to $\min_{\vecmu} \min_c D_\beta(\vecx||c\vecmu)$, we
obtain:
\begin{align*}
&\min_{\vecmu} \min_c D_\beta(\vecx||c\vecmu)\\
=&\min_{\vecmu}\frac{1}{\beta(1+\beta)}
\left[\sum_i x_i^{1+\beta}
+\beta\sum_i\left(\frac{\sum_jx_j\mu_j^\beta}{\sum_j\mu_j^{1+\beta}}\mu_i\right)^{1+\beta}\right.\\
&\left.-(1+\beta)\sum_ix_i\left(\frac{\sum_jx_j\mu_j^\beta}{\sum_j\mu_j^{1+\beta}}\mu_i\right)^\beta\right]\\
=&\min_{\vecmu}\frac{1}{\beta(1+\beta)}
\left[\sum_i x_i^{1+\beta}
-\frac{\left(\sum_ix_i\mu_i^\beta\right)^{1+\beta}}{\left(\sum_j\mu_j^{1+\beta}\right)^\beta}
\right]
\end{align*}

Dropping the constant, and by Lemma \ref{lem:log}, the above is
equivalent to minimizing
\begin{align*}
\frac{1}{\beta(1+\beta)}\left[
\beta\ln\left(\sum_j\mu_j^{1+\beta}\right)
-(1+\beta)\ln\left(\sum_ix_i\mu_i^\beta\right) \right]
\end{align*}
Adding a constant $\frac{1}{\beta(1+\beta)}\ln\left(\sum_i
x_i^{1+\beta}\right)$, the objective becomes minimizing
$\gamma$-divergence (replacing $\beta$ with $\gamma$; see Eq.\@
(\ref{eq:divgamma})).

We can apply the similar technique to prove Theorem
\ref{theo:connectrenyi2alpha}. For $\alpha\in\bbR\backslash\{0,1\}$,
zeroing $\fracpartial{D_\alpha(\vecx||c\vecmu)}{c}$ gives
\begin{align}
c^*=\left(\frac{\sum_ix_i^\alpha
\mu_i^{1-\alpha}}{\sum_i\mu_i}\right)^{1/\alpha}
\end{align}
Putting it back, we obtain
\begin{align}
&D_\alpha(\vecx||c^*\vecmu)\\
=&\frac{1}{\alpha(1-\alpha)}\sum_i\left\{\alpha x_i
+(1-\alpha)\left(\frac{\sum_jx_j^\alpha
\mu_j^{1-\alpha}}{\sum_j\mu_j}\right)^{1/\alpha}\mu_i\right.\\
&\left.-x_i^\alpha \left[\left(\frac{\sum_jx_j^\alpha \mu_j^{1-\alpha}}{\sum_j\mu_j}\right)^{1/\alpha}\mu_i\right]^{1-\alpha}\right\}\\
=&\frac{1}{\alpha-1}\left[\sum_ix_i^\alpha
\left(\frac{\mu_i}{\sum_j\mu_j}\right)^{1-\alpha}\right]^{1/\alpha}+\frac{\sum_ix_i}{1-\alpha}.
\end{align}
Dropping the constant $\frac{\sum_ix_i}{1-\alpha}$, and by Lemma
\ref{lem:log}, minimizing the above is equivalent to minimization of
(for $\alpha>0$)
\begin{align}
\frac{1}{\alpha-1}\ln\left[\sum_ix_i^\alpha \left(\frac{\mu_i}{\sum_j\mu_j}\right)^{1-\alpha}\right]
\end{align}
Adding a constant $\frac{\alpha}{1-\alpha}\ln\sum_ix_i$ to the above, the
objective becomes minimizing R\'enyi-divergence (replacing $\alpha$
with $\rho$; see Eq.\@ (\ref{eq:divrenyi})).

The proofs for the special cases are similar, where the main steps are given below
\begin{itemize}
\item $\beta=\gamma\rightarrow0$ (or $\alpha=\rho\rightarrow1$): zeroing
$\fracpartial{D_{\beta\rightarrow0}(\vecx||c\vecmu)}{c}$ gives $c^*
= \frac{\sum_ix_i}{\sum_i\mu_i}$. Putting it back, we obtain
$D_{\beta\rightarrow0}(\vecx||c^*\vecmu)
=\left(\sum_ix_i\right)D_{\gamma\rightarrow0}(\vecx||\vecmu)$.

\item $\beta=\gamma\rightarrow-1$: zeroing
$\fracpartial{D_{\beta\rightarrow-1}(\vecx||c\vecmu)}{c}$ gives $c^* =
\frac{1}{M}\sum_i\frac{x_i}{\mu_i}$, where $M$ is the length of
$\vecx$. Putting it back, we obtain
$D_{\beta\rightarrow-1}(\vecx||c^*\vecmu)=MD_{\gamma\rightarrow-1}(\vecx||\vecmu)$.

\item $\alpha=\rho\rightarrow0$: zeroing
$\fracpartial{D_{\alpha\rightarrow0}(\vecx||c\vecmu)}{c}$ gives
\begin{align*}
c^*=\exp\left(-\frac{\sum_i\mu_i\ln\frac{\mu_i}{x_i}}{\sum_i\mu_i}\right).
\end{align*}
Putting it back, we obtain
\begin{align*}
D_{\alpha\rightarrow0}(\vecx||c^*\vecmu)=
-\exp\left(-\sum_i\mut_i\ln\frac{\mut_i}{x_i}\right) +\sum_ix_i,
\end{align*}
where $\mut_i=\mu_i /\sum_j\mu_j$. Dropping the constant $\sum_ix_i$,
minimizing $D_{\alpha\rightarrow0}(\vecx||c^*\vecmu)$ is equivalent to
minimization of $\sum_i\mut_i\ln\frac{\mut_i}{x_i}$.  Adding the
constant $\ln\sum_jx_j$ to the latter, the objective becomes identical
to $D_{\rho\rightarrow0}(\vecx||\vecmu)$, i.e.\@
$D_\text{KL}(\vecmu||\vecx)$.

\end{itemize}
%%%%%%%%%%%%%%%%%%%%%%%%%%%%%%%%%%%%%%%%%%%%%%%%
\bibliographystyle{IEEEtran}
\bibliography{learning_divergence}

\begin{thebibliography}{10}
\providecommand{\url}[1]{#1}
\csname url@samestyle\endcsname
\providecommand{\newblock}{\relax}
\providecommand{\bibinfo}[2]{#2}
\providecommand{\BIBentrySTDinterwordspacing}{\spaceskip=0pt\relax}
\providecommand{\BIBentryALTinterwordstretchfactor}{4}
\providecommand{\BIBentryALTinterwordspacing}{\spaceskip=\fontdimen2\font plus
\BIBentryALTinterwordstretchfactor\fontdimen3\font minus
  \fontdimen4\font\relax}
\providecommand{\BIBforeignlanguage}[2]{{%
\expandafter\ifx\csname l@#1\endcsname\relax
\typeout{** WARNING: IEEEtran.bst: No hyphenation pattern has been}%
\typeout{** loaded for the language `#1'. Using the pattern for}%
\typeout{** the default language instead.}%
\else
\language=\csname l@#1\endcsname
\fi
#2}}
\providecommand{\BIBdecl}{\relax}
\BIBdecl

\bibitem{kompass2006divergence}
R.~Kompass, ``A generalized divergence measure for nonnegative matrix
  factorization,'' \emph{Neural Computation}, vol.~19, no.~3, pp. 780--791,
  2006.

\bibitem{dhillo2006nips}
I.~S. Dhillon and S.~Sra, ``Generalized nonnegative matrix approximations with
  {B}regman divergences,'' in \emph{Advances in Neural Information Processing
  Systems}, vol.~18, 2006, pp. 283--290.

\bibitem{cichocki2008alphanmf}
A.~Cichocki, H.~Lee, Y.-D. Kim, and S.~Choi, ``Non-negative matrix
  factorization with $\alpha$-divergence,'' \emph{Pattern Recognition Letters},
  vol.~29, pp. 1433--1440, 2008.

\bibitem{TNN2011}
Z.~Yang and E.~Oja, ``Unified development of multiplicative algorithms for
  linear and quadratic nonnegative matrix factorization,'' \emph{IEEE
  Transactions on Neural Networks}, vol.~22, no.~12, pp. 1878--1891, 2011.

\bibitem{hinton2002sne}
G.~Hinton and S.~Roweis, ``Stochastic neighbor embedding,'' in \emph{Advances
  in Neural Information Processing Systems}, 2002, pp. 833--840.

\bibitem{maaten2008tsne}
L.~van~der Maaten and G.~Hinton, ``Visualizing data using {t-SNE},''
  \emph{Journal of Machine Learning Research}, vol.~9, pp. 2579--2605, 2008.

\bibitem{blei2001lda}
D.~Blei, A.~Y. Ng, and M.~I. Jordan, ``Latent dirichlet allocation,''
  \emph{Journal of Machine Learning Research}, vol.~3, pp. 993--1022, 2001.

\bibitem{sato2012rethinking}
I.~Sato and H.~Nakagawa, ``Rethinking collapsed variational bayes inference for
  lda,'' in \emph{International Conference on Machine Learning (ICML)}, 2012.

\bibitem{minka2005divergence}
T.~Minka, ``Divergence measures and message passing,'' Microsoft Research,
  Tech. Rep., 2005.

\bibitem{alphadiv}
H.~Chernoff, ``A measure of asymptotic efficiency for tests of a hypothesis
  based on a sum of observations,'' \emph{The Annals of Mathematical
  Statistics}, vol.~23, pp. 493--507, 1952.

\bibitem{amari1985diff}
S.~Amari, \emph{Differential-Geometrical Methods in Statistics}.\hskip 1em plus
  0.5em minus 0.4em\relax Springer Verlag, 1985.

\bibitem{betadiv}
A.~Basu, I.~R. Harris, N.~Hjort, and M.~Jones, ``Robust and efficient
  estimation by minimising a density power divergence,'' \emph{Biometrika},
  vol.~85, pp. 549--559, 1998.

\bibitem{gammadiv}
H.~Fujisawa and S.~Eguchi, ``Robust paramater estimation with a small bias
  against heavy contamination,'' \emph{Journal of Multivariate Analysis},
  vol.~99, pp. 2053--2081, 2008.

\bibitem{cichocki2010abgdiv}
A.~Cichocki and S.-i. Amari, ``Families of alpha- beta- and gamma- divergences:
  Flexible and robust measures of similarities,'' \emph{Entropy}, vol.~12,
  no.~6, pp. 1532--1568, 2010.

\bibitem{cichocki2011abnmf}
A.~Cichocki, S.~Cruces, and S.-I. Amari, ``Generalized alpha-beta divergences
  and their application to robust nonnegative matrix factorization,''
  \emph{Entropy}, vol.~13, pp. 134--170, 2011.

\bibitem{csiszarfdiv}
I.~Csisz\'ar, ``Eine informationstheoretische ungleichung und ihre anwendung
  auf den beweis der ergodizitat von markoffschen ketten,'' \emph{Publications
  of the Mathematical Institute of Hungarian Academy of Sciences Series A},
  vol.~8, pp. 85--108, 1963.

\bibitem{morimotofdiv}
T.~Morimoto, ``Markov processes and the h-theorem,'' \emph{Journal of the
  Physical Society of Japan}, vol.~18, no.~3, pp. 328--331, 1963.

\bibitem{bregmandiv}
L.~M. Bregman, ``The relaxation method of finding the common points of convex
  sets and its application to the solution of problems in convex programming,''
  \emph{USSR Computational Mathematics and Mathematical Physics}, vol.~7,
  no.~3, pp. 200--217, 1967.

\bibitem{ICANN2011ROZ}
Z.~Yang, H.~Zhang, Z.~Yuan, and E.~Oja, ``Kullback-leibler divergence for
  nonnegative for nonnegative matrix factorization,'' in \emph{Proceedings of
  21st International Conference on Artificial Neural Networks}, 2011, pp.
  14--17.

\bibitem{ICANN2009ROZ}
Z.~Yang and E.~Oja, ``Projective nonnegative matrix factorization with
  $\alpha$-divergence,'' in \emph{Proceedings of 19th International Conference
  on Artificial Neural Networks}, 2009, pp. 20--29.

\bibitem{fevotte09nonnegative}
C.~F\'evotte, N.~Bertin, and J.-L. Durrieu, ``Nonnegative matrix factorization
  with the {I}takura-{S}aito divergence. {W}ith application to music
  analysis,'' \emph{Neural Computation}, vol.~21, no.~3, pp. 793--830, 2009.

\bibitem{jorgensen87exponential}
B.~J{\o}rgensen, ``Exponential dispersion models,'' \emph{Journal of the Royal
  Statistical Society. Series B (Methodological)}, vol.~49, no.~2, pp.
  127--162, 1987.

\bibitem{cichocki09nonnegative}
A.~Cichocki, R.~Zdunek, A.~H. Phan, and S.~Amari, \emph{Nonnegative Matrix and
  Tensor Factorization}.\hskip 1em plus 0.5em minus 0.4em\relax John Wiley and
  Sons, 2009.

\bibitem{yilmaz12alpha_beta}
Y.~K. Yilmaz and A.~T. Cemgil, ``Alpha/beta divergences and {T}weedie models,''
  \emph{CoRR}, vol. abs/1209.4280, 2012.

\bibitem{hyvaerinen05estimation}
A.~Hyv{\"{a}}rinen, ``Estimation of non-normalized statistical models using
  score matching,'' \emph{Journal of Machine Learning Research}, vol.~6, pp.
  695--709, 2005.

\bibitem{lee99learning}
D.~D. Lee and H.~S. Seung, ``Learning the parts of objects by non-negative
  matrix factorization,'' \emph{Nature}, vol. 401, pp. 788--791, 1999.

\bibitem{yuan05pnmf}
Z.~Yuan and E.~Oja, ``Projective nonnegative matrix factorization for image
  compression and feature extraction,'' in \emph{Proceedings of 14th
  Scandinavian Conference on Image Analysis}, 2005, pp. 333--342.

\bibitem{TNN2010}
Z.~Yang and E.~Oja, ``Linear and nonlinear projective nonnegative matrix
  factorization,'' \emph{IEEE Transactions on Neural Networks}, vol.~21, no.~5,
  pp. 734--749, 2010.

\bibitem{lu12selecting}
Z.~Lu, Z.~Yang, and E.~Oja, ``Selecting $\beta$-divergence for nonnegative
  matrix factorization by score matching,'' in \emph{Proceedings of the 22nd
  International Conference on Artificial Neural Networks (ICANN 2012)}, 2012,
  pp. 419--426.

\bibitem{eguchi01robustifing}
S.~Eguchi and Y.~Kano, ``Robustifing maximum likelihood estimation,'' Institute
  of Statistical Mathematics, Tokyo, Tech. Rep., 2001.

\bibitem{minami02robust}
M.~Minami and S.~Eguchi, ``Robust blind source separation by beta divergence,''
  \emph{Neural Computation}, vol.~14, pp. 1859--1886, 2002.

\bibitem{renyidivergence}
A.~R\'enyi, ``On measures of information and entropy,'' in \emph{Procedings of
  4th Berkeley Symposium on Mathematics, Statistics and Probability}, 1960, pp.
  547--561.

\bibitem{mollah2007beta}
M.~Mollah, S.~Eguchi, and M.~Minami, ``Robust prewhitening for ica by
  minimizing beta-divergence and its application to fastica,'' \emph{Neural
  Processing Letters}, vol.~25, pp. 91--110, 2007.

\bibitem{choi2010alphaintegration}
H.~Choi, S.~Choi, A.~Katake, and Y.~Choe, ``Learning alpha-integration with
  partially-labeled data,'' in \emph{Proc. of the IEEE International Conference
  on Acoustics, Speech, and Signal Processing}, 2010, pp. 14--19.

\bibitem{dunn05series}
P.~K. Dunn and G.~K. Smyth, ``Series evaluation of {T}weedie exponential
  dispersion model densities,'' \emph{Statistics and Computing}, vol.~15,
  no.~4, pp. 267--280, 2005.

\bibitem{dunn01tweedie}
------, ``Tweedie family densities: methods of evaluation,'' in
  \emph{Proceedings of the 16th International Workshop on Statistical
  Modelling}, 2001.

\bibitem{hyvaerinen07extensions}
A.~Hyv\"{a}rinen, ``Some extensions of score matching,'' \emph{Comput. Stat.
  Data Anal.}, vol.~51, no.~5, pp. 2499--2512, 2007.

\bibitem{fevotte11algorithms}
C.~F\'{e}votte and J.~Idier, ``Algorithms for nonnegative matrix factorization
  with the beta-divergence,'' \emph{Neural Computation}, vol.~23, no.~9, 2011.

\bibitem{tan2013pami}
C.~F. V.~Tan, ``Automatic relevance determination in nonnegative matrix
  factorization with the $\beta$-divergence,'' \emph{IEEE Transactions on
  Pattern Analysis and Machine Intelligence}, 2013, accepted, to appear.

\bibitem{plsi}
T.~Hofmann, ``Probabilistic latent semantic indexing,'' in \emph{International
  Conference on Research and Development in Information Retrieval (SIGIR)},
  1999, pp. 50--57.

\bibitem{QNMF}
Z.~Yang and E.~Oja, ``Quadratic nonnegative matrix factorization,''
  \emph{Pattern Recognition}, vol.~45, no.~4, pp. 1500--1510, 2012.

\bibitem{TNN2011ROZ}
------, ``Unified development of multiplicative algorithms for linear and
  quadratic nonnegative matrix factorization,'' \emph{IEEE Transactions on
  Neural Networks}, vol.~22, no.~12, pp. 1878--1891, 2011.

\bibitem{swimmerdata}
D.~Donoho and V.~Stodden, ``When does non-negative matrix factorization give a
  correct decomposition into parts?'' in \emph{Advances in Neural Information
  Processing Systems 16}, 2003, pp. 1141--1148.

\bibitem{vincent2009ardnmf}
V.~Y.~F. Tan and C.~F\'{e}votte, ``Automatic relevance determination in
  nonnegative matrix factorization,'' in \emph{Proceedings of 2009 Workshop on
  Signal Processing with Adaptive Sparse Structured Representations
  (SPARS'09)}, 2009.

\bibitem{LVA2010}
Z.~Yang, Z.~Zhu, and E.~Oja, ``Automatic rank determination in projective
  nonnegative matrix factorization,'' in \emph{Proceedings of the 9th
  International Conference on Latent Variable Analysis and Signal Separation
  (LVA2010)}, 2010, pp. 514--521.

\bibitem{parzen}
E.~Parzen, ``{On Estimation of a Probability Density Function and Mode},''
  \emph{The Annals of Mathematical Statistics}, vol.~33, no.~3, pp. 1065--1076,
  1962.

\bibitem{kulis2009lowrank}
B.~Kulis, M.~A. Sustik, and I.~S. Dhillon, ``Low-rank kernel learning with
  bregman matrix divergences,'' \emph{Journal of Machine Learning Research},
  vol.~10, pp. 341--376, 2009.

\bibitem{gausslaguerre}
M.~Abramowitz and I.~A. Stegun, Eds., \emph{Handbook of Mathematical Functions
  with Formulas, Graphs, and Mathematical Tables}, 9th~ed.\hskip 1em plus 0.5em
  minus 0.4em\relax New York: Dover, 1972, page 890.

\end{thebibliography}
%%%%%%%%%%%%%%%%%%%%%%%%%%%%%%%%%%%%%%%%%%%%%%%%
\end{document}